\documentclass[11pt]{article}

\usepackage[T1]{fontenc}
\usepackage{array}
\usepackage{booktabs}
\usepackage[utf8]{inputenc}
\usepackage[english]{babel}
\usepackage[margin=.89in]{geometry} % sets all margins to 1 inch
\usepackage{comment}
\usepackage{cancel} 
\usepackage{hyperref}
\usepackage{tabularx}
\usepackage[numbers]{natbib}
\usepackage{makecell} % Required for line breaks within table cells

\usepackage{amsmath, amssymb, mathabx, mathrsfs, mathtools, nccmath, amsthm}
\usepackage{hyperref}
\usepackage{graphicx}
\usepackage{color}
\usepackage{tcolorbox}

\usepackage{epsfig}
\usepackage{tikz, pgfplots}
\usepackage{subcaption}
\usepackage{float}

\usepackage{comment}
\usepackage[multiple]{footmisc}
\usepackage{bigfoot}
\usepackage{multicol}
\usepackage{ulem}
\usepackage{stackengine}
\usepackage{algorithm}
\usepackage{algpseudocode}
\usepackage{url}

\usepackage{algorithm}
\usepackage{algpseudocode}

\usepackage[toc,page]{appendix}

\newtheorem{assumption}{Assumption}
\theoremstyle{plain}
\newtheorem{thm}{Theorem}[section]
\newtheorem{cor}[thm]{Corollary}
\newtheorem{lem}[thm]{Lemma}

\newtheorem{defn}[thm]{Definition}

\newtheorem{ex}[thm]{Example}

\theoremstyle{definition}
\theoremstyle{remark}
\newtheorem{remark}{Remark}[section]

\DeclarePairedDelimiter{\abs}\lvert\rvert

\newcommand{\R}{\mathbb{R}}
\title{
Pruning Deep Neural Networks via a Combination of the  Marchenko-Pastur Distribution and Regularization}

\author{
    Leonid Berlyand \\
    {\small Department of Mathematics} \\
    {\small Pennsylvania State University} \\
    {\small University Park, PA 16802, USA} \\
    \and
    Theo Bourdais\\
    {\small Department of Computing and Mathematical Sciences} \\
    {\small California Institute of Technology} \\
    {\small 1200 E California Blvd} \\
    {\small Pasadena, CA 91125, USA} \\
      \and
     Houman Owhadi\\
    {\small Department of Computing and Mathematical Sciences } \\
    {\small California Institute of Technology} \\
    {\small 1200 E California Blvd} \\
    {\small Pasadena, CA 91125, USA} \\
    \and
    Yitzchak Shmalo \\
    {\small Department of Mathematics} \\
    {\small Pennsylvania State University} \\
    {\small University Park, PA 16802, USA} \\
}
\date{} % Remove this line if you want to display the current date

\date{} % Remove this line if you want to display the current date

\begin{document}

\maketitle

\begin{abstract}
Deep neural networks (DNNs) have brought significant advancements in various applications in recent years, such as image recognition, speech recognition, and natural language processing.
In particular, Vision Transformers (ViTs) have emerged as a powerful class of models in the field of deep learning for image classification.
In this work, we propose a novel Random Matrix Theory (RMT)-based method for pruning pre-trained DNNs, based on the sparsification of weights and singular vectors, and apply it to ViTs. RMT provides a robust framework to analyze the statistical properties of large matrices, which has been shown to be crucial for understanding and optimizing the performance of DNNs. We demonstrate that our RMT-based pruning can be used to reduce the number of parameters of ViT models (trained on ImageNet) by 30-50\% with less than 1\% loss in accuracy\footnote{Code for this can be found at \url{https://github.com/yspennstate/RMT_pruning_ViT}}. To our knowledge, this represents the state-of-the-art in pruning for these ViT models.  We also show numerically that RMT-based pruning can be used to increase the accuracy of fully connected DNNs. Furthermore, we provide a rigorous mathematical underpinning of the above numerical studies, namely we proved a theorem for fully connected DNNs, and other more general DNN structures, describing how the randomness in the weight matrices of a DNN decreases as the weights approach a local or global minimum (during training). This decrease in randomness may explain why RMT-based pruning can enhance DNN performance for fully connected DNNs, as the pruning can help reduce the randomness in the weight layers, allowing for faster convergence to the minima.  We finally verify this theorem through numerical experiments on fully connected DNNs, providing empirical support for our theoretical findings. Moreover, we prove a theorem that describes how DNN loss decreases as we remove randomness in the weight layers, and show a monotone dependence of the decrease in loss with the amount of randomness that we remove. Our results also provide significant RMT-based insights into the role of regularization during training and pruning.  
\end{abstract}

\maketitle

\tableofcontents

\section{Introduction}

DNNs have emerged as a powerful tool in the realm of classification tasks, where they categorize objects from a set $S \subset \mathbb{R}^n$. They are a class of parametric functions and will be defined in section \ref{introduction_deep_neural_networks}. The training of DNNs on labeled training datasets $T \subset \mathbb{R}^n$ is done by optimizing its parameters to minimize a loss in order to increase the network's accuracy, typically the cross-entropy loss, as shown in \eqref{CE}.  DNNs have been applied to various real-world challenges, achieving remarkable results in fields such as handwriting recognition \cite{LBD}, image classification \cite{krizhevsky2017imagenet}, speech recognition \cite{hinton2012deep}, and natural language processing \cite{sutskever2014sequence} among others. One significant hurdle in DNN training is overfitting, where the model excessively adapts to the training data, leading to poor generalization of unseen data. This often results in a decline in performance on test datasets despite high accuracy on the training set T. To mitigate overfitting, researchers have developed various regularization strategies, including dropout \cite{srivastava2014dropout}, early stopping \cite{prechelt2012early}, and weight decay regularization \cite{Park2023}.

Recently, Random Matrix Theory (RMT) has emerged as a promising approach to tackling overfitting in deep learning \cite{martin2021implicit,mahoney2019traditional}. This includes the development of RMT-based stopping criteria \cite{meng2023impact} and regularization methods \cite{xiao2023heavy}. Furthermore, RMT has been applied to analyze the spectrum of weight layer matrices \cite{thamm2022random} and study the input-output Jacobian matrix \cite{pastur2020random,pastur2023random}. Investigations into RMT-based initializations for DNNs have also been conducted, see e.g, \cite{saada2023initialisation}. It has been demonstrated that the spectral characteristics of DNN weight layer matrices can provide insightful information about the network's performance, even in the absence of test data. In particular, it has been shown in \cite{martin2021predicting,martin2020heavy} that the spectrum of the weight layers of a DNN can be indicative of the DNN's accuracy on unseen test sets. 
In this work, we show that analyzing these spectral properties helps identify weights that have minimal impact on accuracy, enabling the pruning of a DNN’s weights while maintaining its accuracy. 

In the comprehensive review by \cite{vadera2022methods}, numerous studies on pruning are categorized into three main strategies: magnitude-based pruning, redundancy identification through clustering, and sensitivity analysis-based pruning. Other works and surveys on pruning in deep learning can be found in \cite{cheng2023survey} and \cite{khetan2020prunenet}. Several previous studies, such as \cite{yang2020learning,xue2013restructuring,cai2014fast,anhao2016svd,xu2019trained}, have leveraged Singular Value Decomposition (SVD) to prune small singular values from DNN weight matrices, focusing on empirical techniques like energy thresholds and validation set error monitoring. 
In our work, we introduce a theoretically grounded magnitude-based pruning method using RMT. The Marcenko-Pastur (MP) distribution, a famous RMT distribution (see Section \ref{sec:MP_distribution}), characterizes the distribution of the singular values of random matrices. Using the MP distribution, we may recover the amount of noise in the weight layers and compare their spectrum to that of purely random matrices. Eigenvalues that deviate significantly from what would be expected in a purely random matrix indicate weights that contain information and are critical to the network's function. Conversely, weights associated with eigenvalues that align closely with those of random matrices are considered candidates for removal. Additionally, we quantify the degree to which each layer’s spectrum aligns with the MP distribution, providing a measure of “randomness” that guides targeted training and regularization efforts. For the classical case of fully connected DNNs, it has been demonstrated in \cite{berlyand2023enhancing} both numerically and theoretically using rigorous mathematical analysis that MP-based pruning can reduce a network's parameters by over $99\%$  while increasing its accuracy.

The goal of this work is to develop state-of-the-art pruning algorithms for Vision Transformer (ViT) models based on the MP spectral approach from RMT theory. We compare our pruning algorithms with those presented in \cite{zhu2021visual, goyal2020powerbert, pan2021scalable, song2022cp}. A key feature of our method, compared to other works, is that our pruning does not rely on any access to a validation or test set. In principle, this pruning approach can be extended to other deep neural network models, such as Convolutional Neural Networks (CNNs). We also provide a rigorous mathematical analysis of how removing randomness from the weight layers of DNNs reduces loss without affecting accuracy. We demonstrate that during training, the randomness in the DNN weight layers diminishes over time. As DNN parameters converge to a local minimum, the randomness vanishes.

The paper has been organized as follows. In Sections 1 and 2, a brief overview of key notions and concepts is presented. In Section \ref{sec:pruning_VITs}, we apply our MP-based pruning and fine-tuning approach to ViTs pre-trained on the 1K ImageNet dataset \cite{5206848}. We study the reduction in parameters and how it affects the accuracy of the DNN on the ImageNet validation set. We compare our results with the following work \cite{song2022cp} and the references therein, showing that we achieve higher pruning with less reduction in accuracy. 

In Sections \ref{Main_theory} and \ref{Main_theory_gen}, we present our main mathematical results that explain the evolution of randomness in DNN weight layers during training. We prove two key theorems. Our first theorem demonstrates that for a DNN trained with $L2$ regularization, the random parameters within the network's weight matrices diminish and vanish as the network's loss approaches a local or global minimum. This result underscores the transition of the network's parameters from a regime dominated by randomness (from their initialization) to one where deterministic structures prevail, thereby contributing to how DNNs extract information from noisy data under training. The second theorem  describes how DNN loss decreases as we remove randomness in the weight layers and shows a monotone dependence of the decrease in loss with the amount of randomness that we remove.

Our numerical findings and corresponding mathematical results illustrate the interplay between regularization and pruning (see Theorem \ref{theorem_reduction_of_loss}), emphasizing the critical role of regularization techniques in the training and pruning of DNNs. Our numerical approach is underpinned by rigorous mathematical results and tested on advanced ViT models, demonstrating its applicability to modern deep-learning models.

\section{Randomness in Deep Neural Networks}
\subsection{Introduction to Deep Neural Networks}
\label{introduction_deep_neural_networks}

In classification tasks, the goal is to categorize elements of a set $S$ into one of $K$ distinct classes. Let $C(s)\in\{1,..,K\}$ be the correct class of $s\in S$. Given a labeled training set $T\subset S$, we wish to develop an approximate classifier to extend the accurate classification from the training set $T$ to the entire set $S$. To do so, we use a DNN as a parameter-dependant classifier $\phi( s,\alpha)$, where $\alpha \in \mathbb{R}^\nu$ are its parameters. The DNN outputs a probability vector $(p_1(s), \dots, p_K(s))$ such that $p_k(s):=\mathbb{P}[C(s)=k]$ is the estimated probability distribution of $C(s)$. \\
We consider DNNs which are a composition $\rho \circ X(\cdot,\alpha)$, with $\rho$ the softmax (given in \eqref{soft_max_new}) and $X(\cdot,\alpha)$ defined as follows. Let 
 \begin{equation}
X(\cdot,\alpha)=\lambda\circ M_L(\cdot,\alpha)\circ\cdots\circ\lambda\circ M_1(\cdot,\alpha)\end{equation} with:
 \begin{itemize}
    	\item $M_k(\cdot,\alpha)$  an affine function $\mathbb R^{N_{k-1}}\to\mathbb R^{N_k}$ parameterized by a weight matrix $W_k\in\mathbb{R}^{N_{k}\times N_{k-1}}$ and bias vector $\beta_k\in\mathbb{R}^{N_{k}}$ such that $M_k(x)=W_k\cdot x+\beta_k$.
		\item $\lambda: \R^m \mapsto \R^m$ is a nonlinear activation function. In this case,  we assume $\lambda$ is the absolute value activation function, such that $\lambda:x\in\R^m \mapsto (|x_1|, \dots, |x_m|)\in\R^m$.  
		\item The softmax function $\rho$ normalizes the outputs of $X(\cdot,\alpha)$ into a probability vector over $\{1,\dots,K\}$. The individual components of $\rho$ are computed as:
\begin{equation}
\label{soft_max_new}
\rho(v)_i=\frac{\exp(v_i)}{\sum_{i=1}^{K}\exp(v_i)}\text{ for }v\in\mathbb{R}^K
\end{equation}
 \end{itemize}
Given the neural network architecture above, we train our classifier $\phi(\cdot,\alpha)$ by minimizing the cross-entropy loss. This loss quantifies the accuracy of our classifier over the training set, and is defined as 

\begin{equation}
\label{CE}
L(\alpha)=-\frac{1}{|T|}\sum_{s\in T} \log\left(\phi_{C(s)}(s,\alpha)\right).
\end{equation}

Training a DNN is fundamentally about exploring a complex, high-dimensional, and non-convex landscape of loss to identify the global minimum. However, the intricacies of these landscapes often result in the emergence of local minima, saddle points, and broad flat regions. These features can act as impediments in the learning process, obstructing the path to an optimal solution. The prevalence of such challenges escalates with the increase in dimensionality and complexity of the DNN \cite{dauphin2014identifying, choromanska2015loss}.

In this context, the potential of RMT is noteworthy. By applying RMT to prune DNN weight layers, we can simplify the loss landscape, thereby reducing local minima and saddle points. This simplification assists the optimization process in finding the global minimum, potentially leading to higher training accuracy without reaching a plateau and overall improved model performance. We observe this improvement in performance (higher accuracy) for fully connected DNNs trained with $L1$ and $L2$ regularization; see Appendix \ref{fully_conneced}. 

\subsection{The MP Distribution in Machine Learning contexts}
\label{sec:MP_distribution}

The identification of the Marchenko-Pastur distribution, a fundamental result in RMT, has far-reaching implications in fields like signal processing, wireless communication, and machine learning. References such as \cite{vershynin2018high,ge2021large,serdobolskii2000multivariate,couillet2011random} detail its applications. This distribution provides insights into the spectral density of large random matrices, revealing the asymptotic eigenvalue distribution in these matrices and predicting their behavior under various conditions. Furthermore, the MP distribution is pivotal in dimension reduction techniques, including principal component analysis (PCA), as highlighted in \cite{abdi2010principal,bro2014principal,ringner2008principal}.\\
We introduce the concept of the empirical spectral distribution (ESD) for an $N \times M$ matrix $G$, then state the Marcenko-Pastur theorem:
\begin{defn}\label{ESD_Definition}
The ESD for an $N \times M$ matrix $X$ is defined as:
\begin{equation}
\mu_{X_M} = \frac{1}{M} \sum_{i=1}^M \delta_{\sigma_i},
\end{equation}
where $\sigma_i$ are the non-zero singular values of $X$, and $\delta$ signifies the Dirac measure.
\end{defn}
\begin{thm}[Marchenko and Pastur (1967) \cite{marchenko1967distribution}]
\label{RMT_MP_theorem}
Consider an $N\times M$ random matrix $W$ with $M \leq N$. Let the entries $W_{i,j}$ be independent, identically distributed with zero mean and finite variance $\sigma^2$. Let $X = \frac{1}{N} W^T W$. When $N \to \infty$ and $\frac{M}{N} \to c \in (0,+\infty)$, the ESD of $X$, denoted by $\mu_{X_M}$, converges in distribution to the Marchenko-Pastur probability distribution:
\begin{equation}
\label{MP_distribution}
\frac{1}{2\pi\sigma^2} \frac{\sqrt{(\lambda_+ - x)(x - \lambda_-)}}{cx} \mathbf{1_{[\lambda_-, \lambda_+]}} dx
\end{equation}
with
\begin{equation}
\label{lambda_parameters}
\lambda_\pm = \sigma^2(1\pm\sqrt{c})^2.
\end{equation}
\end{thm}
Theorem \ref{RMT_MP_theorem} states that as the dimensions of a random matrix increase, its eigenvalue distribution converges to the MP distribution. The MP distribution is deterministic and hinges on two parameters: the variance of the matrix's entries, $\sigma^2$, and the aspect ratio of the matrix, denoted by $c$.

\subsection{Reduction of randomness in DNN weights via MP Distribution}
\label{overallpicture}

As outlined in Subsection \ref{introduction_deep_neural_networks}, a DNN is composed of affine functions $M_l$, represented by an $N \times M$ matrix $W_l$ of parameters and a bias vector $\beta_l$. Before training, these matrices are often initialized with i.i.d. entries according to a normal distribution $\mathcal
{N}(0,\frac{1}{N})$. Given the large size of these matrices, the Empirical Spectral Distribution (ESD) of $X_l=\frac{1}{N}W_l^T W_l$ is close to the MP distribution \emph{before training}. Yet, previous studies have used the spiked model in random matrices to examine $W_l$, showing that the ESD of $X_l$ \emph{after training} exhibits eigenvalues that diverge from the Marchenko-Pastur limit $\lambda_+$ \cite{martin2021implicit,staats2022boundary}. This observation suggests that the weight matrices contain signal obtained from the data after training. 

With random initialization and training through Stochastic Gradient Descent (SGD), the sequence of parameters during training $\alpha(t)$ satisfies\begin{align}
    &\alpha(0)\text{ is i.i.d. Gaussian}\\
    &\alpha(t+1) = \alpha(t) - \tau\nabla L(\alpha(t))
\end{align}
While $\nabla L(\alpha(t))$ contains information derived from the training data (signal) it also includes noise due to the data itself possibly being noisy, and the use of random minibatches by SGD. Thus,  each update of $\alpha$ contains both noise and signal. To account for this observation, we model the training process using the following suppositions:

\paragraph{Supposition 1} After $t$ steps of training, the weight matrix at layer $l$ can be decomposed as\begin{equation}
    W_l(t)=R_l(t)+S_l(t)
\end{equation} where $S_l(t)$ represents the signal (structured information learned from the data) and $R_l(t)$ is an independent random noise perturbation. 

As training progresses, we expect the magnitude of the signal $\lVert S_l(t)\rVert_F$ to increase while the magnitude of the noise $\lVert R_l(t)\rVert_F$ decreases, where $\|\cdot\|_F$ is the Frobenius norm of a matrix. The theory will explore this question in detail. Given this decomposition, we will use the MP distribution to separate the signal from the noise. To do so, we assume that the signal in $W_l$ is strong enough to be distinguishable from the noise. Since the singular values of $R_l$ are determined by the MP distribution, with maximal value $\sqrt{\lambda_+}$, this supposition can be reformulated as: 

\paragraph{Supposition 2} Singular values $\sigma_i$ of $W_l$ below $\sqrt{\lambda_+}$ are likely from $R_l$, where $\lambda_+$ is the upper bound of the MP distribution of $R_l^TR_l$ 
\bigbreak

In section \ref{sec:pruning_VITs}, we will use Supposition 2 to design a pruning strategy and apply it to VIT models. In section \ref{Main_theory}, we formalize these suppositions and prove that removing $R_l(t)$ does not decrease accuracy (see Lemma \ref{main_result_remove _R}). The theory, through Theorem \ref{RMT_MP_theorem}, provides insights into how the training reduces the noise in weight matrices. We first present our result for a simplified case where the DNN has only three weight layer matrices and no bias vector. In future works, the results will be extended to more general ViT models.

\section{Pruning Visual Transformers (ViTs)}
\label{sec:pruning_VITs}

ViTs have emerged as state-of-the-art models for image classification tasks on benchmarks such as ImageNet \cite{5206848}. Most weights in ViTs are organized as dense matrices, representing affine transformations through matrix-vector products, such as in Multi-Layer Perceptrons (MLPs) and Attention layers \cite{vaswani2017attention}. This structure makes ViTs particularly compatible with the framework of RMT. For further details on the architecture of ViTs, see Appendix \ref{ViT_modl}.

\subsection{Pruning strategy based on Random Matrix Theory}
\label{pruning_strategy}

In this section, we detail our pruning strategy that leverages insights from RMT. The pruning process involves selectively removing weights from a DNN based on certain criteria derived from RMT. Our goal is to reduce the model size while preserving its accuracy.

Our pruning strategy is based on the suppositions made in section \ref{overallpicture}, and extensive experiments on ViT models. Given our supposition that the weight matrices are perturbed low-rank structures, we can prune them by setting small values, which are likely noise, to zero. We define the pruning function for $x,\theta\in\mathbb{R}$:\begin{equation}
    \text{Prune}(x,\theta)=\left\{\begin{matrix}
        x \text{ if }\abs{x}>\theta\\
        0 \text{ if }\abs{x}\leq\theta
    \end{matrix}\right.
\end{equation}
This function can be vectorized for matrices $W\in\mathbb{R}^{N\times M}$ s.t. $\text{Prune}(W,\theta)_{i,j}=\text{Prune}(W_{ij},\theta)$.\\
There are three types of pruning that can be done:\begin{enumerate}
    \item \textbf{Singular Vector Pruning:} Small values in the singular vectors may be interpreted as noise. This step helps protect the model during pruning by enhancing the low-rank structure and preserving accuracy. 
    \item \textbf{Singular value Pruning:} Although small singular values could be considered noise according to our suppositions, experiments showed that pruning these values significantly degraded the network’s performance. Therefore, this technique was excluded from our final pruning procedure.
    \item \textbf{Direct coefficients pruning:} We may also directly set to zero small coefficients in the weights layer, as they may be interpreted as noise. 
\end{enumerate}
After pruning, we observe that applying a combination of $L_1$ and $L_2$ regularization helps in refining the weight matrices, promoting sparsity while maintaining the network’s performance.\\
To guide the pruning and define what is interpreted as noise, we introduce two key metrics:
\begin{itemize}
    \item \textbf{Spike Metric $\gamma$:} This metric quantifies the proportion of spikes in the weight matrices of the layers. It is defined as:
    \[
    \gamma = \frac{1}{N} \cdot \text{Card}\left\{ \sigma \leq \sqrt{\lambda_+ N}, \text{ for } \sigma \text{ singular value of } W_{\ell} \right\}
    \]
    where $N$ is the size of the weight matrix $W_{\ell}$, and $\lambda_+$ is the upper edge of the Marchenko-Pastur distribution fitted to the Empirical Spectral Density (ESD) of $X_{\ell} = \frac{1}{N} W_{\ell}^T W_{\ell}$. In the base ViT described below, its average value is 88.55\%. This metric suggests how large the rank of the deterministic matrix $S$ is compared to the total rank of $W$. See Appendix \ref{numerics_updated} for more on spikes. 

    \item \textbf{MP Fit Metric $\mu$:} This metric measures the error in fitting the MP distribution to the ESD of $X_{\ell} = \frac{1}{N} W_{\ell}^T W_{\ell}$. It is computed based on a fitting algorithm (see Subsection~\ref{Alignment_Evaluation}). In the base VIT described below, its average value is 0.3. 
\end{itemize}

\subsubsection{Singular vector pruning}
%To prune the singular vectors, we use the heuristic that \textbf{put heuristic}.

For a triplet $(v,w,\sigma)$ of respectively left, right singular vectors and singular value, we apply pruning such that: \begin{equation}
    v \leftarrow \text{Prune}\left(v,\theta\max\left\{\frac{1}{750},\left[1-\frac{\sigma}{\sqrt{\lambda_+}}\right]^{30}\right\}\right)
\end{equation}
A similar update is performed on $w$. The specifics of this pruning were found by extensive experiments. The pruning of the components of the singular vectors is based on their size, with smaller components being pruned. The threshold of this pruning is based on $\theta= .00001125 \times r \times N \times M$, with $N$ and $M$ being the sizes of the matrix. This process reduces the complexity of the model by removing less significant components of the singular vectors. It seems to "protect" the DNN during the pruning phase and make the weight layers "more" low rank, see Subsection \ref{SV-sparsification}. 

\subsubsection{Direct coefficients pruning}

To prune the weight matrices, we use the metric $\zeta_1(t)$, which depends on the randomness of the layer and the current pruning cycle $t$. It is defined as:
\begin{equation}
\label{zeta_equation}
\zeta_1(t) = \left[ (1 - \mu) \cdot \gamma \right]^{\frac{1.5}{t}} \cdot r \cdot \text{Card}\left\{ | W_{\ell} |_{ij} > 0 \right\}
\end{equation}
where:
\begin{itemize}
    \item $\mu$ is the MP Fit Metric of the layer.
    \item $\gamma$ is the spike metric of the layer.
    \item $r$ is the desired proportion of parameters to remove in each pruning cycle (set to $r = 6\%$).
    \item $t$ is the current pruning cycle (with total cycles $N_{\text{cycle}} = 19$).
    \item $\text{Card}\left\{ | W_{\ell} |_{ij} > 0 \right\}$ is the number of non-zero parameters in the weight matrix $W_{\ell}$.
\end{itemize}

The above equation for $\zeta_1(t)$ was derived through a process of trial and error, but it also embodies some common sense principles. In particular, weight matrices that are determined to be more random—based on the randomness criteria provided by the MP Fit Metric ($\mu$) and the spike metric ($\gamma$)—are pruned more aggressively in the initial pruning cycles. This approach leverages the idea that less structured (i.e., more random) weights are likely to be less critical to the network's performance.

The exponent $\frac{1.5}{t}$ ensures that as we progress through the pruning cycles (i.e., as $t$ increases), the influence of the randomness metrics $\mu$ and $\gamma$ diminishes. This design allows us to focus more on pruning layers that are identified as more random in the early cycles and shift towards a more uniform pruning strategy in later cycles.
Using this, each matrix is updated using the following rule:\begin{equation}
    W_{\ell}\leftarrow \text{Prune}(W_\ell,f\max\{3,5[(1-\mu)\gamma]^{\frac{1.5}{t}}\})
\end{equation}
where $f$ is the pruning factor. $f$ is found to be the smallest number so that the total amount of weights pruned in the matrix is larger or equal to $\zeta_1(t)$. In practice, $f$ is initialized at $f=1e-6$ and increased by steps of $5e-6$ until the total amount of weights pruned is more than $\zeta_1(t)$. 

\subsubsection{Regularization}
\label{regu}

Regularization helps mitigate potential accuracy loss due to pruning. It also helps in refining the weight matrices after pruning, promoting sparsity while maintaining the network's performance. 

After every pruning cycle, we minimize the loss function 
\[
\mathcal{L} = \mu_1 \| W \|_1 + \mu_2 \| W \|_2^2,
\]
with $\mu_1 = 5 \times 10^{-6}$ and $\mu_2 = 2 \times 10^{-6}$ using SGD and a learning rate of $0.05 \times 10^{-6}$ (without any training data), so that although an infinite number of minimization steps would drive all weights to zero, the finite number of epochs applied causes the weights to simply become smaller—with some small weights vanishing—thus allowing us to decay the weight magnitudes without directly pruning them. Again, note that no training data is used during this regularization, and there are no other constraints for this minimization.

\subsection{Pruning procedure overview}

Having defined the different steps used in our procedure, we may now describe the overall process. To progressively remove weights, we perform $N_{cycle}=19$ times the following pruning cycle. First, we \textbf{prune the singular vectors} (only every other cycle). Then, coefficients of the weight matrix are pruned by \textbf{direct coefficient pruning}. Finally, we perform \textbf{regularization} for $n_{reg}$ epochs. We start with $n_{reg}=15$ and increase by 5 epochs each cycle, up to a maximum
of 40 epochs.

\paragraph{Fine-tuning} Finally, after completing all pruning and regularization cycles, we fine-tune the model to recover any potential loss in accuracy. This fine-tuning is performed while keeping the pruned weights frozen, thus keeping the sparsification we obtained through the procedure described above. We use SGD with a momentum of $0.9$, a base learning rate of $3 \times 10^{-5}$ with the cosine annealing scheduler and a warmup phase of 2 epochs (learning rate divided by 100). The model is fine-tuned for 20 epochs over the entire training set using cross-entropy loss. Finally, the gradient norm is clipped at $1$ to prevent exploding gradients.\\
Results are presented before and after fine-tuning.

\subsection{ Results for the visual transformer Vit\_Base\_Patch16\_224 Model}

Our numerical simulations explore the pruning of a pre-trained \texttt{vit\_base\_patch16\_224} model using RMT. This model has approximately 86.5 million parameters and archives a $85.1\%$ top-1 accuracy on the ImageNet validation set. We reduce the DNN's size by 30-40\% and evaluate its accuracy. We assess the model's top-1 and top-5 accuracy on the ImageNet validation set, consisting of $ \approx 200,000$ objects. 

\begin{remark}
    Some layers of this DNN model fit the MP distribution (based on Algorithm \ref{Alignment_Evaluation}) with an error of $1\%$ while others fit the MP distribution with an error of over $50\%$, see Subsection \ref{alpha_and_metrics} for more.    The spike metric and MP fit metric depends on $\alpha$, see Subsection \ref{alpha_and_metrics}. When $\alpha=.25$, we have that LRM is $86.25$ and MP fit metric of this ViT model is $.18$.  
\end{remark}

As the pruning progresses, the process yields a series of models with varying levels of sparsity and performance metrics. The final model selection can be based on a trade-off between accuracy and complexity. See Fig. \ref{ViT_example1_sparsification_with_RMT_1_FT} for the accuracies of the ViT model vs the number of parameters kept (there is no fine-tuning yet). We see that the DNN that was pruned by approximately $30\%$ has a reduction in top-1 accuracy by only $1.1\%$. After 1 epoch of fine-tuning the ViT model on the ImageNet training set, the DNN accuracy improved to $84.5\%$, resulting in a total decrease in accuracy of only $.6\%$.

\begin{figure}[h!]%
\label{vit_pruning_spar}
    \centering
    \subfloat[\centering Top 1 accuracy vs params. kept ]{{\includegraphics[width=7cm]{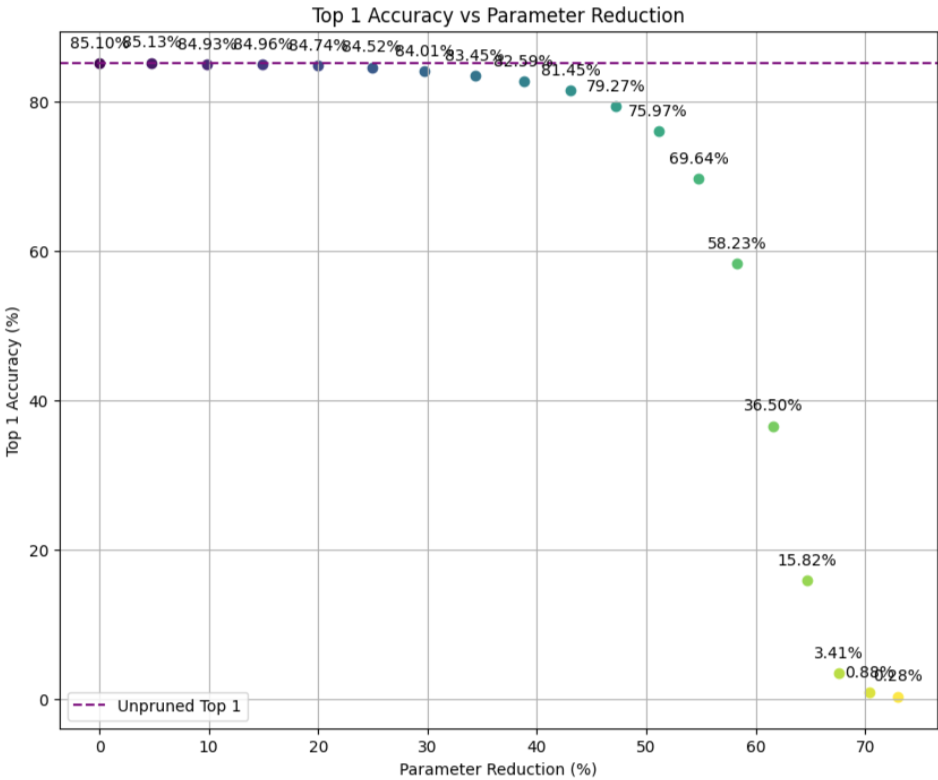} }}%
    \qquad
    \subfloat[\centering Top 5 accuracy vs params. kept ]{{\includegraphics[width=7cm]{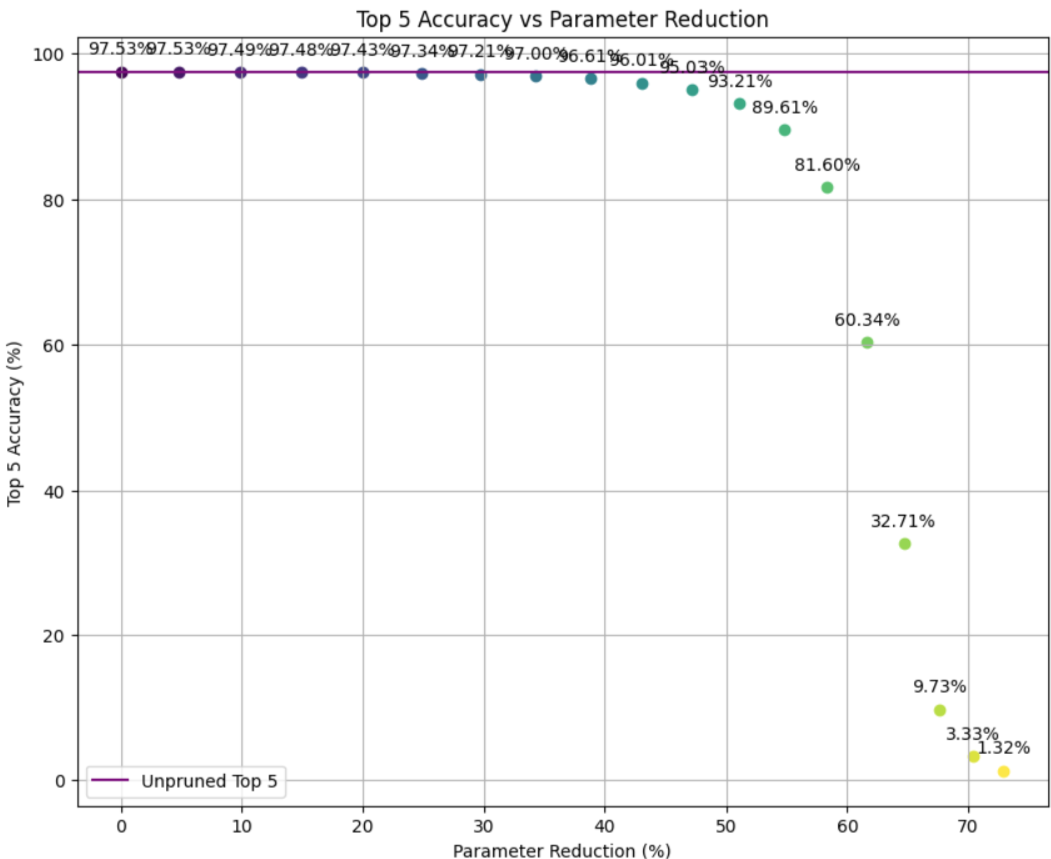} }}
    \caption{The top 1 accuracy and top-5 accuracy of ViT-base vs. percentage of parameters kept for pruning through RMT-based sparsification and no fine-tuning. The accuracy of the DNN is given as a percentage above the data points.}
    \label{ViT_example1_sparsification_with_RMT_1_FT}%
\end{figure}

\subsection{Pruning the visual transformer Vit\_large\_Patch16\_224 Model}
\label{large_vit_pruning_with_FT}

We performed the same pruning strategy on the \text{Vit\_large\_Patch16\_224 Model}. This ViT model is bigger than the one in the previous subsection, having approximately 300 million parameters. It achieves a $85.85\%$ accuracy on the ImageNet validation set. In Fig. \ref{ViT_larger}, we see the accuracy vs percentage of parameters kept for this pruning. After pruning $30\%$ of the DNN parameters, the accuracy only drops by $.8\%$. Again, after fine-tuning for only one epoch, we recover that full reduction in accuracy (i.e., has a 85.85\% accuracy). Finally, we also took the DNN for which $50\%$ of the parameters were pruned and fine-tuned for $12$ epochs, obtaining a $84.48\%$ accuracy after the fine-tuning.

\begin{figure}[h!]%
\label{vit_pruning_spar}
    \centering
    \subfloat[\centering Top 1 accuracy vs params. kept ]{{\includegraphics[width=7cm]{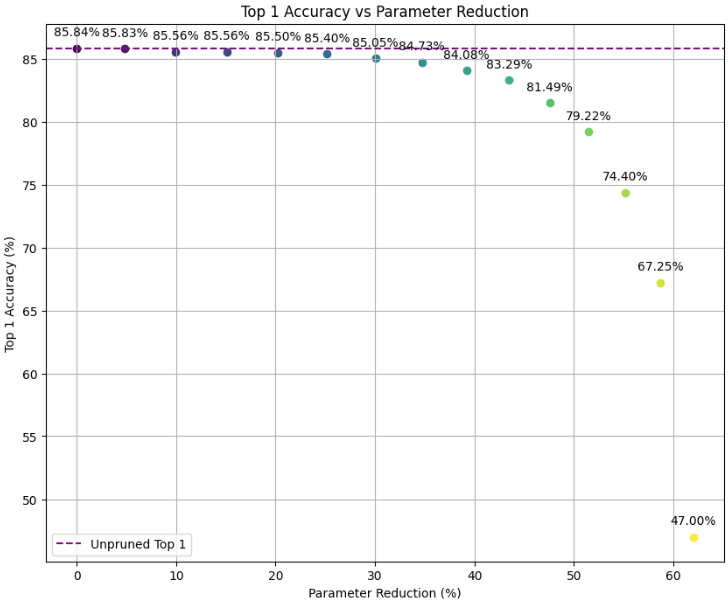} }}%
    \qquad
    \subfloat[\centering Top 5 accuracy vs params. kept ]{{\includegraphics[width=7cm]{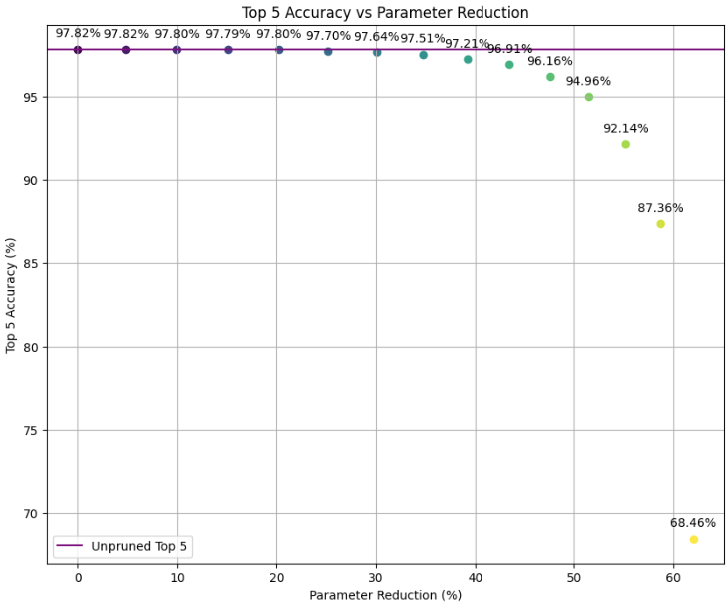} }}
    \caption{The top 1 accuracy and top-5 accuracy of ViT-large model vs. percentage of parameters kept for pruning through RMT-based sparsification and no fine-tuning. The accuracy of the DNN is given as a percentage above the data points.}
    \label{ViT_larger}%
\end{figure}

\subsection{Pruning results for ViT-B16/224 and ViT-L16/224 on ImageNet}

In this section, we compare our pruning results with CP-ViT \cite{song2022cp}, using the same architecture but different pre-trained weights. Our initial accuracies are higher: 85.1\% for ViT-B16/224 and 85.85\% for ViT-L16/224, compared to 77.91\% and 76.5\% in CP-ViT. Figure \ref{fig:res_comparison_largevit} shows the pruning results with FLOP reduction and accuracy reduction. Our fine-tuning was done for $20$ epochs, while the fine-tuning in \cite{song2022cp} was done with $30$ epochs.

\begin{figure}
    \centering
    \includegraphics[width=1\linewidth]{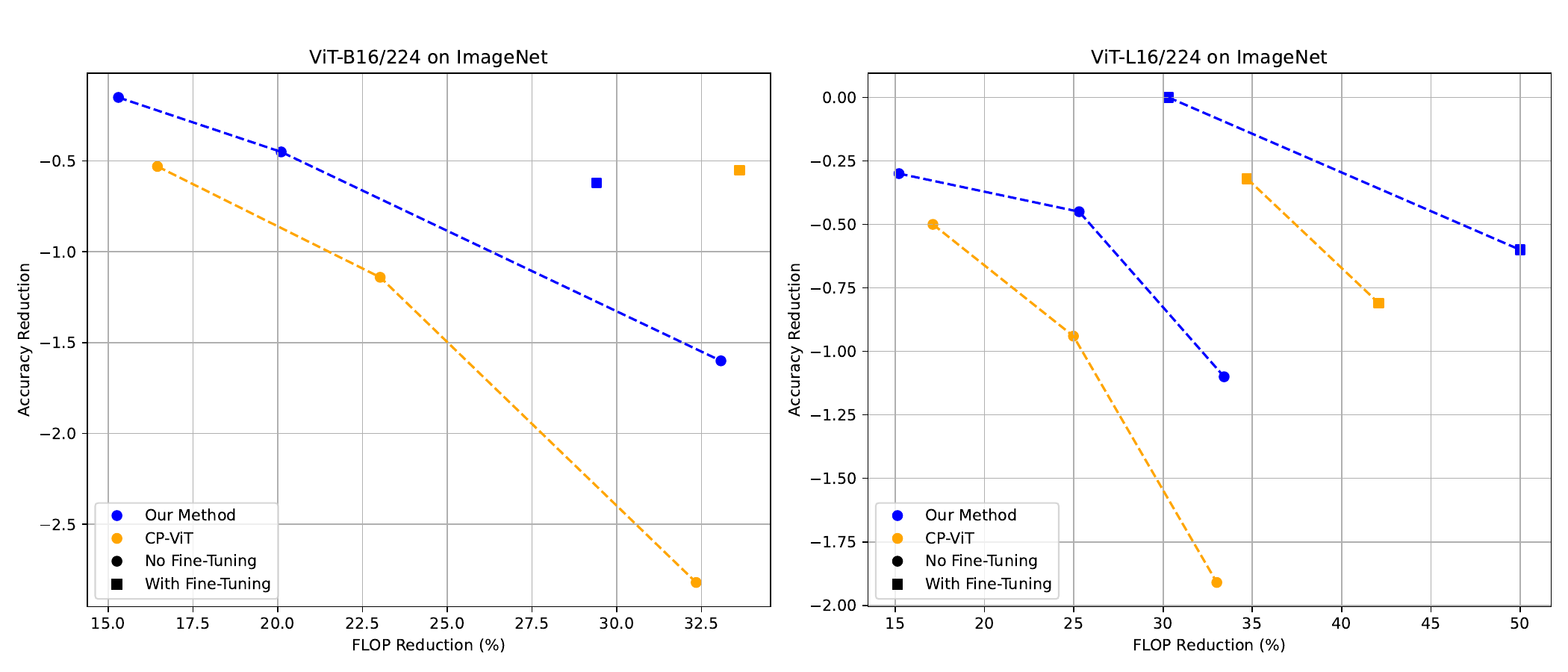}
    \caption{Comparison of pruning results between our method and CP-ViT \cite{song2022cp}}
    \label{fig:res_comparison_largevit}
\end{figure}

\subsection{RMT-based pruning of DeiT}

We perform the same pruning strategy with the same hyperparameters from Subsection \ref{pruning_strategy} on the three DeiT models, tiny, small and base. Table \ref{tab:comparison_table_2} shows the result of our RMT-based pruning method and compares them with other pruning methods, both with and without fine-tuning. For these models, we fine-tune for only $7$ epochs, which is much less than the fine-tuning done in the other works; for example, in \cite{song2022cp}, they fine-tune for $30$ epochs. From Table \ref{tab:comparison_table_2} and Fig. \ref{fig:res_comparison_largevit}, we see that we always outperform \cite{song2022cp} when there is no fine-tuning. Furthermore, Fig. \ref{fig:res_comparison_largevit} shows that we outperform \cite{song2022cp} for large DNNs even with fine-tuning, which makes sense from an RMT perspective given that the larger the DNN and the more parameters we have, the more RMT applies. Given the massive DNNs being trained these days, with over a trillion parameters, this RMT perspective becomes especially important.

\begin{table}[ht]
\centering
\caption{Comparison with different ViT pruning methods on the ImageNet dataset. The accuracy without finetuning was not mentioned in VTP, PoWER, and HVT but were replicated in \cite{song2022cp}. We fine-tune for only $7$ epochs on the ImageNet training set.}
\label{tab:comparison_table_2}
\begin{tabular}{lccccc}
\toprule
 & \multicolumn{2}{c}{Not Finetune} & \multicolumn{2}{c}{Finetune} \\
\cmidrule(lr){2-3} \cmidrule(lr){4-5}
Model & Top-1 Acc.(\%) & FLOPs Saving & Top-1 Acc.(\%) & FLOPs Saving \\
\midrule
\multicolumn{5}{c}{\textbf{DeiT-Ti 16/224 \cite{touvron2021training}}} \\
Baseline \cite{touvron2021training}  & 72.20 & -           & 72.20 & -           \\
VTP \cite{zhu2021visual}           & 69.37 (-2.83) & 21.68\%   & 70.55 (-1.65) & 45.32\%   \\
PoWER \cite{goyal2020powerbert}     & 69.56 (-2.64) & 20.32\%   & 70.05 (-2.15) & 41.26\%   \\
HVT \cite{pan2021scalable}          & 68.43 (-3.77) & 21.17\%   & 70.01 (-2.19) & 47.32\%   \\
CP-ViT \cite{song2022cp}           & 71.06 (-1.14) & 23.02\%   & 71.24 (-0.96) & 43.34\%   \\
RMT-based (ours)        & \textbf{71.09 (-1.11)} & 24.83\% &  &  \\
\midrule
\multicolumn{5}{c}{\textbf{DeiT-S 16/224 \cite{touvron2021training}}} \\
Baseline \cite{touvron2021training}  & 79.80 & -           & 79.80 & -           \\
VTP \cite{zhu2021visual}           & 77.35 (-2.45) & 20.74\%   & 78.24 (-1.56) & 42.52\%   \\
PoWER \cite{goyal2020powerbert}     & 77.02 (-2.78) & 21.46\%   & 78.30 (-1.50) & 41.36\%   \\
HVT \cite{pan2021scalable}          & 76.72 (-3.08) & 20.52\%   & 78.05 (-1.75) & 47.80\%   \\
CP-ViT \cite{song2022cp}           & 78.84 (-0.96) & 20.96\%   & 79.08 (-0.72) & 42.24\%   \\
RMT-based (ours)       & \textbf{78.91 (-0.89)} & 23.72\% & 78.34 (-1.46)  & 41.92\%   \\
\midrule
\multicolumn{5}{c}{\textbf{DeiT-B 16/224 \cite{touvron2021training}}} \\
Baseline \cite{touvron2021training}  & 81.82 & -           & 81.82 & -           \\
VTP \cite{zhu2021visual}           & 79.46 (-2.36) & 19.84\%   & 80.70 (-1.12) & 43.20\%   \\
PoWER \cite{goyal2020powerbert}     & 79.09 (-2.73) & 20.75\%   & 80.17 (-1.65) & 39.24\%   \\
HVT \cite{pan2021scalable}          & 78.88 (-2.94) & 20.14\%   & 79.94 (-1.88) & 44.78\%   \\
CP-ViT \cite{song2022cp}           & 80.91 (-0.91) & 22.16\%   & 81.13 (-0.69) & 41.62\%   \\
RMT-based (ours)      & \textbf{81.44 (-0.38)} & 21.38\% & 80.58 (-1.24) & 44\%  \\
\bottomrule
\end{tabular}
\end{table}

\section{Main theoretical results}
\label{Main_theory}

\subsection{Assumptions for main theorems: Gaussian case}
\label{assumptions}

For simplicity of the presentation, we choose a DNN with only three weight layer matrices and have other simplifying assumptions. However, many of the assumptions on the DNN and the weight layer matrices can be relaxed. For example, we can include arbitrary layers (see Subsection \ref{proof_main_theorem}),  components of $R_2$ can be i.i.d. from a distribution different then the normal distribution, and other activation functions can be used (such as Leaky ReLU).  In Section \ref{Main_theory_gen}, we present these results for more general DNN structures.  Recall the \(\ell_1\) norm of a matrix \(W\) with entries \(w_{ij}\):
\[
\|W\|_{1} = \max_{1 \leq j \leq n} \sum_{i=1}^m |w_{ij}|.
\]

\begin{assumption}
\label{as1}
 \textbf{Architecture of the DNN and bounds on the norms of weight layer matrices:} We consider the following DNN $\phi$:
\begin{align}
X(\alpha,s) &=  \lambda \circ W_3 \circ \lambda \circ W_2 \circ \lambda \circ W_1 s, \hspace{.4cm} s \in T,\\
\phi(\alpha,s)&=\rho \circ X(\alpha,s)
\end{align}
where $W_1, W_2, W_3$ are the weight layer matrices and $\lambda$ is the absolute value  or ReLU activation function, $T$ is the training set, $X(\alpha,s)$ outputs the components of the DNN, and $\rho$ is the softmax. We also denote $\{W_l(t)\}_{l=1}^3$ the weight layer matrices to highlight the dependence in training time $t$\\
We wish to study the behavior of the inner layer $W_2$ for large $N$, as training progresses. Thus, we assume the outer layers satisfy:
For all $ s $ in the training set $T$, the quantity 
\begin{equation}\label{eqans}
a(N,s):=  \frac{\| W_3 \|_1 \| W_1 s \|_2}{N^{1.5 / 4}}
\end{equation}
converges towards $ 0$ as $N \to \infty$ (see Subsection \ref{details_of_C1} for more details). 

\end{assumption}

\begin{assumption}
\label{as2}

\textbf{Deformed matrix form of matrix $W_2$:}

The $N \times N$ matrix $W_2(t)$, is such that: 
\begin{equation}
    W_2(t) = R_2(t) + S_2(t), \quad (\textit{deformed matrix})
\end{equation}
where $R_2(t)$ is a matrix with i.i.d. random normal entries that have mean zero and variance less than $\frac{1}{N}$, i.e. such that $R_2(t)_{i,j}\overset{i.i.d.}{\sim}\mathcal{N}(0,\frac{g(t)}{N})$, where $g(t) \leq 1$. The bound of $g(t)$ is unimportant and taken to be less than 1 for simplicity. It is necessary that $Var[R_2(t)_{i,j}]=\mathcal{O}(\frac{1}{N})$ so that $R_2(t)$ has a bounded spectral norm as $N\to\infty$.  
\begin{remark}Similar results will apply with a rectangular matrix $W_2(t)\in\mathbb{R}^{N\times M}$ such that\begin{equation}
    \frac{M}{N} \to c \in (0,+\infty)
\end{equation}
\end{remark}
\end{assumption}

\begin{assumption}
\label{as3}
    \textbf{ Low-rank nature of $S_2(t)$:}
    
    Let $\lambda_+$ be the rightmost edge of the Marchenko-Pastur distribution (see \eqref{lambda_parameters}) for the empirical spectral distribution of the Wishart matrix $R_2^T(t) R_2(t)$ as $N \to \infty$. Writing $S_2(t)$ in its singular value decomposition (SVD) form with singular values $\sigma_i(t)$ and singular vectors $u_i(t), v_i(t) \in \mathbb{R}^N$:
\begin{equation}
\label{low_rank}
S_2(t) = \sum_{i=1}^N \sigma_i(t) u_i(t) v_i^T(t) = \sum_{i=1}^r \sigma_i(t) u_i(t) v_i^T(t),
\end{equation}
the singular values $\sigma_i(t)$ satisfy:
\begin{align}
\sigma_i(t) &= 0 \quad (N \geq i > r) \quad \text{indicating low-rank nature of $S_b(t)$}\\
\sigma_i(t) &> \frac{\sqrt{\lambda_+}}{2}  \quad (1 \leq i \leq r) \quad \text{to ensure eigenvalues extend beyond $\lambda_+$}.
\end{align}
We assume that as $N \to \infty$, $r$ remains constant, ensuring the low rank of $S_2(t)$.
\end{assumption}

\begin{remark}
     For a detailed numerical analysis of  $a(N,s)$, see Subsection \ref{details_of_C1}.
\end{remark}

\subsection{A theorem on the asymptotic magnitude of the noise}

We consider the cross-entropy loss function with L2 regularization for a DNN $\phi(\cdot,\alpha(t))$, expressed as follows:
\begin{equation}
\label{loss_noise_det_1}
 L(\alpha(t))=-\frac{1}{|T|}\sum_{s\in T}\log\left(\phi_{C(s)}(s,\alpha(t))\right) + \mu \sum_{i=1}^L \|W_i(t)\|_F^2, \space  \mu >0,
\end{equation}
where $C(s)$ denotes the correct class of the object $s$. The term $\mu$ acts as a regularization hyperparameters and prevents overfitting in the DNN.

\begin{thm}
\label{main_result_reducing_noise}

Suppose Assumptions \ref{as1}, \ref{as2} and \ref{as3} hold. If the parameters $\alpha$ of the DNN converge as $t \to \infty$ to a local minimum of the loss \eqref{loss_noise_det_1}, denoted $\alpha^*(N)$, as $t\to\infty$, then $\forall \epsilon>0$:

\begin{equation}
\lim_{N \to \infty} \lim_{t\to \infty}\mathbb{P}(\|R_2\|_F \leq \epsilon) \to 1.
\end{equation}

\end{thm}
\noindent The proof of this theorem can be found in Section \ref{proof_main_theorem}.

The theorem states that the random part of the weight matrix $W_2$ vanishes as the network is trained and its parameters converge towards a local minimum ($\alpha(t) \to \alpha^*$). 
 It is impossible to decompose  $W_2$  into  $R_2$ + $S_2$ , where  $R_2$  satisfies Assumption 2 and  $S_2$  satisfies Assumption 3. If such a decomposition were possible, the ESD of  $W_2$  would exhibit an MP bulk and spikes (see \cite{baik2005phase, benaych2011eigenvalues, couillet2022random} and Appendices \ref{singualr_values_simple_case} and \ref{finding_lambda_updated}). However, at the local minimum, as we will show,  $W_2$ ’s ESD lacks these features, ruling out any such decomposition.

It should be noted that for a matrix $R$ with i.i.d, centered Gaussian entries, 
\begin{align}
     \frac{1}{N} \lVert R\rVert_2^2\overset{a.s.}{\underset{N\to\infty}{\sim}} 4 \ Var[R_{ij}]\label{eq:spec_conv}\\
    \lVert R\rVert_F^2\overset{a.s.}{\underset{N\to\infty}{\sim}}N^2Var[R_{ij}].\label{eq:frob_conv}
\end{align} Through Assumption 2, we assume i.i.d Gaussianity with mean zero, and $Var[R_{ij}]=\mathcal{O}(\frac{1}{N})$. However, we may see from (\ref{eq:frob_conv}) that this does not straightforwardly imply $\lVert R\rVert_F\to 0$, as the theorem proves. Hence, we prove that the random fluctuations of the weight matrices have to be much smaller than originally assumed.

 One might assume that a local minimum of the loss is deterministic, and thus the random part of the weights would have to vanish. While intuitive, this reasoning is not accurate. First, the theorem states that for \textbf{any} decomposition of $W_2$ into $R_2+S_2$ satisfying the above assumptions we would have that $\lim_{N \to \infty} \lim_{t \to \infty} 
 \mathbb{P}\left( \| R_2 \|_F \leq \epsilon \right) \to 1$. This means that at a local minimum of the loss if one takes a nonzero random matrix $R_2$ and then sets $S_2=W_2-S_2$, we would have that $S_2$ will \textbf{not} satisfy Assumption 3 (with some probability depending on $N$). Otherwise, the theorem would be violated. Furthermore, the training data $T$ itself might have noise, and so the loss landscape, which depends on the training data, might have some noise/randomness. That is, the local minima might be somewhat random (meaning that the parameters $\alpha$ at the local min might be random).  However, when we train with regularization, so long as we can decompose $W_2$ into $R_2+S_2$ satisfying Assumptions 2 and 3, we have that $\lim_{N \to \infty} \lim_{t \to \infty} 
 \mathbb{P}\left( \| R_2 \|_F \leq \epsilon \right) \to 1$.

The role of the regularization in \eqref{loss_noise_det_1} is well known in Machine Learning. In particular, it limits the amount of noise from the data being learned, as the DNN's weight may otherwise reflect the random nature of the data, see \cite{staats2022boundary}. In our theory, it is essential in showing that $\lVert R\rVert_F\to 0$. Thus regularization successfully mitigates randomness in the network. It is important to note that Assumption 2, which ensures that the deterministic matrix $S_2(t)$ has no small nonzero singular values, essentially is a statement about the data $T$. It says that any noise from the data cannot be (in some sense) "much larger" than the deterministic part of the data (relative also to the hyperparameter $\mu$).

 It is interesting to note that this phenomenon does not necessarily occur when there is no regularization.  In fact, without regularization, one can replace the random matrix $R_2$ with a different random matrix (or add any random matrix to $W_2$), and the magnitude of the loss will not change. Thus, around every local minima, there would be a flat region of minima, all corresponding with parameters $W_2+R^*_2$, with $R^*_2$  being some random matrix satisfying Assumptions 2 and 3.

\subsection{A theorem on how pruning randomness reduces loss}

We can also show that reducing the randomness in the parameters reduces the loss. Specifically, we show that for $W_2=R_2+S_2$, if we replace $W_2$ in the DNN with $S_2$ (at some fixed time) then with some probability (which goes to 1 as $N \to \infty$) we have that the loss of the DNN with $S_2$ is smaller than the loss of the DNN with $W_2$ by the amount $\mu\|R_2\|^2_F$. 

\begin{thm}
\label{theorem_reduction_of_loss}
Suppose  Assumptions \ref{as1}, \ref{as2} and \ref{as3} hold. For $W_2=R_2+S_2$,  remove $R_2$ so that the second weight layer matrix is $S_2$ alone. Then $\forall \epsilon>0$ and any fixed training time t we have

\begin{equation}
\lim_{N \to \infty}\mathbb{P}(L(\alpha_{S_2}) < (1+\epsilon) (L(\alpha_{W_2})-  \mu\|R_2\|^2_F))\to 1.
\end{equation}

  Furthermore $\exists G(N)\to 0$  as $N \to \infty$ s.t.
    
    \begin{equation}    \label{pruning_equation_2}
\lim_{N \to \infty}\mathbb{P}\bigg(|acc_{\alpha_{W_2}}(t)-acc_{\alpha_{S_2}}(t)| \leq G(N) \bigg)\to 1,
   \end{equation}
where $\alpha_{S_2}$ and $\alpha_{W_2}$ represent the DNNs with weight layer matrices $S_2$ and $W_2$ respectively, $L$ is the loss function in \eqref{loss_noise_det_1} and $\mu$ is the regularization 

\end{thm}
The proof for this theorem can be found in Subsection \ref{main_theorem_proof}. For a numerical simulation that shows that removing (or adding) a random matrix from the DNN reduces its loss, see Example \ref{adding_noise}. 
This theorem demonstrates that pruning random weights in a DNN reduces its loss. While Theorem \ref{main_result_reducing_noise} suggests that randomness vanishes near a local minimum, Theorem \ref{theorem_reduction_of_loss} provides a partial answer to whether removing randomness accelerates convergence, showing that it at least lowers the loss. The intuition behind this result lies in the perturbations caused by  $R_2$ : the cross-entropy term is not affected by adding or removing $R_2$, while the regularization term is influenced by its Frobenius norm and decreases as we remove $R_2$. Since the regularization perturbation dominates (as shown in \eqref{eq:frob_conv} and \eqref{eq:spec_conv}), removing  $R_2$  significantly reduces the loss. Theorem \ref{theorem_reduction_of_loss} formalizes this, showing that removing the random part  $R_2$  of the weight matrix  $W_2$  in the DNN  $\phi(\alpha, s)$  reduces the loss by an amount proportional to  $\|R_2\|_F^2$. Specifically, when  $W_2 = S_2$  (i.e.,  $R_2$  is removed), the loss satisfies:\begin{equation}
    L(\alpha_{S_2}, s) + \mu \|R_2\|^2_F < (1+\epsilon) L(\alpha_{W_2}, s),
\end{equation}
with high probability. This implies that the loss decreases by approximately  $\mu \|R_2\|_F^2$  when  $R_2$  is removed, as the regularization term involving  $R_2$  vanishes.

This effect can be extended by introducing a scaling factor  $\gamma \in [0, 1]$  to  $R_2$ , defining the modified weight matrix as:\begin{equation}
    W_2(\gamma) = \gamma R_2 + S_2.
\end{equation}
As  $\gamma$  decreases from 1 to 0, the randomness in  $W_2$  is reduced, and the corresponding loss decreases proportionally by $\gamma \mu \|R_2\|_F^2$ .

Thus, we can obtain the following corollary:

\begin{cor}
\label{cor:scaled_reduction_of_loss}
Suppose Assumptions  \ref{as1}, \ref{as2} and \ref{as3} hold. For any $\gamma\in[0,1]$, define the modified second-layer weight matrix by
\[
W_2(\gamma) = \gamma R_2 + S_2,
\]
and let $\alpha_{W_2(\gamma)}$ denote the corresponding DNN. Then for every $\epsilon>0$ and any fixed training time $t$, we have
\[
\lim_{N\to\infty} \mathbb{P}\Big( L(\alpha_{W_2(\gamma)}) < (1+\epsilon)\Big(L(\alpha_{W_2}) - \gamma\, \mu\, \|R_2\|_F^2\Big) \Big) \to 1.
\]

\end{cor}

Supporting evidence from \cite{staats2022boundary} shows that removing small singular values in weight matrices (reducing noise) improves accuracy for DNNs trained on noisy data. Additionally, numerical simulations (e.g., Example \ref{higher_acc_MP_based_training}) confirm that MP-based pruning reduces loss and increases accuracy, highlighting the practical benefits of this approach. See Appendices \ref{fully_conneced} and \ref{reg_problem} for numerical results on classification and regression problems that confirm these theoretical findings.

\subsection{Key lemmas: how removing randomness from DNN weight layers affects the output, loss and accuracy of a DNN}

To prove the previous theorems, we rely on two key lemmas presented here. These lemmas demonstrate that replacing  $W_2$  with its deterministic part has minimal effect on the DNN’s output, accuracy, and loss. Since the random component primarily acts as noise, its removal results in negligible changes.

 \begin{lem}
     
\label{main_result_remove _R}
   Suppose assumptions 1, 2, and 3, and suppose we replace the weight layer matrix $W_2$ with the deterministic matrix $S_2$. Then 
    \begin{equation}
        \label{pruning_equation_1}
\mathbb{P}\bigg(|X_i(s,\alpha_{S_2})-X_i(s,\alpha_{W_2})| \leq a(N,s) \bigg)\geq 1-2 \exp ( - \frac{N^{1/4}}{2} ).
   \end{equation}
   Here, $a(N,s)=$\eqref{eqans}, and $\alpha_{S_2}$ are the parameters of the DNN, which has the weight matrix $S_2$ and  $\alpha_{W_2}$ are the parameters of the DNN with the weight layer matrix $W_2$.

 \end{lem}
 For a proof of the Lemma see Subsection \ref{per_cor_proof_component_bound}. This lemma examines the impact of removing the random component from a weight matrix in a DNN layer. It shows that replacing  $W_2$  with  $S_2(t)$  alone results in a small difference in the DNN’s output with high probability. Intuitively, the random component  $R_2(t)$  acts as noise and has a diminishing effect in large networks due to averaging. The deterministic part  S$_2(t)$, which captures the network’s learned structure, dominates the behavior. As a result, removing  $R_2(t)$ has negligible impact on the output, reinforcing that the deterministic component sufficiently preserves the network’s function. To prove Theorem \ref{theorem_reduction_of_loss} we use this lemma and then show that the regularization part of the loss decreases by $\mu \|R_2\|_F$ when we replace $W_2$ with $S_2$.

\section{Generalized Main Theoretical Results}
\label{Main_theory_gen}

In this section, we generalize the previous results for DNNs with different architectures and for the case where the random matrix $R$ is not necessarily i.i.d Gaussian.

\subsection{Assumptions for Generalized Main Theorems}
\label{assumptions}

\begin{assumption}
    \label{as4}
Consider a DNN denoted by $\phi$, and assume that it can be written as:

\begin{equation}
    \phi=\rho \circ \psi_2\circ (R+S)\circ \psi_1,
\end{equation}
where $\psi_1$ and $\psi_2$ are arbitrary functions and $\rho$ is softmax. Furthermore, assume that $\exists C_1$ constant such that for any two arbitrary vectors $v$ and $w$ we have that $\|\psi_2v-\psi_2 w\|_{\infty} \leq C_1 \|\psi_2(v-w)\|_{\infty}$, with $\|\cdot\|_{\infty}$ the max norm of the vector.

\end{assumption}

Take $W=R+S$ and assume that the matrices $R$ and $S$ satisfying the following assumptions:

\subsection*{Assumptions on the matrices $R$ and $S$}
\label{assumptions_2}
  We considered a class of admissible matrices $W$, where $W=R+S$ and $W$, $R$ and $S$ satisfy the following three assumptions. The first  assumption is a condition on $R$: 

\begin{assumption}
    \label{as5}
Assume \(R\) is a \(N\times M\) matrix such that for any vector $v$ we have:

\begin{equation} 
\label{gen_ass_for_R}
\mathbb{P}( \|Rv\|_{\infty} > d_1(N)J(v) ) \leq d_2(N),
\end{equation}
with $d_1(N),d_2(N) \to 0$ as $N \to \infty$ and $J(v)$ some function that depends on $v$ alone. Further, as $N \to \infty$, we have that  $\sigma_{\max}(R) \to \sqrt{\lambda_+}$ a.s.    \\ 
\end{assumption}

\begin{ex}
    For example, when $R$ is i.i.d Gaussian, one can show, using the Borell-TIS inequality (see Theorem \ref{Borell-TIS Inequality}), that:
\[
\mathbb{P}\left( \|Rv\|_\infty > (\sqrt{\frac{2 \log N}{N}}  + \frac{1}{N^{\frac{3}{8}}})\| v\|_2 \right) \leq 2 \exp\left( -\frac{N^{\frac{1}{4}}}{2} \right)
\]

See Subsection \ref{per_cor_proof_component_bound} for a proof. 
\end{ex}

   We then assume the following for the matrix $S$:

   \begin{assumption}
  \label{as6}

   Assume $S$ is a matrix with $S=\sum_{i=1}^r \sigma_i u_i v^T_i=U \Sigma V^T$, with $\sigma_i$ the singular values and $u_i$, $v^T_i$  column and row vectors of $U$ and $V$. Thus, $S$ has $r$ non-zero singular values corresponding to the diagonal entries of $\Sigma$, and all other singular values of $S$ are zero. We also assume that these $r$ singular values of $S$ have multiplicity $1$.\\
   \end{assumption}

   Finally, we assume for $W:=R+S$:

   \begin{assumption}
   \label{as7}    
     Take $\sigma_i$ to be the singular values of $S$, with corresponding left and right singular vectors $u_i$ and $v^T_i$ and $\sigma'_i$ to be the singular values of $W=R+S$, with corresponding left and right singular vectors $u'_i$ and $v'^T_i$. First we assume that $\frac{N}{M} \to c \in (0,+\infty)$ as $N \to \infty$. Second, assume also that we know explicit functions $g_{\sigma_i,R}$, $g_{v_i,R}$ and $g_{u_i,R}$ such that as $N \to \infty$:
\end{assumption}

\begin{equation}\label{singualr_value_cases_assumption}
\sigma'_i(W) \xrightarrow[\text{}]{a.s.} 
\begin{cases}
  g_{\sigma_i,R} & \sigma_i>\bar{\theta}(\lambda_+)\\
  \sqrt{\lambda_+} & \sigma_i<\bar{\theta}(\lambda_+),
\end{cases}
\end{equation}

\

\begin{equation}\label{left_singualr_vectors_cases}
|<u'_i,u_i >|^2 \xrightarrow[\text{}]{a.s.} 
\begin{cases}
  g_{u_i,R} & \sigma_i>\bar{\theta}(\lambda_+)\\
  0 & \sigma_i<\bar{\theta}(\lambda_+),
\end{cases}
\end{equation}

and

\begin{equation}\label{right_singualr_vectors_cases}
|\langle v'_i,v_i \rangle|^2 \xrightarrow[\text{}]{a.s.}
\begin{cases}
  g_{v_i,R} & \sigma_i>\bar{\theta}(\lambda_+)\\
  0 & \sigma_i<\bar{\theta}(\lambda_+).
\end{cases}
\end{equation}

Third, also assume that 
 for $i \neq j$:

   \begin{equation}
   |<v'_i,v_j >|^2 \xrightarrow[\text{}]{a.s.} 0 
   \end{equation}

   and 

     \begin{equation}  
   |<u'_i,u_j >|^2 \xrightarrow[\text{}]{a.s.} 0. 
   \end{equation}

Here we take $\bar{\theta}(\lambda_+)$ to be a known explicit function depending on $\lambda_+$, for example see \eqref{theta_bar}. For the function $g_{\sigma_i,R}$ we assume the following: if $\sigma_i>\bar{\theta}(\lambda_+)$ then $\sigma'_i(W)> \lambda_+$ a.s., furthermore, $\sigma'_i(W)>\sigma_i$ a.s. Finally, for $\gamma$ a variable between $0$ and $1$, take $W=\gamma R +S$ we have that $ \sigma'_i(W) \to \sigma_i$ monotonically, a.s., as $\gamma \to 0$.

Empirically, it has been observed that these assumptions are reasonable for weight matrices of a DNN; see \cite{thamm2022random,staats2022boundary}.  There are various spiked models in which Assumptions \ref{as5}-\ref{as7}  hold, for more on the subject see \cite{baik2005phase, benaych2011eigenvalues,dharmawansa2022eigenvectors,couillet2022random, bao2021singular, o2018random, agterberg2022entrywise, chen2021asymmetry, bao2022eigenvector, leeb2021matrix, zhang2020tracy,dharmawansa2022eigenvectors,o2018matrices}. Also, a number of works in RMT addressed the connection between a random matrix $R$ and the singular values and singular vectors of the deformed matrix $W=R+S$, see \cite{benaych2011eigenvalues2, benaych2011eigenvalues}. 
%The analysis of the probabilistic manner in which a random matrix $R$ affects the singular values and singular vectors of a matrix $W=R+S$ is indeed a very large area of study in random matrix theory with many useful results. 

 These assumptions are quite natural and hold for a wide range of DNN architectures. Assumption \ref{as5} focuses on the random matrix $R$. This assumption ensures that the random matrix $R$ captures the essential randomness in the weight layer while also satisfying the requirements given in Theorem \ref{RMT_MP_theorem}. 

Assumption \ref{as6}-\ref{as7} pertains to the deterministic matrix $S$, which is assumed to have a specific structure, with $r$ non-zero singular values and all other singular values being zero. Moreover, these $r$ singular values have multiplicity 1, which is a reasonable expectation for a deterministic matrix that contributes to the information content in the weight layer matrix $W$. 

The assumption that the singular values of the deterministic matrix $S$ are larger than some $\bar{\theta}(\lambda_+)$ is also quite natural, see \cite{thamm2022random,staats2022boundary}. This is because the deterministic matrix $S$ represents the information contained in the weight layer, and its singular values are expected to be large, reflecting the importance of these components in the overall performance of the DNN. On the other hand, the random matrix $R$ captures the inherent randomness in the weight layer, and with high probability depending on $N$, its singular values should be smaller than the MP-based threshold. This means that there is a clear boundary between the information and noise in the layer $W$, which is also natural, see \cite{staats2022boundary}. 

%Note, if $\sigma'_i(W) \to \sigma_i(S)$ when $\sigma_i(S) \to \infty$ for all non-zero singular values of $S$, we have that $\bar{\theta}(\lambda_+) \to \sqrt{\lambda_+}$ as $\sigma(S) \to \infty$ for the non-zero singular values of $S$. 

This distinction between the singular values of $S$ and $R$ highlights the separation between the information and noise in the weight layer, allowing us to effectively remove the small singular values without impacting the accuracy of the DNN. The assumption thus provides a solid basis for studying the behavior of DNNs with weight layers modeled as spiked models. It contributes to our understanding of the effects of removing small singular values based on the random matrix theory MP-based threshold $\sqrt{\lambda_+}$.

One can show that the following two simpler properties on the matrices $R$ and $S$ are sufficient to ensure that $R$ and $S$ satisfy the above Assumptions \ref{as5}-\ref{as7}. 

Recall that a bi-unitary invariant random matrix $R$ is a matrix with components taken from i.i.ds such that for any two unitary matrices  $U$ and $V^T$, the components of the matrix $URV^T$ have the same distribution as the components of $R$. We then assume:

\textbf{Property 1} (statistical isotropy):
Assume \(R\) to be a bi-unitary invariant random \(N\times M\) matrix with components taken from i.i.ds with zero mean and variance $\frac{1}{N}$.      

   %\textbf{Assumption 2}: Assume that the distribution $\mu_{R_M}$, given in Def. \ref{ESD_Definition}, converges almost surely, weakly,  to a non-random compactly supported probability measure \(\mu_R\), as $N \to \infty$ and $\frac{N}{M} \to c\in[0,1]$.  Let \(a\),\(b\) be the infimum and supremum of the support of \(\mu_{R_M}\) respectively, and the smallest and largest singular values of \(R\) converge almost surely to \(a\) and \(b\) respectively.}

   We then assume the following for the deterministic matrix $S$:

   \textbf{Property 2} (low rank of deterministic matrix): Assume $S$ is a deterministic matrix with $S=\sum_{i=1}^r \sigma_i u_i v^T_i=U \Sigma V^T$, with $\sigma_i$ the singular values and $u_i$, $v^T_i$  column and row vectors of $U$ and $V$. Thus, $S$ has $r$ non-zero singular values contained on the diagonal entries of $\Sigma$, and all other singular values are zero. We also assume that these $r$ singular values of $S$ have multiplicity $1$. Finally, we assume that $\frac{N}{M} \to c \in (0,+\infty)$ as $N \to \infty$. 

   An explicit relationship between Properties 1-2 and Assumptions \ref{as5}-\ref{as7} can be found in \cite{benaych2011eigenvalues}. The Property $1$ is indeed strong, as it implies that the random matrix $R$ is random in every direction. In other words, for any unitary matrices $U$ and $V^T$, the matrix $URV^T$ has the same distribution as $R$. Random matrices with complex Gaussian entries, also known as Ginibre matrices, are a class of random matrices that are bi-unitary invariant  \cite{kosters2015limiting}.

For a DNN $\phi$ satisfying Assumption \ref{as4} we start by defining, 

\begin{equation}
\label{from_DNN_st}
    g_{\phi}(s):=C_1\|\psi_1s\|_2  \|\psi_2\|_2,
\end{equation}
where $\|\cdot\|_2$ is the induced $l_2$ operator norm (for the case $\|\psi_2\|_2$) and $C_1$ comes from Assumption \ref{as4}. Note that $\psi_1s$ is simply a vector. We also define,

\begin{equation}
\label{from_DNN_st_2}
    h_{\phi}(s):=C_1J(\psi_1s)  \|\psi_2\|_1,
\end{equation}
where $\|\cdot\|_1$ is the induced $l_1$ operator norm and $J(\cdot)$ is given in \eqref{gen_ass_for_R}.

\subsection{Main Theoretical results for RMT-based decrease in loss}
\label{decrease_in_loss}

In the context of analyzing DNNs, we consider the cross-entropy loss function with L2 regularization for a DNN $\phi(\cdot,\alpha(t))$, expressed as follows:
\begin{equation}
\label{loss_noise_det_2}
 L(\alpha(t))=-\frac{1}{|T|}\sum_{s\in T}\log\left(\phi_{C(s)}(s,\alpha(t))\right) + \mu \sum_{i=1}^L \|W_i(t)\|_F^2,
\end{equation}
where $W_i$ are the weight layer matrices of the DNN, and $C(s)$ denotes the correct class of the object $s$. The term $\mu$ acts as a regularization constant, which is employed to prevent overfitting in the DNN.

\begin{thm}
\label{main_result_reducing_noise_general}

Let $\phi$ be a DNN satisfying Assumptions \ref{as4}-\ref{as7}. Assume the following:

\begin{enumerate}
    \item $\forall s \in T$, the quantity $a_{\phi}(N,s):=  
    h_{\phi}(s)d_1(N) \to 0$ as $N \to \infty$ (with $ h_{\phi}(s)$ defined in \eqref{from_DNN_st_2} and $d_1(N)$ coming from \eqref{gen_ass_for_R} see Subsection \ref{details_of_C1} for more details). 
    \item The loss function \eqref{loss_noise_det_1} attains a local minimum at $\alpha^*(N)$, where $N$ is the size of the matrix $W$.
    \item As training time $t \to \infty$, the parameter $\alpha(t,N)$ converges to $\alpha^*(N)$.
\end{enumerate}

Then, $\forall \epsilon>0$:

\begin{equation}
\lim_{N \to \infty} \lim_{t\to \infty} 
 \mathbb{P}(\|R\|_F \leq \epsilon) \to 1. 
\end{equation}

\end{thm}

\noindent The proof of this theorem can be found in Appendix \ref{proof_main_theorem}.

We can also show that reducing the randomness in the parameters reduces the loss:

\begin{thm}
\label{theorem_reduction_of_loss_general}
Assume the setup and conditions of Theorem \ref{main_result_reducing_noise}, particularly regarding the DNN $\phi$ satisfying Assumptions \ref{as4}-\ref{as7}. Assume that $\forall s \in T$, the quantity $a_{\phi}(N,s):=  h_{\phi}(s)d_1(N) \to 0$ as $N \to \infty$ (see Subsection \ref{details_of_C1} for more details). 

Let $R$ be removed such that the second weight layer matrix is $S$ alone. Then $\forall \epsilon>0$ and for any fixed time $t$, we have

\begin{equation}
\lim_{N \to \infty}\mathbb{P}(L(\alpha_{S}) < (1+\epsilon) (L(\alpha_{W})-  \mu\|R\|^2_F))\to 1
\end{equation}

   Furthermore,  $\exists G(N) \to 0$    as $N \to \infty$ s.t.
    
    \begin{equation}    \label{pruning_equation_2_gen}
\lim_{N \to \infty}\mathbb{P}\bigg(|acc_{\alpha_{W}}(t)-acc_{\alpha_{S}}(t)| \leq G(N) \bigg)\to 1,
   \end{equation}
where $\alpha_{S}$ and $\alpha_{W}$ are the parameters of the DNNs with weight layer matrices $S$ and $W$ respectively, $L$ is the loss function in \eqref{loss_noise_det_1} and $\mu$ is the regularization hyperparameter.

\end{thm}

\subsection{Key Lemmas: how removing randomness from DNN weight layers affects the output, loss and accuracy of a DNN}

In order to prove the main result, we will consider a few critical lemmas. These lemmas describe how the output, accuracy, and loss of a DNN remain largely unaffected when we replace the matrix $W_2$ with its deterministic part only. In other words, the random component of $W_2$ does not significantly affect the loss, accuracy, and output of the DNN. This outcome should be expected, as the random part should act more like noise, without causing substantial changes in the loss and accuracy. The following lemmas illustrate this concept.

\textbf{Recall:}
We define the final output of the DNN before softmax as: 
\begin{equation}
\label{final_before_out}
X = \psi_2\circ (R+S)\circ \psi_1(s).    
\end{equation}

 \begin{lem}
     
\label{main_result_remove _R_general}
   Given \eqref{final_before_out}, let $W+R+S$ be a $N \times N$ matrix such that $W(t)=R(t)+S(t)$, with $R(t)$ a random matrix with components i.i.d. and $S(t)$ a deterministic matrix satisfying Assumptions \ref{as4}-\ref{as7}.
  
    Suppose we replace the weight layer matrix $W$ with the deterministic matrix $S$. Then 
    \begin{equation}
        \label{pruning_equation_1}
\mathbb{P}\bigg(|X_i(s,\alpha_{S})-X_i(s,\alpha_{W})| \leq d_1(N)h_{\phi}(s) \bigg)\geq 1-d_2(N) ).
   \end{equation}
   See \eqref{gen_ass_for_R} for the details on $d_1(N)$ and $d_2(N)$, and \eqref{from_DNN_st_2} for the details of $h_{\phi}(s)$. Here, $\alpha_{S}$ are the parameters of the DNN, which has the weight matrix $S$ and  $\alpha_{W}$ are the parameters of the DNN with the weight layer matrix $W$.

 \end{lem}

 For a proof of the Lemma see Subsection \ref{per_cor_proof_component_bound}.  
   
\subsection*{Acknowledgments}
LB, HO and YS acknowledge support from NASA via the AIST program (Kernel Flows: Emulating Complex Models for Massive Data Sets). This work started when LB was on a sabbatical stay at Caltech hosted by H. Owhadi. Both LB and YS are grateful for the hospitality during their visit, supported by NASA.   

The work of LB was partially supported by NSF grant DMS-2005262 and NSF grant IMPRESS-U
2401227. TB and HO acknowledge support from the Air Force Office of Scientific Research under MURI awards number FA9550-20-1-0358 (Machine Learning and Physics-Based Modeling and Simulation), FOA-AFRL-AFOSR-2023-0004 (Mathematics of Digital Twins), by the Department of Energy under award number DE-SC0023163 (SEA-CROGS: Scalable, Efficient, and Accelerated Causal Reasoning Operators, Graphs and Spikes for Earth and Embedded Systems). Additionally, HO acknowledges support from the DoD Vannevar Bush Faculty Fellowship Program.

\bibliographystyle{plainnat}

\bibliography{DNNstability_biblio} 

\appendix

\section{Other Numerical Results}
\label{fully_conneced}

This appendix provides some numerical results which shows that MP-based pruning leads to a reduction in loss and an increase in accuracy, which is relevant to Theorem \ref{theorem_reduction_of_loss}. We also confirm that adding (which in this context is the same as removing) and random matrix to the weight layer matrices of the DNN does not lead to a change in accuracy or cross-entropy loss but does lead to a change in the $L_2$ loss, numerically confirming Theorem \ref{theorem_reduction_of_loss} and Lemma \ref{main_result_remove _R}.

\begin{ex}
    \label{higher_acc_MP_based_training}
      In this example, we trained on Fashion MNIST a fully connected DNN with ReLU activation and inner dimensions as the following sequence, which we call its topology:$$[784,3000,3000,3000,3000,500, 10],$$ This means our network has 6 layers followed by activation functions, with layer $k$ being a linear layer $L_k:\mathbb{R}^{T_k}\to\mathbb{R}^{T_{k+1}}$. The training was performed multiple times, with no regularization and no PM-based pruning, with $L_1$ regularization, $L_2$ regularization, $L_1+L_2$ regularization and MP pruning (see Algorithm \ref{algo_1}) and so on to compare these different forms of training. We also compare our approach with stable rank regularization, that is adding the sum over $l$ of:
      \begin{equation}
         \text{stable}(W_l) = \frac{\|W_l\|^2_F}{ \|W_l\|^2_2}\end{equation}
      to the loss function, see \cite{xiao2023heavy} for more. This regularization is supposed to help induce heavy tail in the weight layer matrices, making the weight layer matrices "less random". 
      
      In all cases, the training was done for $100$ epochs. MP-pruning based training essentially involved pruning the small singular values in the weight layer matrices based on the MP distribution, and using this reduction in rank to reduce the number of parameters in the DNN, see Appendix \ref{algo_1} and \cite{shmalo2023deep, berlyand2023enhancing} for more information.  The ReLU activation function was applied after every layer, including the final layer (which is not typical). While it might be easier to train DNNs without applying the activation function to the final layer, we found that we obtained the highest accuracies when training with the affirmation structure while using MP-based pruning. For example, using $L1+L2$ regularization and MP-pruning, we obtained a $91.12\%$ accuracy on the Fashion MNIST test set, which is the highest accuracy we observed on the data set using a fully connected DNN (see Example \ref{higher_acc_MP_based_training_2}). Without the activation function being applied to all layers, the accuracy was less than $90\%$ (both with or without MP-pruning). MP-based pruning also increases the accuracy of fully connected DNNs that do not have an activation function on the final layer; for example, see Subsection \ref{reg_problem}.

      Every $4$ epochs, we kept a portion of  the smallest singular values (singular values less than $\sqrt{\lambda_+}$ given in Assumption $2$ in Subsection \ref{assumptions}) based on the formula:

\begin{equation}
\label{slow_prune}
   f(\text{{epoch}}) = \max\left(0, -\frac{1}{200} \cdot \text{{epoch}} + 1\right).
\end{equation}

By doing this, we are essentially removing "some" of the random matrices $R_l$ during training. Once the training is finished, we sparsify the weight layer matrices of the DNN by removing the weights smaller than some threshold $\xi$, see Algorithm \ref{spars_algo}. The results are given in  Fig. \ref{fig:result_after_spar} and Table \ref{full_connected_acc_results}.  In \cite{berlyand2023enhancing}, the authors showed that one can reduce the number of parameters by a much larger amount (a reduction in parameters by up to $99.98\%$) if we remove 

\begin{equation}
   f(\text{{epoch}}) = \max\left(0, -\frac{1}{100} \cdot \text{{epoch}} + 1\right)
\end{equation}
of the smallest singular values and sparsify the matrix once the training is complete. However, the accuracy that is obtained is also lower (see \cite{berlyand2023enhancing} for more on the relationship between MP-based pruning and sparsification).  In Fig. \ref{fig:result_after_spar}, we show how accuracy drops as we sparsify the DNNs.

      \begin{table}[ht]
\centering
\begin{tabular}{|c|c|c|}
\hline
Training method & Final accuracy on training set &  Final accuracy on testing set 

\\ \hline
\cancel{MP-based pruning}, \cancel{$L_1$}, \cancel{$L_2$}       & 78.03\%                      & 71.64\%                     \\ \hline
MP-based pruning, \cancel{$L_1$}, \cancel{$L_2$}         & 88.12\%                      & 81.22\%                     \\ \hline
MP-based pruning, $L_1$, $L_2$        &  \textbf{99.82\%}                       & \textbf{90.53\%}

\\ \hline
\cancel{MP-based pruning}, $L_1$, $L_2$        & 87.57\%                       & 81.70\%   

\\ \hline
\cancel{MP-based pruning}, $L_1$, \cancel{$L_2$}        & 87.42\%                       & 81.46\%

\\ \hline
\cancel{MP-based pruning}, \cancel{$L_1$}, $L_2$        & 87.37\%                       & 81.31\%

\\ \hline
MP-based pruning, $L_1$, $L_2$, stable rank         & \textbf{99.97\%}                       & 90.16\%  

\\ \hline
\cancel{MP-based pruning}, $L_1$, $L_2$, stable rank         & 79.01\%                       & 71.38\%  

\\ \hline
MP-based pruning, $L_2$         & \textbf{99.98}\%                       & 90.16\%

\\ \hline
\end{tabular}
\caption{Performance of fully connected DNN at final epoch for various training strategies.}
\label{full_connected_acc_results}
\end{table}

   \begin{table}[ht]
\centering
\begin{tabular}{|c|c|}
\hline
Training method & Final loss on training set 

\\ \hline
\cancel{MP-based pruning}, \cancel{$L_1$}, \cancel{$L_2$}       &  0.629775                                       \\ \hline
MP-based pruning, \cancel{$L_1$}, \cancel{$L_2$}         & 0.170905                                        \\ \hline
MP-based pruning, $L_1$, $L_2$        & \textbf{0.003104}

\\ \hline
\cancel{MP-based pruning}, $L_1$, $L_2$        & 0.325156                       

\\ \hline
\cancel{MP-based pruning}, $L_1$, \cancel{$L_2$}        &  0.399908

\\ \hline
\cancel{MP-based pruning}, \cancel{$L_1$}, $L_2$        & .431001

\\ \hline
MP-based pruning, $L_1$, $L_2$, stable rank         & \textbf{.001134}                      

\\ \hline
\cancel{MP-based pruning}, $L_1$, $L_2$, stable rank         & .476791                     

\\ \hline
MP-based pruning, $L_2$         & \textbf{0.001588}

\\ \hline
\end{tabular}
\caption{Performance of fully connected DNN at final epoch for various training strategies.}
\label{full_connected_acc_results}
\end{table}

The hyperparameters for the training are given in Table \ref{hyper_for_full_connected_training}.

\begin{table}[h]
\centering
\begin{tabular}{|l|l|}
\hline
\textbf{hyperparameters}     & \textbf{Value}      \\ \hline
Stable rank (SR)           & 0.000001
 \\ \hline
Every how many epochs the SR was used in the loss             & 6              \\ \hline
$L_1$          & 0.000001            \\ \hline
$L_2$           & 0.000001            \\ \hline
Every how many epochs the MP pruning was done           & 4            \\ \hline

Momentum           & 0.9           \\\hline

Learning rate          & 0.01 (decay by .96 every 4 epochs)          \\\hline 
\end{tabular}
\caption{hyperparameter values}
\label{hyper_for_full_connected_training}
\end{table}

      \begin{figure}[ht]
    \centering
    \begin{subfigure}[b]{0.3\textwidth}
        \includegraphics[width=\textwidth]{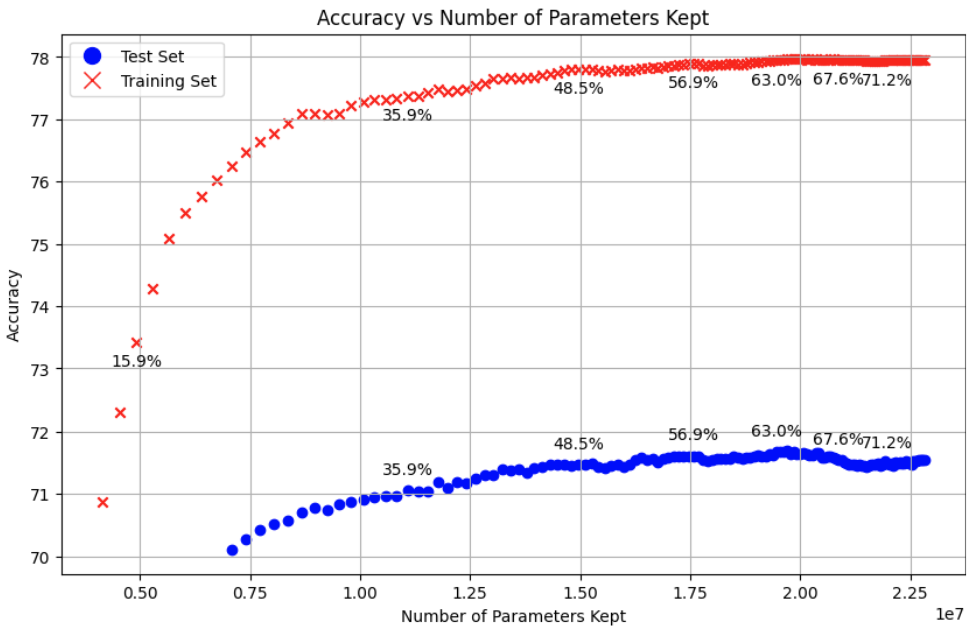}
        \caption{No MP-based pruning, and no $L_1$ or $L_2$ regularization.}
        \label{fig:sub1}
    \end{subfigure}
    \hfill
    \begin{subfigure}[b]{0.3\textwidth}
        \includegraphics[width=\textwidth]{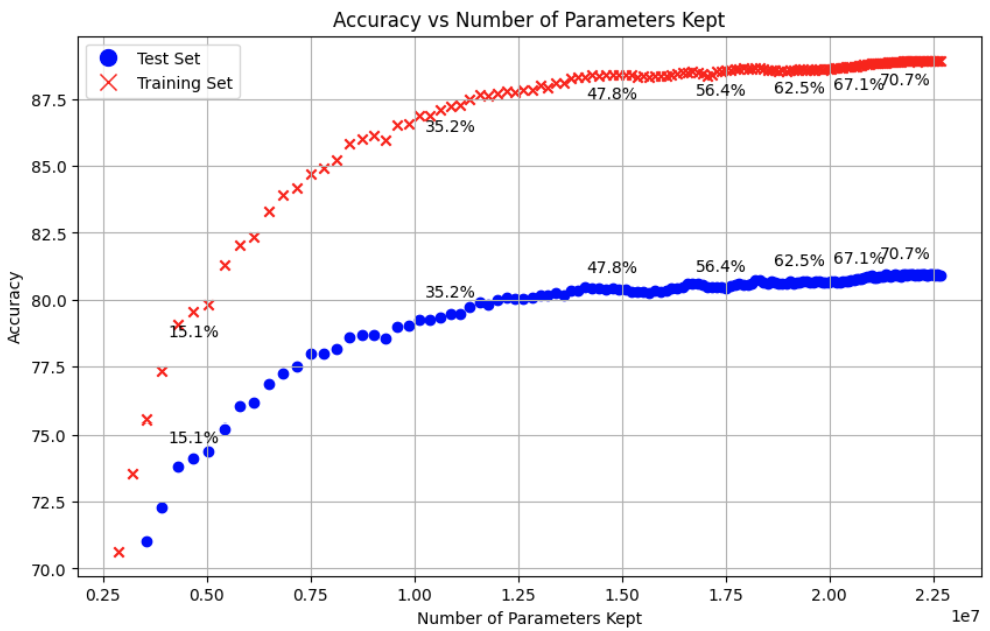}
        \caption{MP--based pruning, no $L_1$ or $L_2$ regularization.}
        \label{fig:sub2}
    \end{subfigure}
    \hfill
    \begin{subfigure}[b]{0.3\textwidth}
        \includegraphics[width=\textwidth]{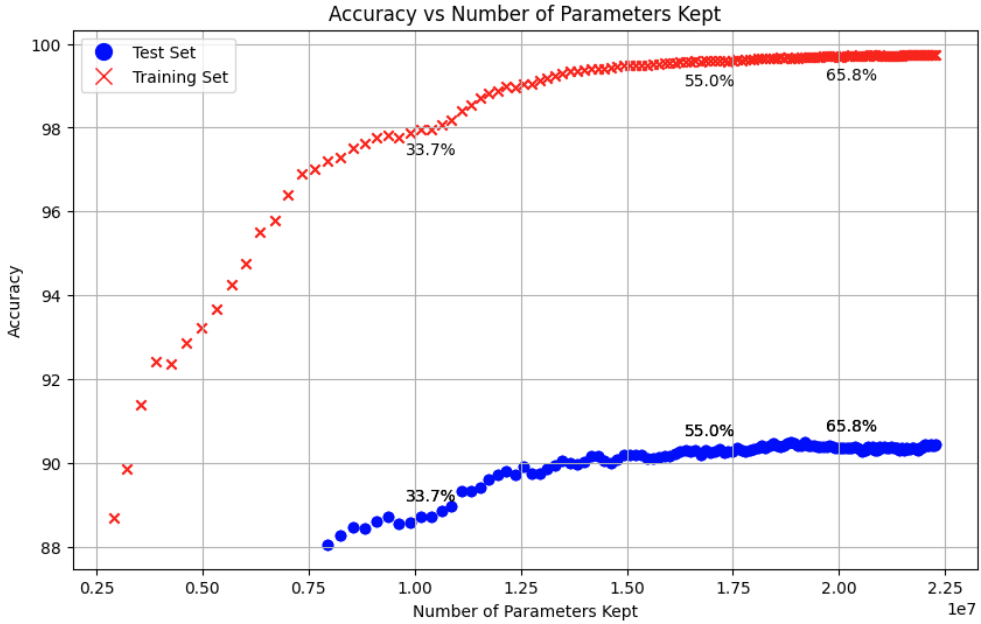}
        \caption{MP-based training + $L_1$ and $L_2$ regularization.}
        \label{fig:sub3}
    \end{subfigure}
    \newline
    \begin{subfigure}[b]{0.3\textwidth}
        \includegraphics[width=\textwidth]{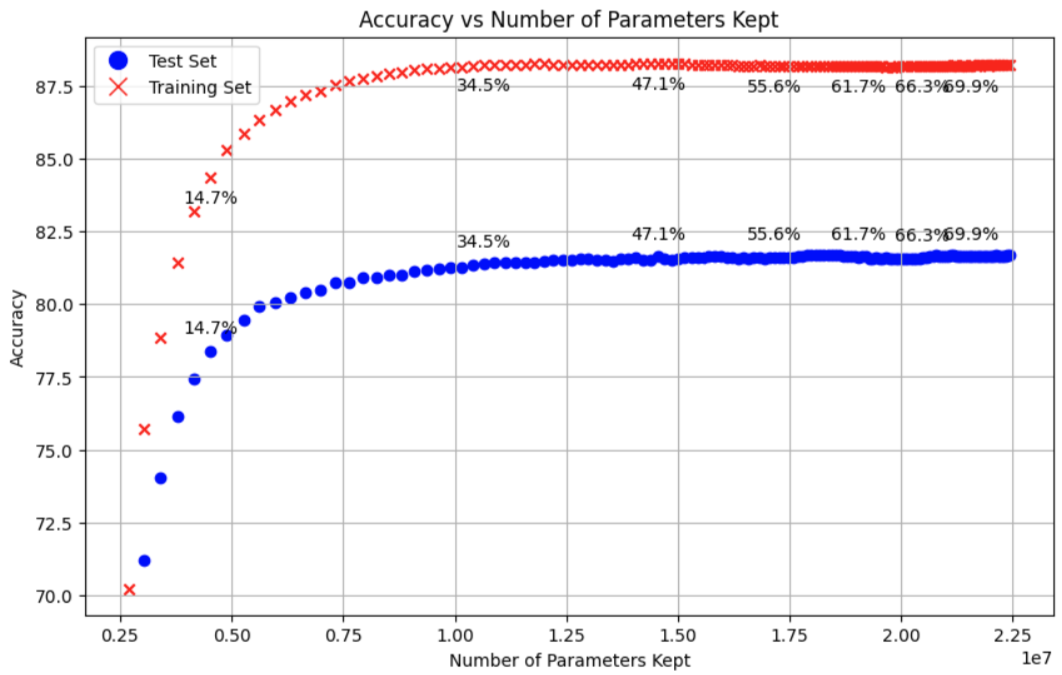}
        \caption{No MP-based pruning, used $L_1 + L_2$ regularization.}
        \label{fig:sub4}
    \end{subfigure}
    \hfill
    \begin{subfigure}[b]{0.3\textwidth}
        \includegraphics[width=\textwidth]{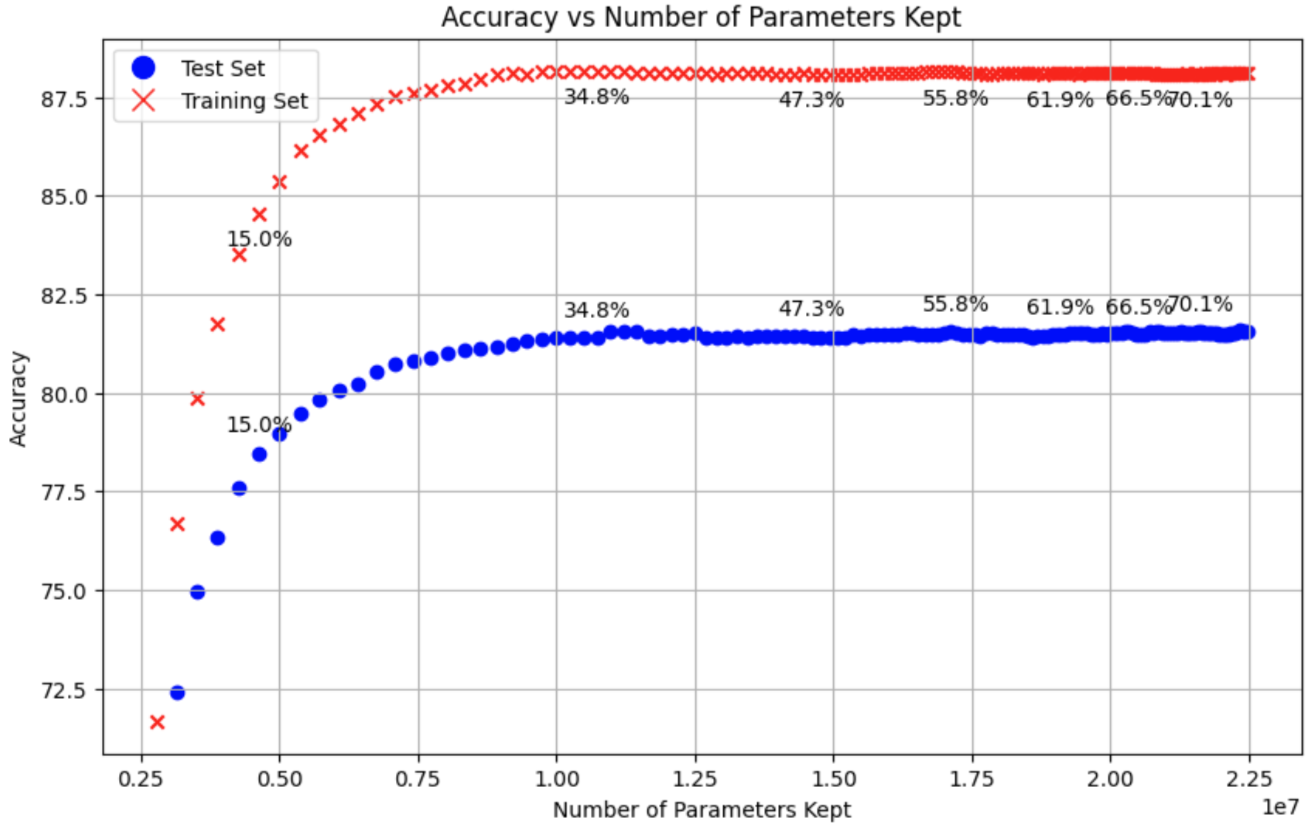}
        \caption{Only used $L_1$ regularization.}
        \label{fig:sub5}
    \end{subfigure}
    \hfill
    \begin{subfigure}[b]{0.3\textwidth}
        \includegraphics[width=\textwidth]{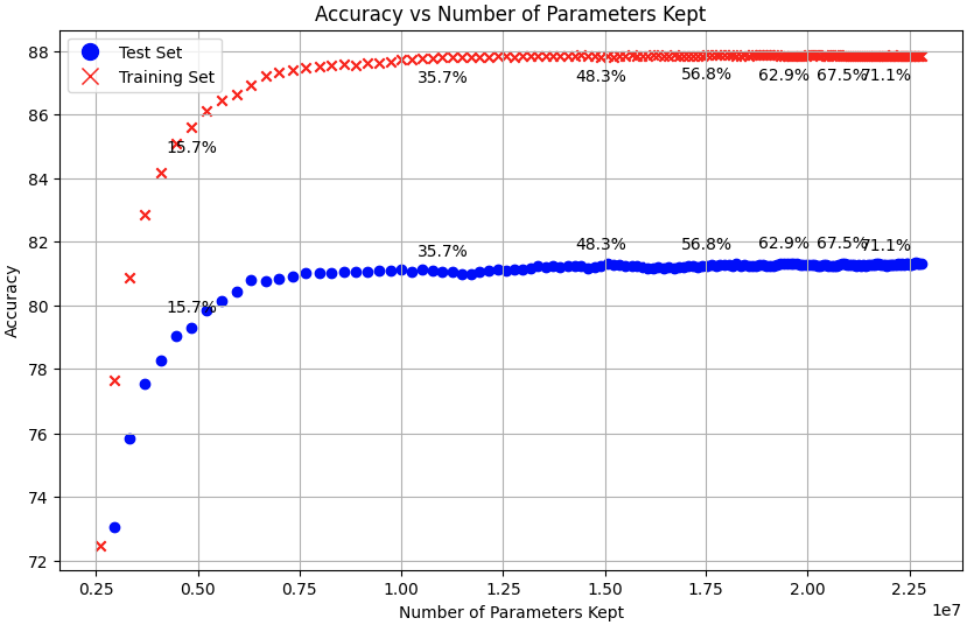}
        \caption{Only used $L_2$ regularization.}
        \label{fig:sub6}
    \end{subfigure}    
    \hfill
    \begin{subfigure}[b]{0.3\textwidth}
        \includegraphics[width=\textwidth]{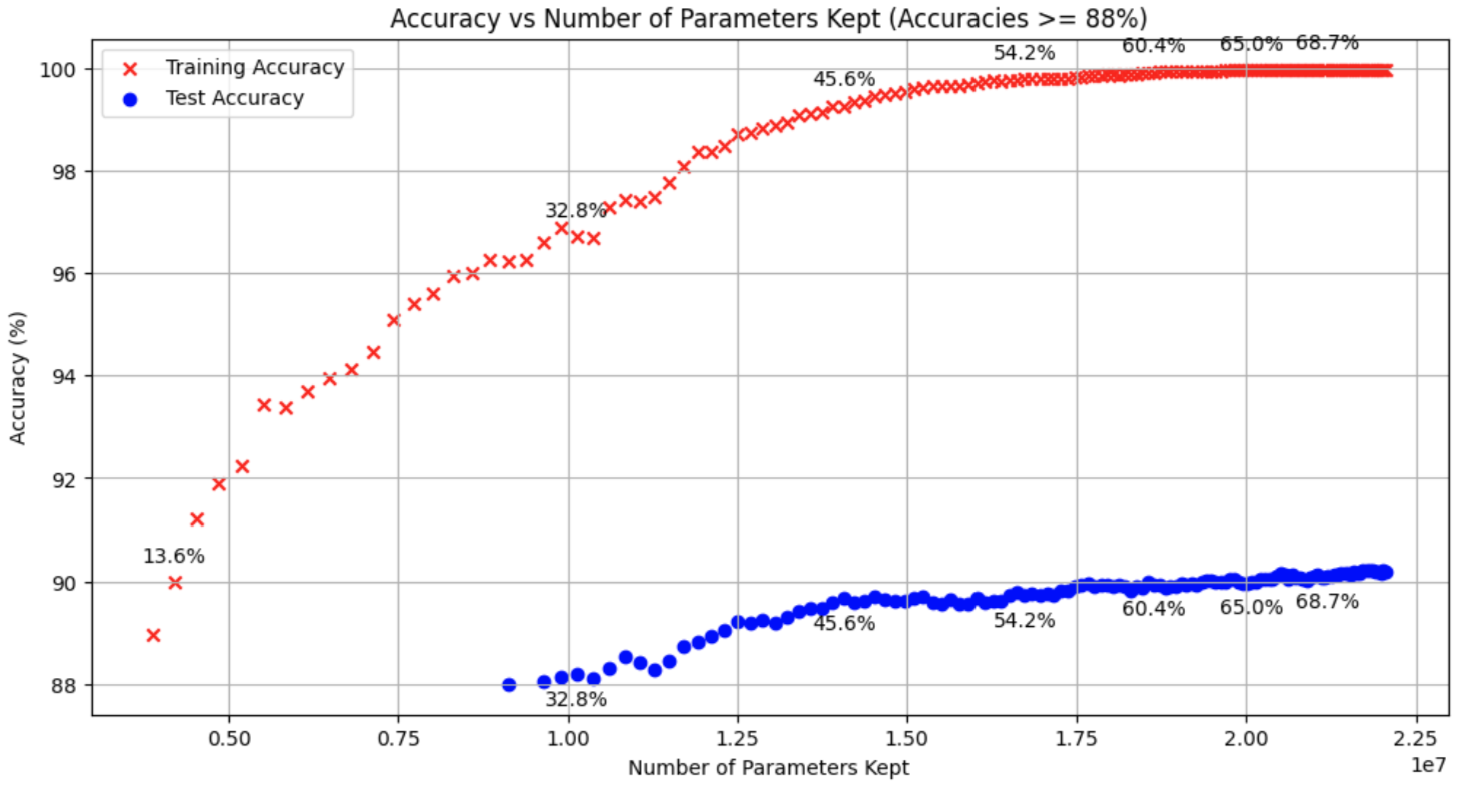}
        \caption{MP-based pruning, $L_1$, $L_2$, stable rank.        }
        \label{fig:sub6}
    \end{subfigure}
\hfill
    \begin{subfigure}[b]{0.3\textwidth}  \includegraphics[width=\textwidth]{MP_pruning_with_RMT_reg.png}
        \caption{$L_1$, $L_2$ and stable rank.        }
        \label{fig:sub6}
    \end{subfigure}
    \hfill
    \begin{subfigure}[b]{0.3\textwidth}  \includegraphics[width=\textwidth]{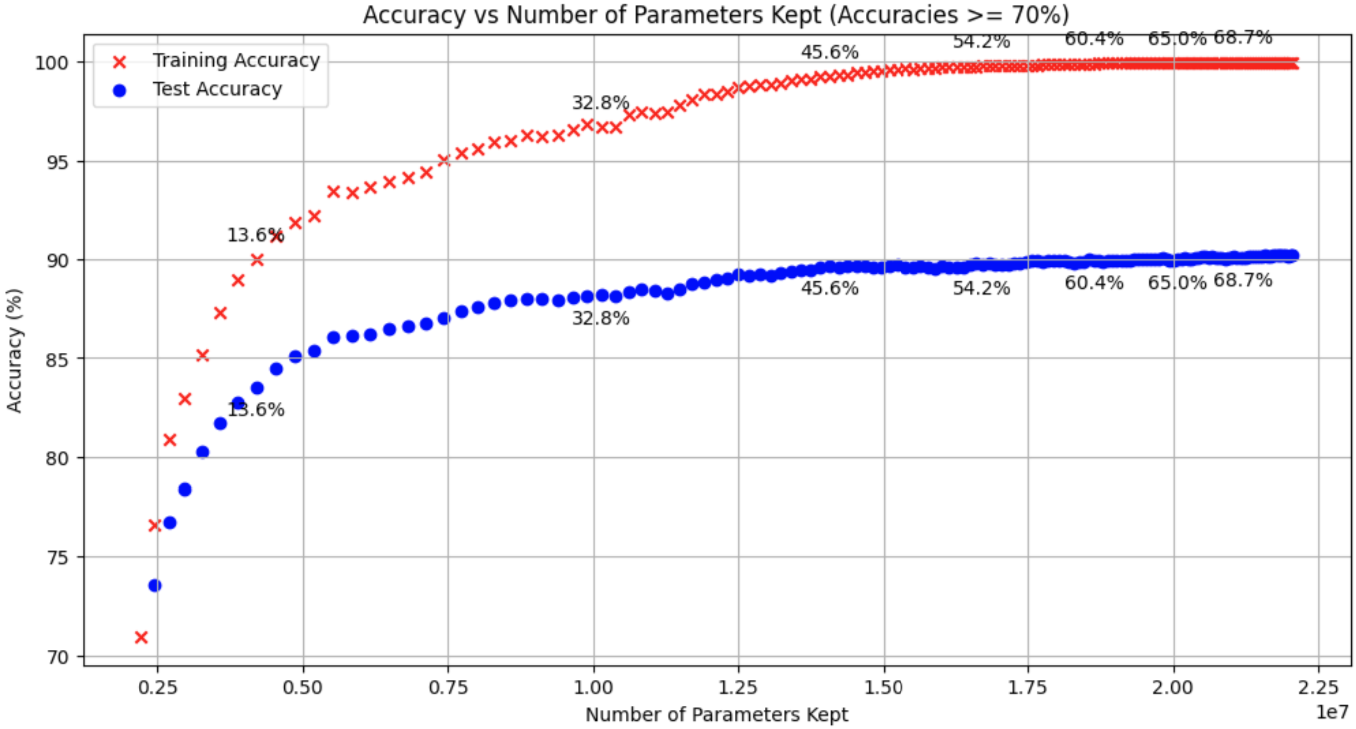}
        \caption{MP-based pruning with $L_2$.        }
        \label{fig:sub6}
    \end{subfigure}
    \hfill
    \caption{DNN performance vs. sparsification}
    \label{fig:result_after_spar}
   
\end{figure}
    \end{ex}

\begin{ex}
    \label{higher_acc_MP_based_training_2}
      In this example, we trained on Fashion MNIST a fully connected DNN with ReLU activation (on all but the final layer) and topology:$$[784, 5000, 5000, 5000, 5000,5000,5000,5000, 10].$$ This means our network has 7 layers followed by activation functions (excluding the final layer), with layer $k$ being a linear layer $L_k:\mathbb{R}^{T_k}\to\mathbb{R}^{T_{k+1}}$. The training was performed with $L_1+L_2$ regularization and PM-based pruning (see Algorithm \ref{algo_1}) and compared with using regularization alone. 
      
      In all cases, the training was done for $1000$ epochs. The ReLU activation function was applied after every layer, excluding the final layer (which is typical). Without MP-based pruning, the accuracy on the test set was lower than $90\%$. For the MP-based pruning, every $12$ epochs we kept a portion of  the smallest singular values (singular values less than $\sqrt{\lambda_+}$ given in Assumption $2$ in Subsection \ref{assumptions}) based on the formula:

\begin{equation}
\label{slow_prune}
   f(\text{{epoch}}) = \max\left(0, -\frac{1}{1000} \cdot \text{{epoch}} + 1\right).
\end{equation}

We used SGD together with a momentum of .97 and a learning rate of $.001$ with the CosineAnnealingLR scheduler. The L1 and L2 regularization hyperparameters were  $0.0000005$. Using MP-based pruning, the test set accuracy was $91.37$.
\end{ex}

\begin{ex}
\label{adding_noise}

In this example, we trained on Fashion MNIST a fully connected DNN with relu activation and  topology:
$$T=[784,3000,3000,3000,3000,500, 10],$$
This means our network has 6 layers followed by activation functions, with layer $k$ being a linear layer $L_k:\mathbb{R}^{T_k}\to\mathbb{R}^{T_{k+1}}$. We use the MP-pruning approach that can be found in Example \ref{higher_acc_MP_based_training}, to achieve a $90.5\%$ accuracy on the test set. Note that in Lemma \ref{main_result_remove _R_loosandacc}, the method for training the DNN does not matter; only the loss and accuracy of the DNN at time $t$ is what matters. Furthermore, this lemma is equally valid if instead of replacing $W_2$ with $S_2$, we replace the weight layer matrix $W_2$ with a new matrix $W_2+R'_2$, with $R'_2$ a random matrix with components i.i.d. taken from $N(1,\epsilon)$ with $\epsilon \sim \frac{1}{N}$. Thus, to all but the final layer matrix of the DNN we add a random matrix $R'_b$ with components i.i.d. from the distribution $N(0,\epsilon)$ and, in Fig. \ref{acc_loss_noise}, we plot the loss (both with and without L2 regularization) and accuracy of the DNN as we vary $\epsilon$. We see that for $\epsilon \sim \frac{1}{N}$ we have that the cross-entropy loss and accuracy don't change much, which numerically confirms Lemma \ref{main_result_remove _R_loosandacc}.

\begin{figure}[h]
    \centering
    \begin{subfigure}{0.41\textwidth}
        \includegraphics[width=1.3\textwidth]{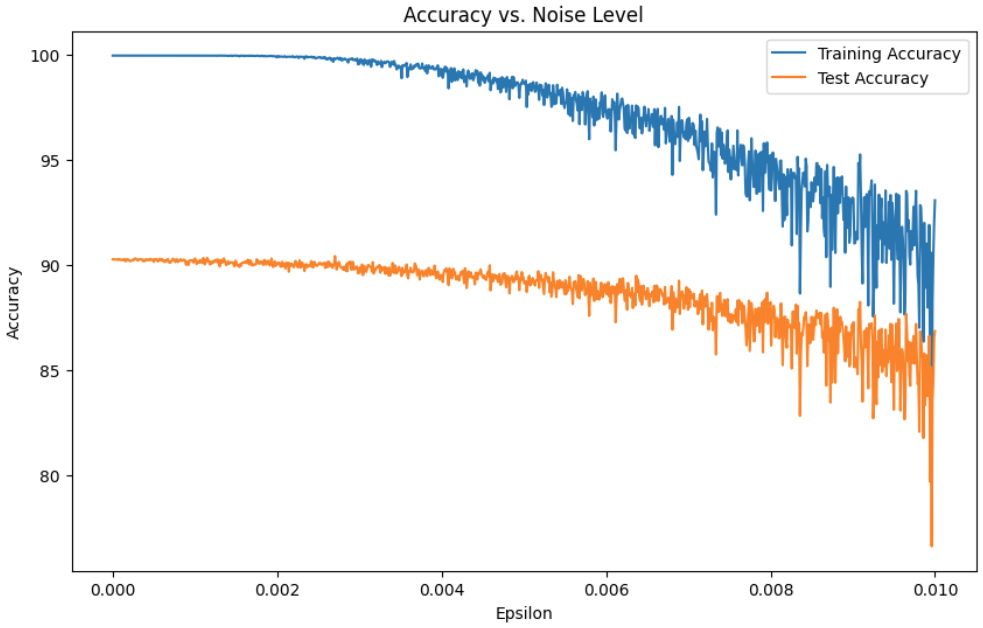}
        \caption{Accuracy of the DNN on the training and test sets, vs. $\epsilon$.}
        \label{fig:acc_vs_noise}
    \end{subfigure}
    \hfill
    \begin{subfigure}{0.41\textwidth}
        \includegraphics[width=1.3\textwidth]{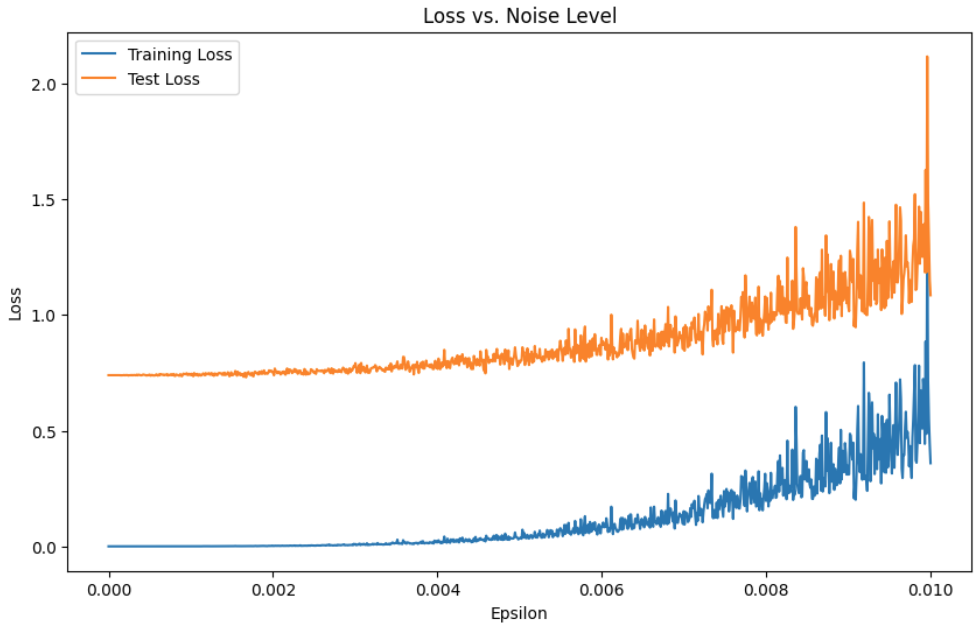}
        \caption{DNN loss (with no regularization) on training and test set vs. $\epsilon$.}
        \label{fig:res101top5}
    \end{subfigure}

    \begin{subfigure}{0.41\textwidth}
        \includegraphics[width=1.3\textwidth]{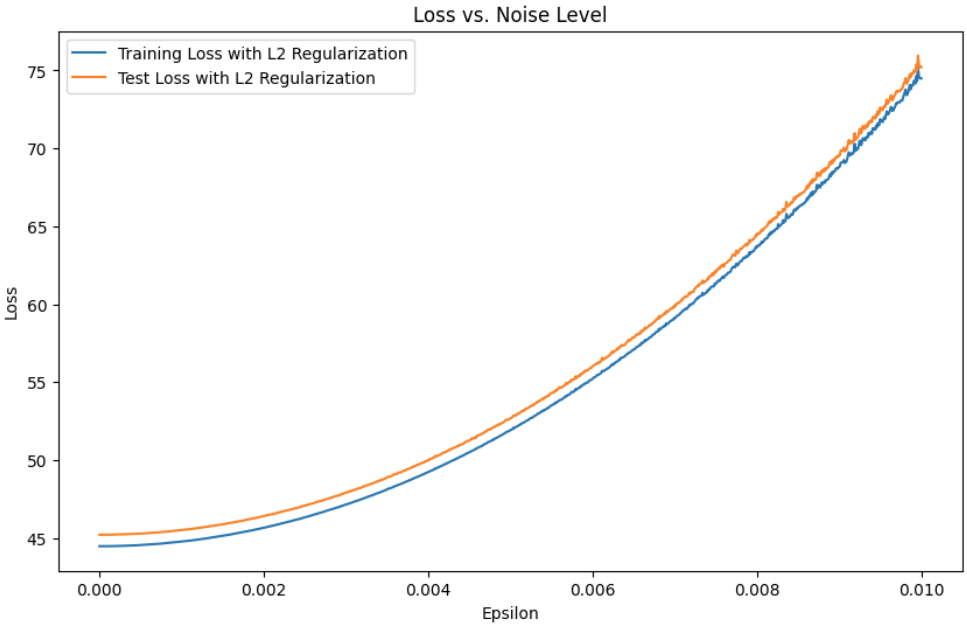}
        \caption{DNN loss with $L_2$ regularization on training and test set vs. $\epsilon$. The hyperparameter $\mu$ in the regularization is $.01$.}
        \label{fig:res101top5}
    \end{subfigure}
    
    \caption{Accuracy and loss as we add noise to the weight layers of the DNN. }
    \label{acc_loss_noise}
\end{figure}

\end{ex}

\begin{algorithm}
\caption{Computing Parameter $\lambda_+$ using MP and Tracy-Widom Distribution}
\label{BEMA_algo}
\begin{algorithmic}[1]
\Procedure{}{}
\State Pick parameters $\alpha \in (0, 1/2), \beta \in (0,1)$.
\For{each $\alpha N \leq k \leq (1-\alpha)N$}
    \State Calculate $q_k$, the $(k/N)$ upper-quantile of the MP distribution with $\sigma^2=1$ and $c = 1$.
    \State $q_k$ satisfies $\int\limits_{0}^{q_k} \frac{1}{2\pi}\frac{\sqrt{(4-\lambda)\lambda}}{\lambda} = k/N$.
\EndFor
\State Evaluate $\hat{\sigma}^2 = \frac{\sum_{\alpha N \leq k \leq (1-\alpha)N}q_k \lambda_k}{\sum_{\alpha N \leq k \leq (1-\alpha)N}q_k^2}$.
\State Derive $t_{1-\beta}$, the $(1-\beta)$ quantile of the Tracy-Widom distribution.
\State Conclude with $\lambda_+ = \hat{\sigma}^2[4+2^{4/3}t_{1-\beta}\cdot N^{-2/3}]$.
\EndProcedure
\end{algorithmic}
\end{algorithm}

 \begin{algorithm}
\caption{Spectral Analysis using BEMA Technique}
\begin{algorithmic}[1]
\Procedure{}{}
\State Take \(X = \frac{1}{N} W^T W\) as input where \(W\) is \(N \times M\).
\State Derive the spectrum of \(X = \{\sigma_1, \dots, \sigma_M\}\).
\State Ascertain the observed cumulative spectral distribution of \(X\), denoted \(F_X\).
\State Apply the BEMA technique with parameters \(\alpha\) and \(\beta\) to discern \(\hat\sigma^2\), the expected coordinate variance of $W$.
\State Determine \(0 \leq i_{\text{min}} < i_{\text{max}} \leq M\) with:
\begin{itemize}
    \item \(i_{\text{min}}\) as the least integer for which \(\frac{i_{\text{min}}}{M} \geq \alpha\).
    \item \(i_{\text{max}}\) as the greatest integer for which \(\frac{i_{\text{max}}}{M} \leq 1 - \alpha\).
\end{itemize}
\State Describe \(F_X'\) as the theoretical CDF for the MP distribution using parameters \(\hat \sigma^2\) and \(\lambda = N/M\).
\State Compute \(s = \max_{i \in [i_{\text{min}}, i_{\text{max}}]} \left |F_X(i) - F_X'(i) \right |\).
\State If \(s > \gamma\), rule out the idea that \(X\) is governed by the stated distribution. Else, if \(s \leq \gamma\), retain this notion.
\EndProcedure
\end{algorithmic}
\end{algorithm}

\begin{algorithm}
\caption{MP-based pruning algorithm}
\label{algo_1}
\begin{algorithmic}[1]
\Require $\ell$, a predetermined number of epochs; $\tau$, a threshold for the MP fit criteria in Subsection \ref{Alignment_Evaluation}; $f(\text{epoch})$, a monotonically decreasing function from $1$ to $0$ (i.e. \eqref{slow_prune}) and for each $1\leq l \leq L$ and weight layer matrix $W_l$ \textbf{state} $split_{l}=false$.
\State Initialize: Train the DNN for $\ell$ epochs. Take $\text{epoch}:=\ell$.

\While{a predefined training condition is met (i.e. $\text{epoch}\leq 100)$}
    \For{each $l$, if $split_{l}=false$ then for weight matrix $W_l$ in the DNN $\phi$}
        \State Perform SVD on $W_l$ to obtain $W_l=U_l\Sigma_lV_l^T$.
        \State Calculate eigenvalues of $ W_l^T W_l$.
        \State Apply BEMA algorithm (see Subsection \ref{finding_lambda_updated}) to find the best fit MP distribution for ESD of $X =  W_l^T W_l$ and corresponding $\lambda_+$.
        \State Check if ESD of $X$ fits the MP distribution using MP fit criteria from Subsection \ref{Conformance_Assessment} and  threshold $\tau$.
        \If{ESD fits the MP distribution}
            \State Eliminate the portion ($1-f(\text{epoch})$) of singular values smaller than $\sqrt{\lambda_+}$ to obtain $\Sigma'$ and form $W'_l=U_l\Sigma'_lV_l^T$.
            \State Use $\Sigma'$ to create $W'_{1,l}=U_l\sqrt{\Sigma'_l}$ and $W'_{2,l}=\sqrt{\Sigma'_l}V_l^T$.
            \If{$W'_{1,l}$ and $W'_{2,l}$ together have fewer parameters than $W'_l$}
            \State Replace $W_l$ in the DNN $\phi$ with  $W'_{1,l}W'_{2,l}$, change $split_l=true$.
            \Else
              \State
                Replace $W_l$ in the DNN $\phi$ with  $W'_l$.
            \EndIf
        \Else
            \State Don't replace $W_l$.
        \EndIf
    \EndFor
    \State Train the DNN for $\ell$ epochs. Take $\text{epoch}:=\text{epoch}+\ell$.
    \For{each $l$, if $split_{l}=true$}
  
    \If{for $W_l:=W'_{1,l}W'_{2,l}$ the ESD of $X_l$ fits the MP distribution with thresholds $\tau$ and $\lambda_+$ \textbf{and} if, we (hypothetically) applied steps 4-12 to $W_l$, the number of parameters in the DNN $\phi$ would decrease}
    \State replace $W'_{1,l}W'_{2,l}$ with $W_l$ and $split_l=false$.
    \Else
    \State Don't change anything.
    \EndIf 
    \EndFor
\EndWhile
\end{algorithmic}
\end{algorithm}

\begin{algorithm}
\caption{Sparsify Singular Vectors }
\label{spar_sing_vectors}
\begin{algorithmic}[1]
\State Let $W_l$ denote the $l$-th layer matrix of a DNN of size $N \times N$ and an initial threshold $\theta$.
\State Perform SVD on $W_l$: $W_l = U \Sigma V^T$.
\State Given threshold $\sqrt{\lambda_+}$.

\For{$i = 1, \ldots, N$}
    \If{$\Sigma_{ii} < \sqrt{\lambda_+}$}
        \State $M_i = 1$
    \Else
        \State $M_i = 0$
    \EndIf
    \State $T_i(\sigma) = \theta \cdot \left(1 - \frac{\sigma}{\sqrt{\lambda_+}}\right)^{30}$
    \If{$M_i = 1$}
        \State $U_{:,i} = \text{apply\_sparsity}(U_{:,i}, T_i(\Sigma_{ii}))$
        \State $V_{i,:} = \text{apply\_sparsity}(V_{i,:}, T_i(\Sigma_{ii}))$
    \EndIf
\EndFor

\State Extra universal sparsification step:
\State Given a threshold, $\theta$.
\For{$i = 1, \ldots, N$}
    \State $U_{:,i} = \text{apply\_sparsity}(U_{:,i}, \frac{\theta}{750})$
    \State $V_{i,:} = \text{apply\_sparsity}(V_{i,:}, \frac{\theta}{750})$
\EndFor

\State Recompose matrix: $W_l' = U \Sigma V^T$.

\end{algorithmic}
\end{algorithm}

\begin{algorithm}
\caption{Element-wise Sparsification}
\label{spars_algo}
\begin{algorithmic}[1]

\State \textbf{Input:} Matrix $W_l$, threshold $\xi$
\State \textbf{Output:} Sparsified matrix $W_l'$

\For{each element $w_{ij}$ in $W_l$}
    \If{$w_{ij} < \xi$}
        \State $w_{ij}' = 0$
    \Else
        \State $w_{ij}' = w_{ij}$
    \EndIf
\EndFor

\State \Return $W_l'$

\end{algorithmic}
\end{algorithm}

\section{A simple illustration of the relationship between RMT-based pruning and regularization}
\label{reg_problem}
The core concept underlying the developed pruning method relies on the observation that Neural Network weight matrices frequently exhibit characteristics akin to those of a low-rank matrix addition to a random matrix component. 
The idea of the pruning method developed relies on the observation that oftentimes, the weight matrices of Neural Networks behave like a very low-rank matrix with an added random matrix. The origin of the randomness is two-fold. Firstly, Neural Networks undergo training through various iterations of stochastic gradient descent. Consequently, the resultant trained network serves as an approximate minimizer of the associated loss function, with its weight distribution potentially reflecting the inherent stochastic nature of the optimization algorithm employed. Secondly, the sample data used to train the network introduces randomness to the training process, particularly when the dataset incorporates random noise. The act of pruning can be conceptualized as a regularization technique operating across the spectrum aimed at promoting the recovery of this low-rank matrix structure.\\
In order to better understand the implications of pruning, we present a simple problem designed to illuminate the effect outcomes of pruning in comparison to alternative forms of regularization.

\subsection{Regularization in a simple regression setting}

Having a low-rank weight matrix implies that the DNN has found a feature map representation of the task that is lower dimensional than the set dimension of the weight matrix. In order to emulate this situation, we introduce a simple regression model where an optimal representation is known. We sample $M$ random function $y^{(k)}\in Span\{\cos(\pi f \cdot),\sin(\pi f \cdot),\text{ for } f\in\{0,0.2,..,2\}\}$, and take random, noisy samples $(X_j,Y_j)\in\mathbb{R}\times\mathbb{R}^M$ s.t. $(Y_j)_k = y^{(k)}(X_j) + 7\epsilon_j$, with $\epsilon_j \sim \mathcal{N}(0, 1)$. Our objective entails the recovery of the $y^{(k)}$ functions from this dataset. An illustrative example of such data along one coordinate is provided in Figure \ref{fig:data_example}. Because of the structure of this problem, an efficient feature map representation of the data is through the Fourier transform, as $y$ belongs to a vector space of low frequency. It has been argued that Neural Networks can be interpreted as a kernel with a learned feature map based on the data. 
For the sake of simplicity in this context, we assume a fixed feature map, the Fourier feature map. Accordingly, the Neural Network materializes as a linear layer on top of the selected feature map, equivalent to the kernel regression defined by the said feature map.

\begin{figure}[h]
    \centering
    \includegraphics[width=0.5\textwidth]{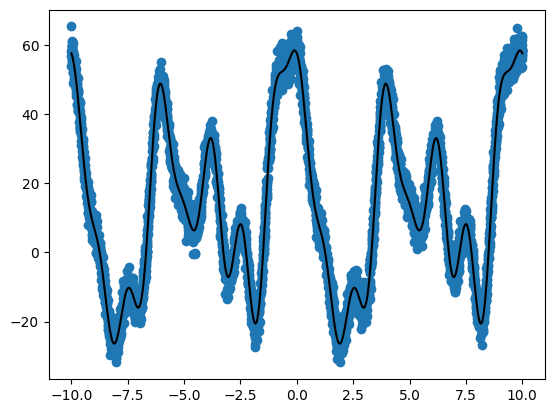}
    \caption{Example of data generated from the problem statement on one axis. The blue line is the function $y_0$, and the blue dots are the data points.}
    \label{fig:data_example}
\end{figure}

In the following paragraphs, we will study the impact of different regularizations, including pruning, on the singular values of the linear layer's weight matrix. Given the significant variations in these singular values on a logarithmic scale, we visualize the cumulative distribution of singular values rather than conventional histograms.

\paragraph{Ideal case} In the absence of noise, the optimal weight matrix has a low rank, as $y$ belongs to a functional space with a finite number of distinct frequencies.
Using kernel regression theory, we find the optimal weight matrix and visualize its singular values in figure \ref{fig:ideal_spectrum}. Most singular values, barring 16, approach zero (numerically below $10^{-15}$). 
This confirms that the optimal weight matrix is inherently low-rank in the absence of noise. Subsequently, our regression task pertains to recovering this low-dimensional structure when confronted with noisy data.

\begin{figure}[h]
    \centering
    \includegraphics[width=0.5\textwidth]{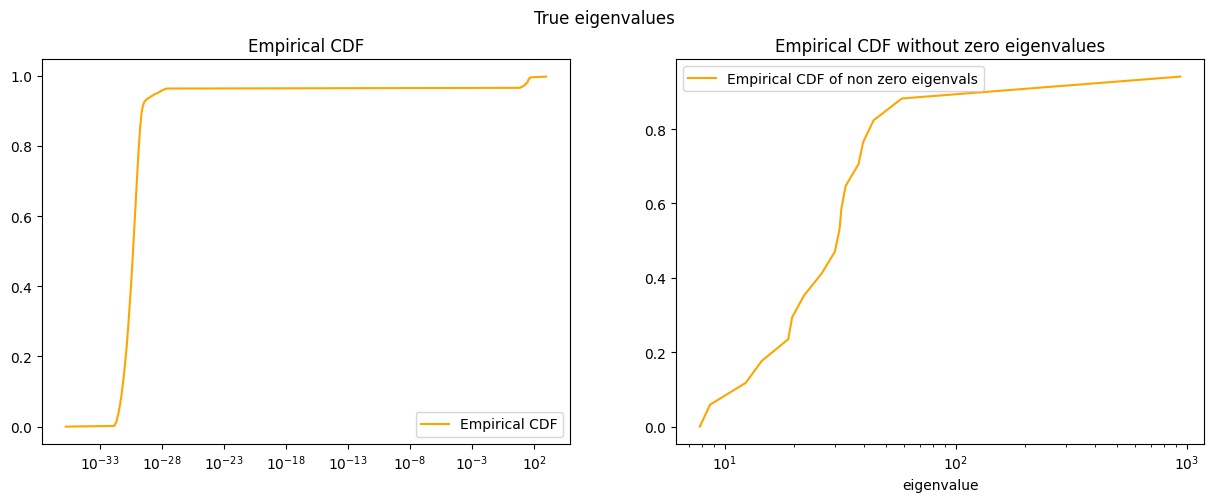}
    \caption{Cumulative distribution of singular values of the optimal weight matrix}
    \label{fig:ideal_spectrum}
\end{figure}

We will now try the pruning regularisation and compare it to two classical regularisations, $L^2$ and $L^1$ regularisation.

\paragraph{$L^2$ regularisation in the Presence of Noise}  Employing Kernel Ridge Regression, we obtain a $L^2$ regularised weight matrix. 
We plot the singular values of this matrix in figure \ref{fig:l2_spectrum}. We can see a large number of minimal singular values.
By running the BEMA algorithm on this matrix, we identify a spectrum resembling the Marchenko-Pastur distribution. 
Moreover, the algorithm finds an upper bound for random singular values consistent with the ideal spectrum we observed above. 
This alignment with the Marchenko-Pastur distribution implies that the predominant source of randomness in our weight matrices is data noise rather than the stochastic nature of the optimization algorithm. Indeed, the weight matrix obtained here is deterministic when conditioned on the data, as the loss admits a minimizer with a closed-form solution. 

\begin{figure}[h]
    \centering
    \includegraphics[width=0.5\textwidth]{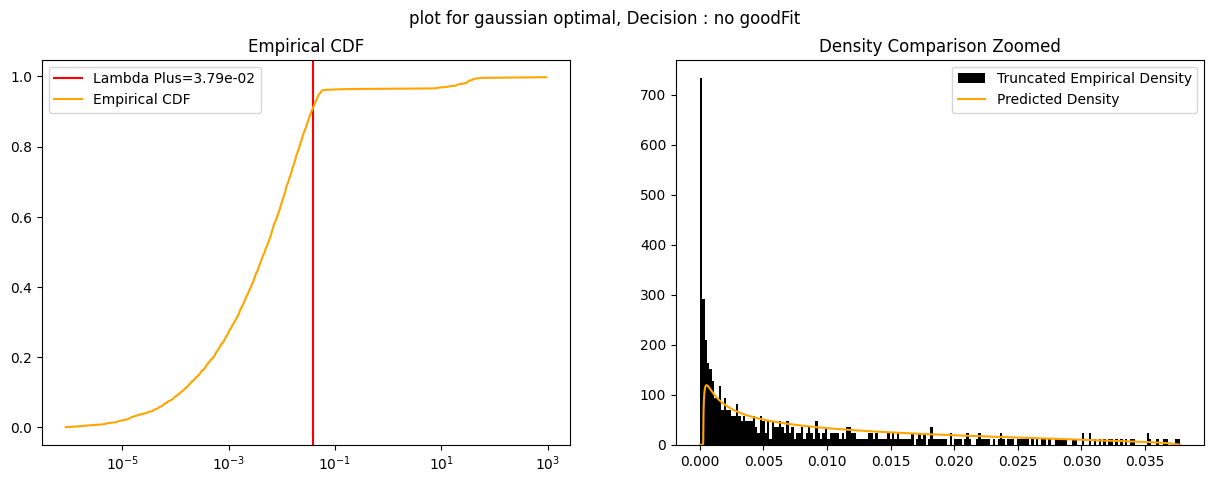}
    \caption{Cumulative distribution of singular values of the $L^2$ regularised weight matrix}
    \label{fig:l2_spectrum}
\end{figure}

\paragraph{$L^1$ regularisation in the Presence of Noise}Turning to $L^1$ regularization, using kernelized Lasso Regression, we plot the singular values of the resulting matrix in Figure \ref{fig:l1_spectrum}. 
Notably, 40\% of the singular values have been zeroed out compared to the $L^2$ case. However, the remaining spectrum approximates an MP distribution, underscoring that the noise-induced randomness persists to some degree. It must be pointed out that $L^1$ regularisation tends to nullify coefficients of the weight matrix instead of singular values. In this scenario, our selected ideal feature map endows the optimal weight matrix with sparsity, thus rendering the recovery of a sparse weight matrix equivalent to the recovery of a low-rank one. 
It warrants emphasis, however, that with an alternate feature map, $L^1$ regularization might prove less effective in revealing the weight matrix's low-rank structure and could even significantly diminish performance if weighed too heavily.

\begin{figure}[h]
    \centering
    \includegraphics[width=0.5\textwidth]{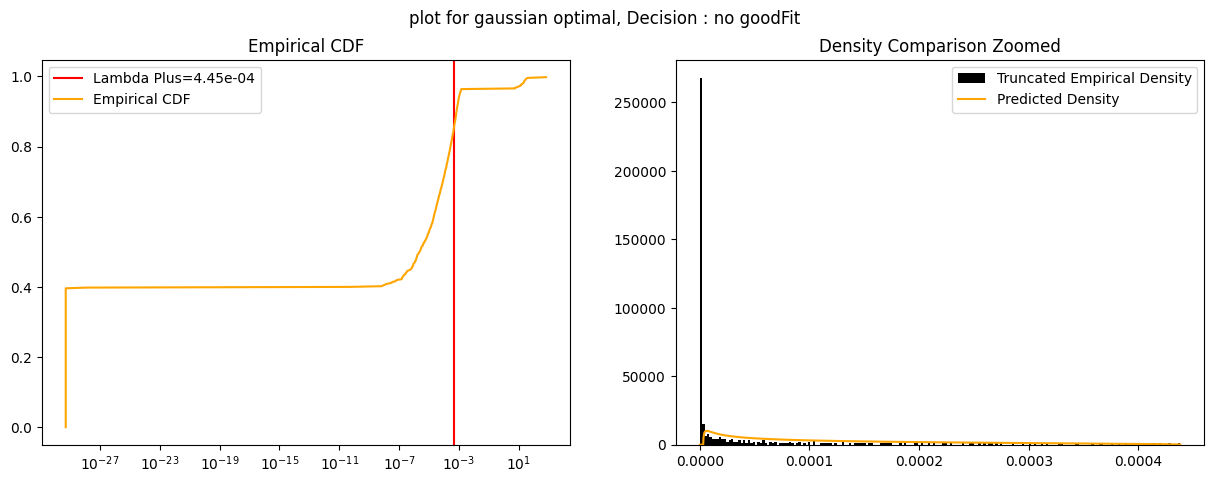}
    \caption{Cumulative distribution of singular values of the $L^1$ regularised weight matrix}
    \label{fig:l1_spectrum}
\end{figure}

\paragraph{Pruning in the Presence of Noise} We visualize the singular values of the unregularised matrix in Figure \ref{fig:pruning_spectrum}.
We see that this spectrum is very similar to the spectrum of the $L^2$ regularised weight matrix. This confirms the intuition that $L^2$ is not an efficient spectrum regularisation. 
Nonetheless, the BEMA algorithm successfully identifies a precise upper bound for the matrix's random singular values. This shows the capacity of pruning regularization to eliminate the noise-driven randomness from the weight matrix.

\begin{figure}[h]
    \centering
    \includegraphics[width=0.5\textwidth]{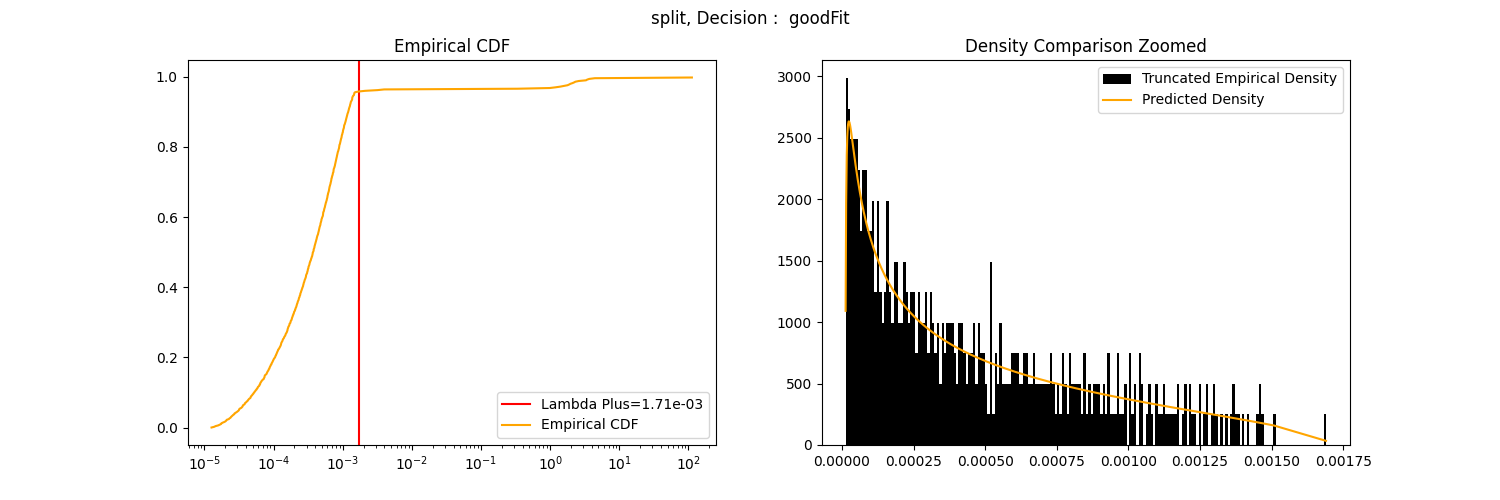}
    \caption{Cumulative distribution of singular values of the weight matrix prior to pruning}
    \label{fig:pruning_spectrum}
\end{figure}

In this simple illustrative instance, pruning regularization emerges as an effective tool for sparsifying the weight matrix's spectrum, thereby facilitating the recuperation of the underlying low-rank structure intrinsic to the problem at hand. Table \ref{tab:mse} reports the effects of the different regularizations on the mean squared error. We see that successfully removing the random part of the weight matrix spectrum allows better accuracy. 

\begin{table}[h]
\centering
\begin{tabular}{|c|c|}
\hline
\textbf{Model} & \textbf{Mean Squared Error} \\ \hline
No Regularization & 2.5735 \\ \hline
L2 Regularization & 2.5627 \\ \hline
L1 Regularization & 0.4157 \\ \hline
Pruning & 0.1491 \\ \hline
\end{tabular}
\caption{Mean Squared Errors of Different Models for the simple regression task}
\label{tab:mse}
\end{table}

\subsection{Proof of Lemma \ref{main_result_remove _R}}
\label{per_cor_proof_component_bound}

We want to bound the probability of a single component of the output of the DNN when a random matrix \( R \) is included. Specifically, we aim to bound the probability that the \( i \)-th component of the following expression exceeds \( t \):

\begin{proof}
\textbf{Step 1: Apply the Triangle Inequality}

By the triangle inequality, we can remove the activation functions (absolute value) and simplify the terms inside the norm:
\begin{align}
Z = \mathbb{P} \left( |W_3 \circ \lambda \circ (R+S) \circ \lambda \circ W_1 s - W_3 \circ \lambda \circ S \circ \lambda \circ W_1 s|_i > t \right)
\end{align}

 \textbf{Step 2: Factor Out \( W_3 \)}

We can now factor out \( W_3 \), using the inequality \( |(Av)_i| \leq \|A\|_1 \|v\|_{\infty} \), to move \( W_3 \) outside of the norm:
\begin{align}
Z = \mathbb{P} \left( \|W_3\|_1 \cdot \|   (R+S) \circ \lambda \circ W_1 s - S \circ \lambda \circ W_1 s)\|_{\infty} > t \right)
\end{align}

 \textbf{Step 3: Focus on \( R \)}

We now separate the terms involving \( R \) and \( S \), leaving us with:
\begin{align}
Z \leq \mathbb{P} \left(  \|W_3\|_1 \cdot \| (R \circ \lambda \circ W_1 s)\|_{\infty} > t \right)
\end{align}

\textbf{Step 4: Focus on the Random Matrix $R$}
Let $Y = R \circ \lambda \circ W_1 s$. Each element of $Y$ is a weighted sum of independent Gaussian random variables:
\[
Y_i = \sum_j R_{ij} (\lambda \circ W_1 s)_j
\]

The variance of each $Y_i$ is:
\[
\text{Var}(Y_i) = \frac{1}{N} \sum_j (\lambda \circ W_1 s)_j^2 = \frac{\| W_1 s\|_2^2}{N}
\]

\textbf{Step 5: Apply Borell-TIS Inequality}
We can now apply the Borell-TIS inequality to $Y$. Note that each $Y_i$ is a Gaussian random variable with mean 0 and variance $\frac{\| W_1 s\|_2^2}{N}$. Let $s_Y^2 = \frac{\| W_1 s\|_2^2}{N}$.

From the Borell-TIS inequality, for any $t > 0$:
\[
\mathbb{P}\left( \|Y\|_\infty > \sqrt{\frac{2 \log N}{N}} \| W_1 s\|_2 + k \right) \leq 2 \exp\left( -\frac{N k^2}{2\| W_1 s\|_2^2} \right)
\]

\textbf{Step 6: Incorporate Deterministic Matrix \( W_3 \)}

Next, incorporate the effect of the deterministic matrix \( W_3 \). Using the properties of matrix norms, we have:
\[
\|W_3 Y\|_i \leq \|W_3\|_1 \|Y\|_\infty
\]

Therefore, the probability bound becomes:
\begin{align}
Z &\leq \mathbb{P}\left( \|W_3\|_1 \|Y\|_\infty > (\sqrt{\frac{2 \log N}{N}} \| W_1 s\|_2 + k)\|W_3\|_1   \right) \\
  &\leq  2 \exp\left( -\frac{N k^2}{2\| W_1 s\|_2^2}  \right),
\end{align}

with $t=(\sqrt{\frac{2 \log N}{N}} \| W_1 s\|_2 + k)\|W_3\|_1$.

\textbf{Step 7: Final Bound}

Substituting \( k = \frac{ \|W_1 s\|_2}{N^{1.5/4}} \) results in:
\[
Z \leq 2 \exp \left( - \frac{N^{1 - 1.5/2}}{2} \right) = 2 \exp \left( - \frac{N^{1/4}}{2} \right)
\]

\textbf{Conclusion:}
\begin{align}
\label{main_equation_5}
Z = \mathbb{P} \Big( \big| \lambda \circ W_3 \circ \lambda \circ (R + S) \circ \lambda \circ W_1 s 
    - \lambda \circ W_3 \circ \lambda \circ S \circ \lambda \circ W_1 s \big|_i > \\
    \sqrt{\frac{2 \log N}{N}} \| W_1 s\|_2 \|W_3\|_1 +  \frac{\|W_3\|_1 \|W_1 s\|_2}{N^{1.5/4}} \Big) 
    \leq 2 \exp \left( - \frac{N^{1/4}}{2} \right)
\end{align}

\end{proof}

For a DNN with more than three-layer matrices, the proof would be the same. For example, in the case of four-layer matrices, we have:

\begin{proof}
\textbf{Step 1: Express the Problem}
We start with the probability we want to bound, now with four weight layers:

\begin{align}
Z = \mathbb{P} \left(|\lambda \circ W_4 \circ \lambda \circ W_3 \circ \lambda \circ (S + R) \circ \lambda \circ W_1 s - \lambda \circ W_4 \circ \lambda \circ W_3 \circ \lambda \circ S \circ \lambda \circ W_1 s|_i > t \right)
\end{align}

\textbf{Step 2: Use Matrix Norm Properties}
We can bound this using the $\ell_\infty$ norm and properties of matrix norms:

\begin{align}
Z \leq \mathbb{P} \left(\|W_4\|_1 \cdot \|W_3\|_1 \cdot \|R \circ \lambda \circ W_1 s\|_\infty > t \right)
\end{align}

The rest of the proof then proceeds as above. 
\end{proof}

For a more general DNN, we first restate the relevant assumptions:

\textbf{Assumption 4:}

Consider a DNN denoted by $\phi$, and assume that it can be written as:

\begin{equation}
    \phi=\rho \circ \psi_2\circ (R+S)\circ \psi_1,
\end{equation}
where $\psi_1$ and $\psi_2$ are arbitrary functions and $\rho$ is softmax. Furthermore, assume that $\exists C_1$ constant such that for any two arbitrary vectors $v$ and $w$ we have that $\|\psi_2v-\psi_2 w\|_{\infty} \leq C_1 \|\psi_2(v-w)\|_{\infty}$, with $\|\cdot\|_{\infty}$ the max norm of the vector. Take $W=R+S$ and assume that the matrices $R$ and $S$ satisfying the following assumptions:

\subsection*{Assumptions on the matrices $R$ and $S$}
\label{assumptions_2}
  We considered a class of admissible matrices $W$, where $W=R+S$ and $W$, $R$ and $S$ satisfy the following three assumptions. The first  assumption is a condition on $R$: 

\textbf{Assumption 5}:
Assume \(R\) is a \(N\times M\) matrix such that for any vector $v$ we have:

\begin{equation} 
\label{gen_ass_for_R}
\mathbb{P}( \|Rv\|_{\infty} > d_1(N)J(v) ) \leq d_2(N),
\end{equation}
with $d_1(N),d_2(N) \to 0$ as $N \to \infty$ and $J(v)$ some function that depends on $v$ alone. Further, as $N \to \infty$, we have that  $\sigma_{\max}(R) \to \sqrt{\lambda_+}$ a.s.    \\

\begin{ex}
    As mentioned, when $R$ is i.i.d Gaussian one can show, using the Borell-TIS inequality (see Theorem \ref{Borell-TIS Inequality}), that:
\[
\mathbb{P}\left( \|Rv\|_\infty > (\sqrt{\frac{2 \log N}{N}}  + \frac{1}{N^{\frac{3}{8}}})\| v\|_2 \right) \leq 2 \exp\left( -\frac{N^{\frac{1}{4}}}{2} \right)
\]

\end{ex}

We now prove Lemma \ref{main_result_remove _R_general}.

\begin{proof}
\textbf{Step 1: Apply Assumption 4}

Take

\begin{equation} 
Z := \mathbb{P} \left( |\psi_2 \circ (R+S) \circ \psi_1 s - \psi_2 \circ (S) \circ \psi_1 s |_i > t \right)
\end{equation}

By Assumption 4, we have:
\begin{align}
Z \leq \mathbb{P} \left( C_1\|\psi_2 \circ ((R+S) \circ \psi_1 s -  (S) \circ \psi_1 s) \|_{\infty} > t \right)
\end{align}

 \textbf{Step 2: Factor Out \( \psi_2 \)}

We can now factor out \( \psi_2 \), using the inequality \( |(\psi_2v)_i| \leq \|\psi_2\|_1 \|v\|_{\infty} \), to move \( \psi_2 \) outside of the norm:
\begin{align}
Z \leq \mathbb{P} \left(C_1 \|\psi_2\|_1 \cdot \|   (R+S) \circ \psi_1 s - S \circ \psi_1 s)\|_{\infty} > t \right)
\end{align}

 \textbf{Step 3: Focus on \( R \)}

We now separate the terms involving \( R \) and \( S \), leaving us with:
\begin{align}
Z \leq \mathbb{P} \left(  C_1 \|\psi_2\|_1 \cdot \| (R \circ \psi_1 s)\|_{\infty} > t \right)
\end{align}

\textbf{Step 4: Focus on the Random Matrix $R$}
From Assumption 5, for any $t > 0$:
\[
\mathbb{P}\left( \| R \circ \psi_1 s\|_\infty > d_1(N)J(\psi_1 s))) \leq d_2(N) \right)
\]

\textbf{Step 6: Incorporate \( \psi_2 \)}

Next, incorporate the effect of  \( \psi_2 \). Using the properties of matrix norms, we have:

\begin{align}
Z \leq \mathbb{P}( C_1\|\psi_2\|_1 \| R \circ \psi_1 s\|_{\infty}   > C_1 \|\psi_2\|_1 d_1(N)J(\psi_1s)   ) \leq d_2(N) ),
\end{align}

with $t=C_1\|\psi_2\|_1d_1(N)J(\psi_1s)$.

\end{proof}

\subsection*{Borell-TIS Inequality for i.i.d. Gaussian Variables}

\begin{thm}[Borell-TIS Inequality]
\label{Borell-TIS Inequality}
Let \( X_1, X_2, \ldots, X_n \) be i.i.d. centered Gaussian random variables with \( X_i \sim N(0, \sigma^2) \). Set \( s_X^2 := \max_{i=1,\ldots,n} \mathbb{E}(X_i^2) = \sigma^2 \). Then for each \( t > 0 \):

\[
P\left( \max_{i=1, \ldots, n} |X_i| - \mathbb{E}\left[ \max_{i=1, \ldots, n} X_i \right] > t \right) \leq \exp\left( -\frac{t^2}{2s_X^2} \right).
\]

For the absolute value bound:

\[
P\left( |\max_{i=1, \ldots, n} |X_i| - \mathbb{E}\left[ \max_{i=1, \ldots, n} X_i \right]| > t \right) \leq 2 \exp\left( -\frac{t^2}{2s_X^2} \right).
\]

In particular, if \( X_i \sim N(0, \frac{1}{N}) \):

\[
s_X^2 = \frac{1}{N},
\]

and for any \( t > 0 \):

\[
P\left( \max_{i=1, \ldots, n} |X_i| > \sqrt{\frac{2 \log n}{N}} + t \right) \leq 2 \exp\left( -\frac{N t^2}{2} \right).
\]
\end{thm}

\subsection*{Generalized Chernoff Bound}

\begin{thm}[Chernoff Bound]
\label{Chernoff Bound}
 For any random variable \( X \) with moment-generating function \( M_X(t) = \mathbb{E}[e^{tX}] \) and for any \( t > 0 \):

\[
P(X \geq a) = P(e^{tX} \geq e^{ta}) \leq M_X(t) e^{-ta}.
\]

For the absolute value bound:

\[
P(|X| \geq a) \leq M_X(t) e^{-ta} + M_{-X}(t) e^{-ta}.
\]

If the distribution is symmetric, such that \( M_X(t) = M_{-X}(t) \):

\[
P(|X| \geq a) \leq 2M_X(t) e^{-ta}.
\]
\end{thm}

\subsection*{Chernoff Bound for Normal Distribution}

\begin{thm}[Chernoff Bound: normal distribution]
\label{Chernoff_Bound_normal_distribution}
For a random variable \( X \) that is normally distributed with \( X \sim N(0, 1/N) \), the moment-generating function is \( M_X(t) = \mathbb{E}[e^{tX}] = e^{\frac{1}{2N} t^2} \). For any \( t > 0 \):

\[
P(X \geq a) \leq e^{\frac{1}{2N} t^2} e^{-ta}.
\]

By choosing \( t = Na \):

\[
P(X \geq a) \leq e^{-\frac{N a^2}{2}}.
\]

For the absolute value bound:

\[
P(|X| \geq a) \leq 2e^{-\frac{N a^2}{2}}.
\]
\end{thm} 

\subsection{Training of a DNN with three weight layers using L1 and L2 regularization with varying layer widths}
\label{details_of_C1}
The goal of this subsection is to check how $a(N)=\sqrt{\frac{2 \log N}{N}} \| W_1 s\|_2 \|W_3\|_1 +  \frac{\|W_3\|_1 \|W_1 s\|_2}{N^{1.5/4}}$ depends on $N$ and the training process. From equation \eqref{main_equation_5}, we know that removing the random matrix $R_2$ from the DNN will change the components of the outputs $X$ by less than $a(N)$. We want to see if this change in the output is small in practice. 
We train a DNN with three fully connected weight layers: $\mathbf{W_1}, \mathbf{W_2}, \mathbf{W_3}$. The middle layer matrix $\mathbf{W_2}$ is of variable width $N$, which we vary across different experiments. We aim to analyze how varying $N$ affects both the network's accuracy and certain key norms of the weight matrices. 

We use both L1 and L2 regularization during training, with the regularization strengths kept small ($.0000001$) so that they minimally affect the training process. Additionally, the learning rate is set to a small value of $0.0001$ to ensure that the weight matrices do not change significantly during the training period.

\subsubsection*{Normalization of the Fashion MNIST Dataset}
The Fashion MNIST dataset consists of grayscale images of size $28 \times 28$. To normalize the data, we ensure that the largest element across the entire dataset has an L2 norm of $0.1$. This normalization is essential for controlling the magnitude of the input vector, which ensures that $a(N)$ stays small. 

The normalization procedure involves the following steps:
\begin{enumerate}
    \item Compute the L2 norm of each input vector (each image flattened to a 1D array).
    \item Find the maximum norm among all vectors in the dataset.
    \item Scale all input vectors so that the largest vector has norm $0.1$ .
\end{enumerate}

\subsubsection*{Training procedure}
 The network is trained for 10 epochs, and we keep the learning rate at $0.0001$. The weights are initialized from $N(0,1/N)$, with $N$ the number of column vectors. 

\subsubsection*{Analysis and metrics}

For each value of $N$ (the layer weight matrix $W_2$  has size $N \times N$), we compute two key quantities:
\begin{itemize}
    \item $a(N)$: This is a metric based on the norms of the weight matrices, calculated as:
    \[
    a(N) = \sqrt{\frac{2 \log N}{N}} \| W_1 s \|_2 \| W_3 \|_1 + \frac{\| W_3 \|_1 \| W_1 s \|_2}{N^{1.5 / 4}}
    \]
    where $\|W_3\|_1$ is the L1 norm of the final weight layer and $\|W_1 s\|_2$ is the largest L2 norm of the transformed input data.
    
    \item Test Accuracy: The percentage of correctly classified test samples, which allows us to track how network performance evolves as the width $N$ increases.
\end{itemize}

We analyze the behavior of both $a(N)$ and the test accuracy as functions of the layer weight matrix $W_2$  of size $N \times N$. The values of $N$ that we test for are $[500, 1000, 1500, 2000, 3000, 4000, 5000, 7000, 10000,20000]$. Fig. \ref{a(N)} shows that as we increase the size $W_2$, we have that $a(N)$ goes to zero quickly.

\begin{figure}[h!]
    \centering
    \includegraphics[width=0.8\textwidth]{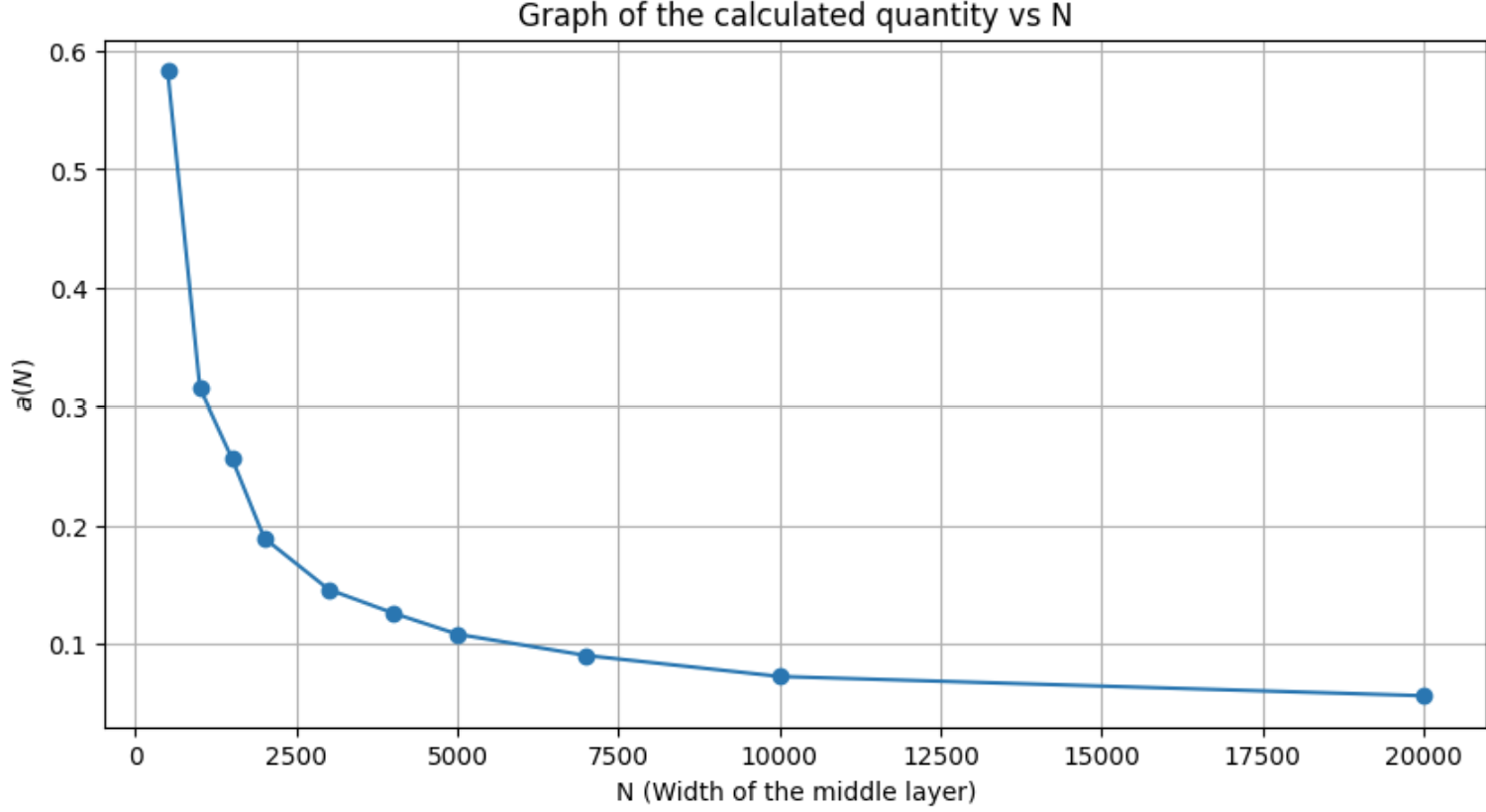}
    \caption{Plot of $a(N)$ vs. the width $N$ of the middle layer $\mathbf{W_2}$. This plot shows how the complexity-related metric $a(N)$ behaves as the width $N$ is increased, reflecting changes in the norms of the weight matrices and their dependence on $N$.}
    \label{a(N)}
\end{figure}

\begin{figure}[h!]
    \centering
    \includegraphics[width=0.8\textwidth]{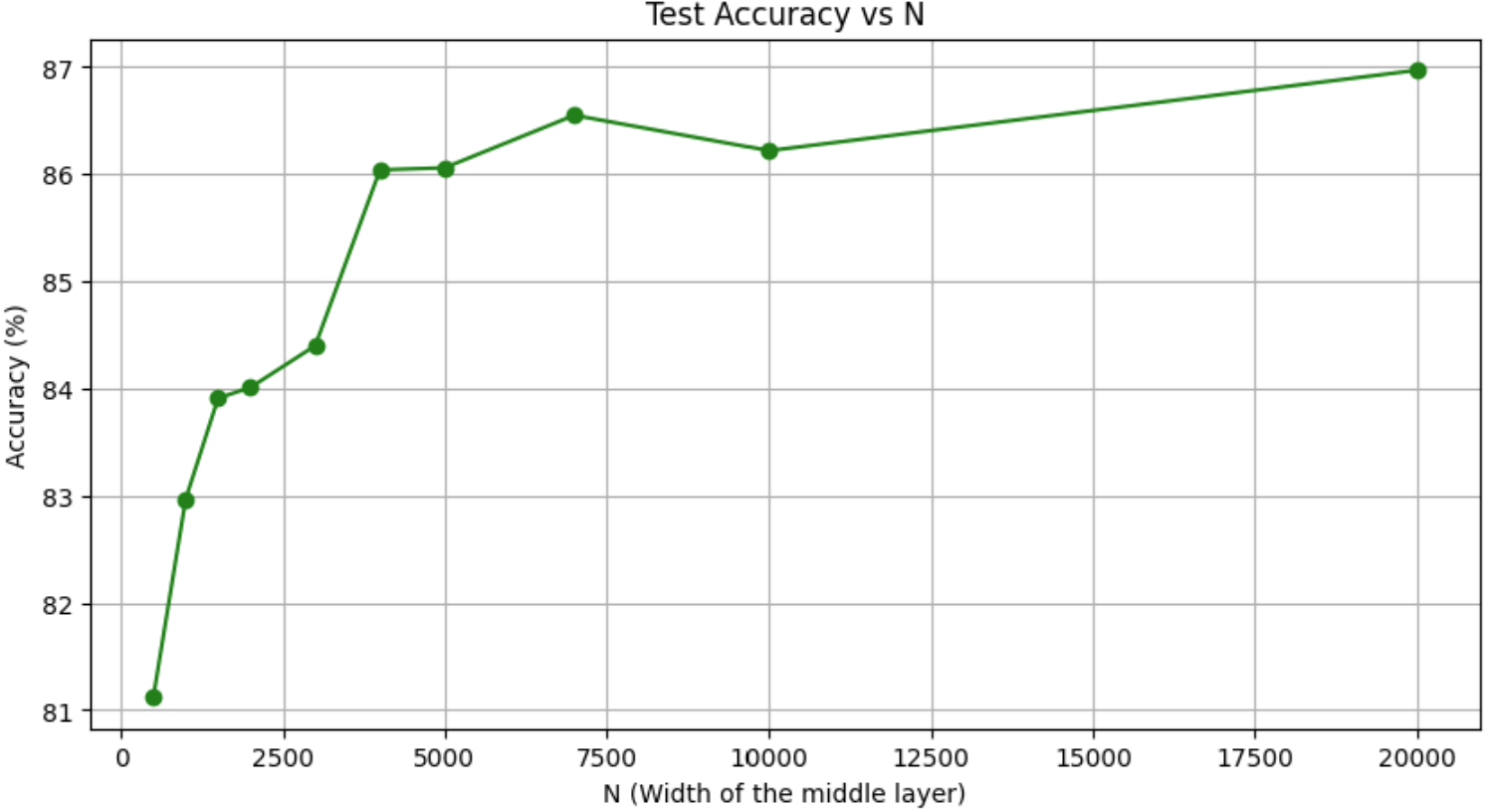}
    \caption{Plot of test accuracy vs. the width $N$ of the middle layer $\mathbf{W_2}$. This plot demonstrates how the performance of the DNN changes with increasing $N$, showing the relationship between network width and generalization ability on the Fashion MNIST dataset.}
\end{figure}

\subsubsection*{Normalization and Initialization of Weight Matrices}
In this experiment, we normalize the input vectors such that the largest input vector has a norm of $10$. The normalization process is identical to that described previously, but instead of scaling the largest vector to $0.1$, we scale it to $1$.

Additionally, the weight matrices for each layer are initialized such that each entry is sampled from a normal distribution with zero mean and variance $1/N^2$, where $N$ is the number of columns in the weight matrix. 

\subsubsection{Analysis of $b(N)$}
Similar to the previous metric $a(N)$, we introduce a new metric $b(N)$, which is computed exactly like $a(N)$, but each occurrence of $N$ is squared in the formula. Specifically, $b(N)$ is defined as:
\[
b(N) = \sqrt{\frac{2 \log N^2}{N^2}} \| W_1 s \|_2 \| W_3 \|_1 + \frac{\| W_3 \|_1 \| W_1 s \|_2}{(N^2)^{1.5 / 4}}
\]
where $\|W_3\|_1$ is the L1 norm of the final weight layer and $\|W_1 s\|_2$ is the largest L2 norm of the transformed input data.

Note, given the above initialization, 
equation \ref{main_equation_5} has $b(N)$ on the LHS rather then $a(N)$.

\subsubsection*{Results and Observations}
We analyze the behavior of both $b(N)$ and the test accuracy as functions of the width $N$ of the middle weight matrix $\mathbf{W_2}$. The values of $N$ tested are $[500, 1000, 1500, 2000, 3000, 4000, 5000, 7000, 10000, 20000]$. Here, the DNN was trained for $20$ epochs with a learning rate of $.001$.

\begin{figure}[h!]
    \centering
    \includegraphics[width=0.8\textwidth]{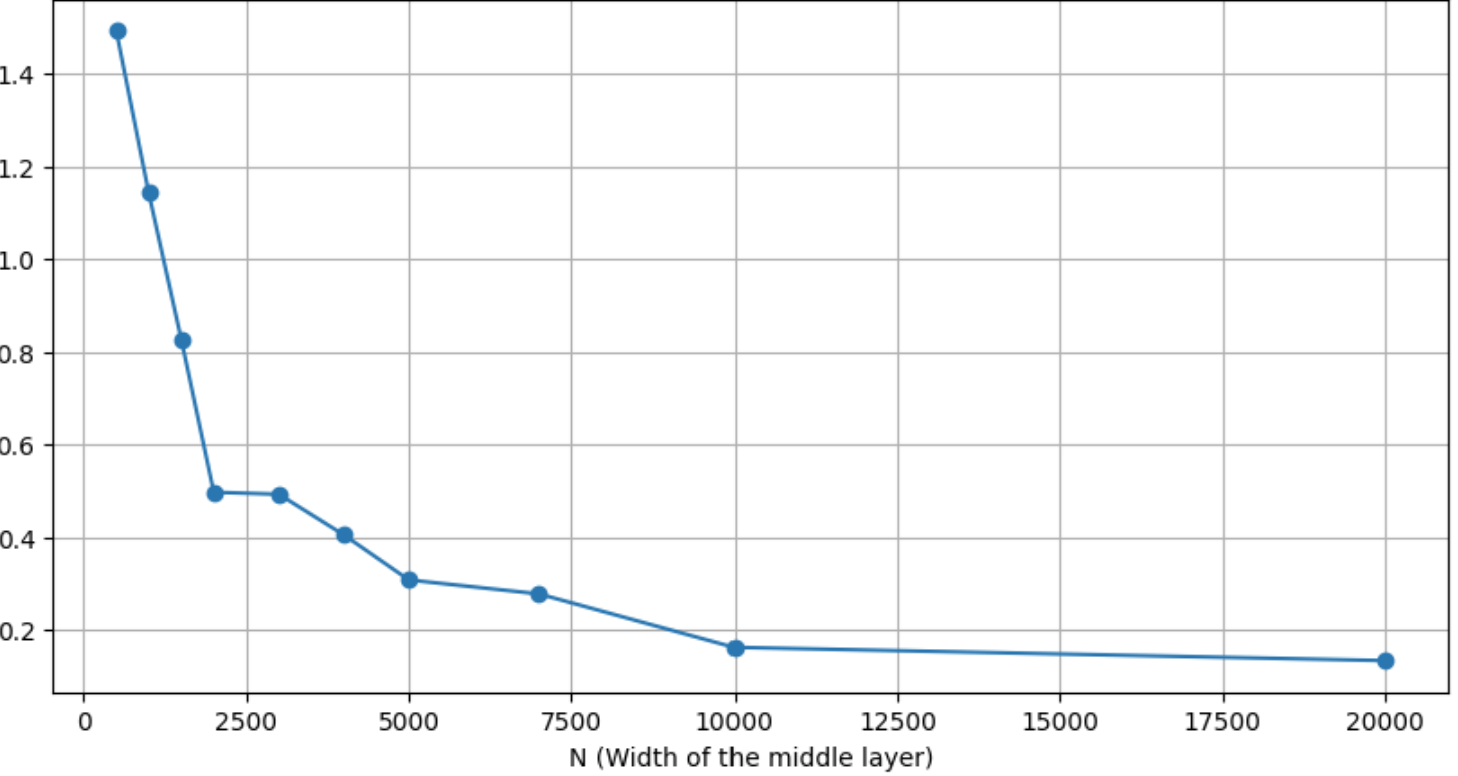}
    \caption{Plot of $b(N)$ vs. the width $N$ of the middle layer $\mathbf{W_2}$. This plot shows how the metric $b(N)$ behaves as the size of the middle layer increases, incorporating a stronger dependence on $N^2$.}
    \label{b(N)}
\end{figure}

\begin{figure}[h!]
    \centering
    \includegraphics[width=0.8\textwidth]{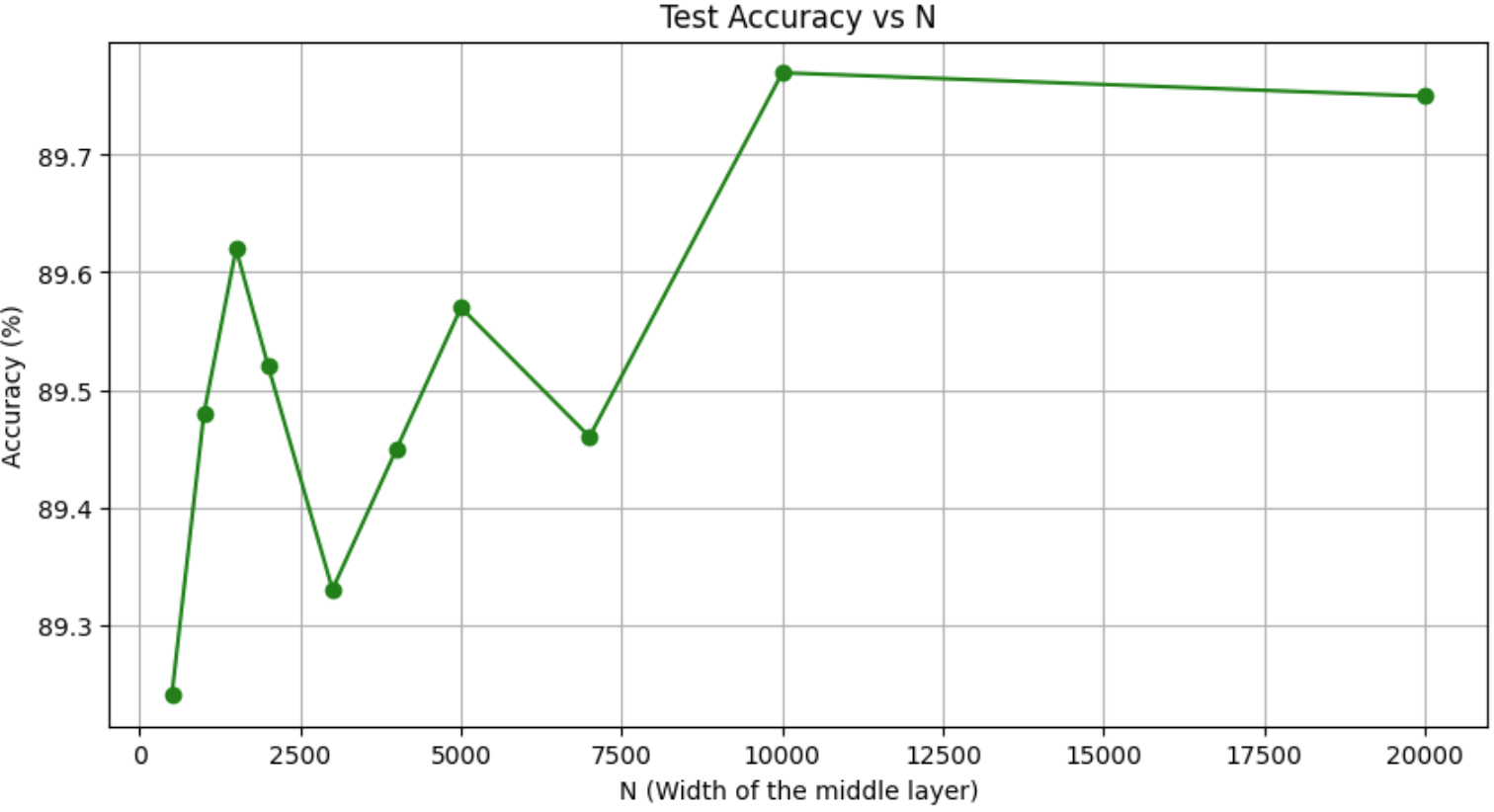}
    \caption{Plot of test accuracy vs. the width $N$ of the middle layer $\mathbf{W_2}$. This plot demonstrates how the network's accuracy changes as $N$ increases, with weight matrices initialized using $1/N^2$.}
    \label{accuracy_vs_N_bN}
\end{figure}

From Fig. \ref{b(N)}, we observe that as we increase the size of $W_2$, the metric $b(N)$ converges toward zero at a faster rate compared to $a(N)$. This suggests that the dependence on \( N^2 \) significantly impacts the network's behavior.

Additionally, in Fig. \ref{accuracy_vs_N_bN}, we see the test accuracy improving as the width $N$ increases, consistent with the expectations from larger model capacity, though the gains diminish at higher values of $N$. The training period for each network was kept constant at 10 epochs with a learning rate of $0.0001$, ensuring consistent comparisons between the different experiments.

\subsection{Some key lemmas: how accuracy and loss are affected when removing the random matrix $R_2$}

Next, we introduce a lemma that describes how the loss, without regularization, and the accuracy of the DNN are affected when we remove the random matrix $R_2$ from the weight layer matrix $W_2$. We define:  

\begin{equation}
\label{loss_noise_det}
L_{\cancel{\text{reg}}}(\alpha(t))=-\frac{1}{|T|}\sum_{s\in T}\log\left(\phi_{C(s)}(s,\alpha(t))\right). 
\end{equation}

thus

\begin{equation}
    L_{\cancel{\text{reg}}} = -\frac{1}{|T|}\sum_{s\in T}\log \left( \frac{e^{X_{C(s)}(\alpha(t), s)}}{\sum_{i=1}^{K} e^{(X_i(\alpha(t), s))}} \right),
    \end{equation}

with $K$ the number of classes in our classification problem.

     We now introduce the \textit{classification confidence}, see \cite{berlyand2021stability}. Take $X(s,\alpha)$ to be the output of the final layer in our DNN before softmax. The classification confidence is defined as follows:
  	
\begin{equation}\label{eq:defdeltaX}
\delta X(s,\alpha):=X_{C(s)}(s,\alpha)-\max_{j\neq C(s)} X_j(s,\alpha).
\end{equation}

  	In other words, 
  	\begin{itemize}
		%\item $\delta X(s,\alpha(t))$ is independent of $\overline{R}$%$\hat\alpha=(\hat\alpha^1,\cdots,\hat\alpha^{M-1})$ in RHS of \eqref{eq:deltaX} $\Rightarrow$ magnitude of $\alpha$, $\overline{R}$, does not affect .
		\item $\delta X(s,\alpha(t))>0\Rightarrow s$ is well-classified by $\phi$.
		\item $\delta X(s,\alpha(t))< 0\Rightarrow s$ is misclassified by $\phi$
			\end{itemize}
			
			For $T'$ the test set we can now define the accuracy of the DNN on $T'$ using the classification confidence,
		\begin{equation}\label{2_class_prob}
		 \text{acc}_{\alpha}(t)=\frac{\#\left(\{s\in T':\delta X(s,\alpha(t))>0\}\right)}{\#T'}.
		\end{equation}

\begin{lem}\label{main_result_remove _R_loosandacc}
Let $\phi$ be a DNN with weight layer matrices $W_1,W_2,W_3$ and $\lambda$ the absolute value activation function: \vspace{-.3cm}
		
		\begin{equation}
			\vspace{-.2cm}
			\phi(\alpha,s)= \rho \circ \lambda \circ W_3  \circ \lambda \circ W_2 \circ \lambda \circ W_1 s, \hspace{.4cm} s\in \R^n. 
		\end{equation}

   Given \eqref{loss_noise_det_1}, let $W_2$ be a $N \times N$ matrix such that   $W_2(t)=R_2(t)+S_2(t)$, with $R_2$ and $S_2$ satisfying Assumptions 2-3.
  
    Suppose we replace the weight layer matrix $W_2$ with the deterministic matrix $S_2$. Then $\exists G_1(N)$ s.t.
    
    \begin{equation}    \label{pruning_equation_1}
\mathbb{P}\bigg(|L_{\cancel{\text{reg}}}(\alpha_{S_2}(t))-L_{\cancel{\text{reg}}}(\alpha_{W_2}(t))| \leq G(N)H(\phi) \bigg)\geq 1-G(N),
   \end{equation}
  with $G_1(N) \to 0$ as $N \to \infty$ and $H(\phi)$ some function depending on $\phi$ with $H(\phi) \to 0$ as $N \to \infty$ if $a(N,s):=  \frac{\| W_3 \|_1 \| W_1 s \|_2}{N^{1.5 / 4}}\to 0$ as $N \to \infty$.

\end{lem}

See Subsection \ref{chnaging_loss_and_acc} for a proof of this lemma. See also Example \ref{adding_noise} for a numerical example of this lemma.

\subsection{Proof of Lemma \ref{main_result_remove _R_loosandacc}}
\label{chnaging_loss_and_acc}

Recall, we define:  

\begin{equation}
\label{loss_noise_det}
L_{\cancel{\text{reg}}}(\alpha(t))=-\frac{1}{|T|}\sum_{s\in T}\log\left(\phi_{C(s)}(s,\alpha(t))\right). 
\end{equation}

thus

\begin{equation}
    L_{\cancel{\text{reg}}} = -\frac{1}{|T|}\sum_{s\in T}\log \left( \frac{e^{X_{C(s)}(\alpha(t), s)}}{\sum_{i=1}^{K} e^{(X_i(\alpha(t), s))}} \right),
    \end{equation}

with $K$ the number of classes in our classification problem.

From Lemma \ref{main_result_remove _R}, we have that:

\begin{equation}
\mathbb{P}\bigg((|
    X_i(\alpha_W,s) - 
    X_i(\alpha_S,s)|\leq  a(N,s) \bigg)\geq 1-F(N),    \end{equation}
    with $F(N) \to 0$ as $N \to \infty$.
    
    Thus, for any fixed $s \in T'$, we have that, for big enough $N$, $\exists G(N) \to 0$ as $N\to \infty$ such that: 

    \begin{equation}\label{loss_bound_one_object}
       \mathbb{P}\bigg(| \log \left( \frac{e^{X_{C(s)}(\alpha_S(t), s)}}{\sum_{i=1}^{K} e^{(X_i(\alpha_S(t), s))}} \right)- \log \left( \frac{e^{X_{C(s)}(\alpha_W(t), s)}}{\sum_{i=1}^{K} e^{(X_i(\alpha_W(t), s))}} \right)|\leq G(N) C(s,\phi) \bigg)\geq 1-G(N).
    \end{equation}

    Furthermore, because the exponential and log functions are continuous, we have that if $a(N,s) \to 0$, then $G(N)C(s,\phi) \to 0$. Because that $T'$ is a finite set, and because \eqref{loss_bound_one_object} holds for all $s \in T'$, we have that $\exists H(\phi)$ such that:

\begin{equation}\label{loss_bound_one_object_2}
       \mathbb{P}\bigg(|-\frac{1}{|T|}\sum_{s\in T} \log \left( \frac{e^{X_{C(s)}(\alpha_S(t), s)}}{\sum_{i=1}^{K} e^{(X_i(\alpha_S(t), s))}} \right)+\frac{1}{|T|}\sum_{s\in T} \log \left( \frac{e^{X_{C(s)}(\alpha_S(t), s)}}{\sum_{i=1}^{K} e^{(X_i(\alpha_S(t), s))}} \right)|\leq G(N)H(\phi)\bigg)\geq 1-G(N).
    \end{equation}

    This proves the first part of Lemma \ref{main_result_remove _R_loosandacc} and gives us equation \eqref{pruning_equation_1}.

  \section{Proof of Theorems \ref{main_result_reducing_noise} and \ref{main_result_reducing_noise_general}}.
  \label{proof_main_theorem}

  In this section, we prove Theorem \ref{main_result_reducing_noise}. Using the more general Assumptions 4-7 for Theorem \ref{main_result_reducing_noise_general}, we can use the proof for these results as well. Before proving this theorem, we state some useful RMT results for the type of model given in assumptions 1-3 in Subsection \ref{assumptions}.

  \section{ Some known results on perturbation of matrices}

       Matrix perturbation theory is concerned with understanding how small changes in a matrix can affect its properties, such as eigenvalues, eigenvectors, and singular values. In this section, we state a couple of known results from matrix perturbation theory.
       \subsection{Asymptotics of singular values and singular vectors of deformation matrix}
       \label{bi_unitary_case}
The results in this subsection are taken from \cite{benaych2011eigenvalues}. Given the assumptions 1'-2' on $R$ and $S$ described in Section \ref{assumptions_2}, the authors were able to show that the largest eigenvalues and corresponding eigenvectors of $W=S+R$  are well approximated by the largest eigenvalues and eigenvectors of $S$.

We start by defining the following function:  \begin{equation}
D_{\mu_R}(z) = \left[ 
\int \frac{z}{z^2 - t^2} d\mu_R(t)
\right] \times \left[ c \int \frac{z}{z^2-t^2} d\mu_R(t)  + \frac{1-c}{z}\right]
\end{equation}

for $z>\sqrt{\lambda_+}$, with $\lambda_+$ given by the MP distribution of $R^TR$. Take $D^{-1}_{\mu_R}(\cdot)$ to be its functional inverse.

Set \begin{equation} 
\label{theta_bar}
\bar{\theta}= D_{\mu_R}(\sqrt{\lambda_+})^{-\frac{1}{2}}
\end{equation}

\begin{thm}\label{phase_transition_2}{Theorem for large singular values [Benaych-Georges and Nadakuditi (2012)]}
Take $W=R+S$, with $W, R$ and $S$ all  $N \times M$ matrices satisfying assumptions $1'-2'$. The $r$ largest singular values of $W$, denoted as $\sigma'_i(W)$ for $1 \leq i \leq r$,  exhibit the following behaviour as $N \to \infty$:

\begin{equation}\label{singualr_value_cases}
\sigma'_i(W) \xrightarrow[\text{}]{a.s.} 
\begin{cases}
  D^{-1}_{\mu_R}(\frac{1}{(\sigma_i)^2}) & \sigma_i>\bar{\theta}\\
  \sqrt{\lambda_+} & \sigma_i<\bar{\theta}
\end{cases}
\end{equation}
 
\end{thm}

\begin{thm}{Norm of projection of largest singular vectors [Benaych-Georges and Nadakuditi (2012)]}\label{singualr_vector_cases}

Take indices $i_0 \in \{1, . ..,r\}$  such that $\sigma_{i_0} > \Bar{\theta}$. Take $\sigma'_{i_0} =\sigma'_{i_0}(W)$ and let $u',v'$ be left and right unit singular vectors of $W$ associated with the singular value $\sigma'_{i_0}$ and $u,v$ be the corresponding singular vectors of $S$.  Then we have, as $N \to \infty$:

\begin{equation}\label{left_singualr_vectors_cases_2}
|<u',\mathrm{Span}\{u_i \ s.t. \ \sigma_i=\sigma_{i_0}\}>|^2 \xrightarrow[\text{}]{a.s.} \frac{-2\phi_{\mu_R}(\rho)}{\sigma^2_{i_0}D'_{\mu_R}(\rho)}
\end{equation}

and 

\begin{equation}\label{right_singualr_vectors_cases_2}
|<v',\mathrm{Span}\{v_i \ s.t. \ \sigma_i=\sigma_{i_0}\}>|^2 \xrightarrow[\text{}]{a.s.} \frac{-2\phi_{\mu_R}(\rho)}{\sigma^2_{i_0}D'_{\Tilde{\mu_R}}(\rho)}
\end{equation}

Here $\rho=D^{-1}_{\mu_R}(\frac{1}{(\sigma_{i_0})^2})$ and $\Tilde{\mu_R}=c\mu_R+(1+c)\delta_0$.

Further, 

\begin{equation}\label{left_singualr_vectors_cases_3}
|<u',\mathrm{Span}\{u_i \ s.t. \ \sigma_i\neq \sigma_{i_0}\}>|^2 \xrightarrow[\text{}]{a.s.} 0
\end{equation}

\begin{equation}\label{right_singualr_vectors_cases_3}
|<v',\mathrm{Span}\{v_i \ s.t. \ \sigma_i\neq \sigma_{i_0}\}>|^2 \xrightarrow[\text{}]{a.s.} 0
\end{equation}

\end{thm}

\begin{ex}\label{simple_case}
Take $S=\sum_{i=1}^r\sigma_i u_i v_i^T$ to be a $N \times N$ deterministic matrix, with $\sigma_i$ the singular values and $v_i$ and $u_i$ the singular vectors of $S$. Take $R$ to be a $N \times N$ random matrix with real i.i.d components taken from the normal distribution $N(0,\frac{1}{N})$. For $W=R+S$ we have:

\begin{thm}(Theorem for large singular values for Example \ref{simple_case})
\label{singualr_values_simple_case}
The $r$ largest singular values of $W$, denoted  $\sigma'_i(W)$ for $1 \leq i \leq r$,  exhibit the following behaviour as $N \to \infty$:
\[
      \sigma'_i(W) \xrightarrow[\text{}]{a.s.} 
\begin{cases}
  \frac{1+\sigma_i^2}{\sigma_i} & \sigma_i>1\\
  2 & \sigma_i<1
\end{cases}
\]
\end{thm}

\begin{thm}(Theorem for large singular vectors  for Example \ref{simple_case})
\label{singualr_vectors_simple_case}
Assuming that the $r$ largest singular values of $W$ have multiplicity $1$, then the right and left singular vectors $u'_i, v'_i$ of $W$ corresponding with the $r$ largest singular values $\sigma'_i(W)$  exhibits the following behaviour as $N \to \infty$:
 \[
      |<v_i,v'_i>|^2, |<u_i,u'_i>|^2  \xrightarrow[\text{}]{a.s.}
\begin{cases}
  (1-\frac{1}{\sigma^2_i}) & \sigma_i>1\\
  0 & \sigma_i<1
\end{cases}
\]  
\end{thm}

\end{ex}

  \subsection{Asymptotic behavior of singular values and vectors of the deformed matrix model given in Subsection \ref{assumptions}}

The following results can be found in \cite{baik2005phase, benaych2011eigenvalues}. See \cite{couillet2022random} for other similar results. 
   
\begin{ex}\label{simple_case}
Take $S=\sum_{i=1}^r\sigma_i u_i v_i^T$ to be a $N \times N$ deterministic matrix, with $\sigma_i$ the singular values and $v_i$ and $u_i$ the singular vectors of $S$. Take $R$ to be a $N \times N$ random matrix with real i.i.d. components taken from normal distribution $N(0,\frac{1}{N})$. For $W=R+S$ we have:

\begin{thm}(Theorem for large singular values for Example \ref{simple_case})
\label{singualr_values_simple_case}
The $r$ largest singular values of $W$, denoted  $\sigma'_i(W)$ for $1 \leq i \leq r$,  exhibit the following behaviour as $N \to \infty$:
\[
      \sigma'_i(W) \xrightarrow[\text{}]{a.s.} 
\begin{cases}
  \frac{1+\sigma_i^2}{\sigma_i} & \sigma_i>1\\
  2 & \sigma_i<1
\end{cases}
\]
\end{thm}

\end{ex}

\begin{remark}
\label{a.s_CIL}
   We say that $X_n \to X$ in probability if $\forall \epsilon$ $\mathbb{P}(|X_n-X|>\epsilon) \to 0$ as $n \to \infty$. One can show that $X_n \to X$ a.s. implies that $X_n \to X$ in law. Thus for Theorem \ref{singualr_values_simple_case}  $\forall \epsilon$ we have:

   \begin{equation}
   \label{singualr_value_approx}
   \mathbb{P}(|\sigma'_i(W) -\frac{1+\sigma_i^2}{\sigma_i}|>\epsilon)<B_N(\epsilon).
   \end{equation}

\end{remark}

In our case, we want the variance of $R(t)$ to depend on $t$. That is, The matrix $W_b(t)$ is expressed as: 
\begin{equation}
    W_b(t) = R_b(t) + S_b(t), \quad (\textit{deformed matrix})
\end{equation}
where $R_b(t)$ is an $N \times M$ matrix with i.i.d. entries drawn from $N(0,\frac{C(t)}{N})$, with $C(t) \leq 1$ representing the training time-dependent variance, and $S_b(t)$ is a deterministic matrix.

\begin{lem}[Lemma for large singular values with training time-dependent variance]
\label{theorem_time_dependent_variance}
Consider the matrix $W_b(t) = R_b(t) + S_b(t)$, where $R_b(t)$ is an $N \times M$ random matrix with i.i.d. entries drawn from $N(0,\frac{g(t)}{N})$ and $S_b(t)$ is a deterministic matrix. The $r$ largest singular values of $W_b(t)$, denoted $\sigma'_i(W_b(t))$ for $1 \leq i \leq r$, exhibit the following behaviour as $N \to \infty$:
\[
\sigma'_i(W_b(t)) \xrightarrow[\text{}]{a.s.} \sqrt{g(t)} \times
\begin{cases}
\frac{1+(\frac{\sigma_i}{\sqrt{g(t)}})^2}{\frac{\sigma_i}{\sqrt{g(t)}}} & \frac{\sigma_i}{\sqrt{g(t)}} > 1\\
2 & \frac{\sigma_i}{\sqrt{g(t)}} < 1
\end{cases}
\]
\end{lem}

\begin{proof}
Define $W' = R' + S'$ where $R' = \frac{1}{\sqrt{g(t)}}R_b(t)$ and $S' = \frac{1}{g(t)}S_b(t)$. The entries of $R'$ are now drawn from $N(0,\frac{1}{N})$, matching the classical case.

For $W'$, the classical theorem gives the behavior of its largest singular values. Specifically, if $\sigma'_i$ are the singular values of $S'$, then the singular values of $W'$, denoted as $\sigma'_i(W')$, behave as:
\[
\sigma'_i(W') \xrightarrow[\text{}]{a.s.}
\begin{cases}
\frac{1+\sigma'^2_i}{\sigma'_i} & \sigma'_i > 1\\
2 & \sigma'_i < 1
\end{cases}
\]

Since $\sigma'_i = \frac{\sigma_i}{\sqrt{g(t)}}$, the theorem for $W'$ becomes:
\[
\sigma'_i(W') \xrightarrow[\text{}]{a.s.}
\sqrt{g(t)} \times\begin{cases}
\frac{1+(\frac{\sigma_i}{\sqrt{g(t)}})^2}{\frac{\sigma_i}{\sqrt{g(t)}}} & \frac{\sigma_i}{\sqrt{g(t)}} > 1\\
2 & \frac{\sigma_i}{\sqrt{g(t)}} < 1
\end{cases}
\]

To revert to the original matrix $W_b(t)$, we multiply the singular values of $W'$ by $\sqrt{g(t)}$. This gives the asymptotic behavior of the singular values of $W_b(t)$ as stated in the lemma.
\end{proof}

\subsection{Proof of Theorems \ref{main_result_reducing_noise}, \ref{theorem_reduction_of_loss}, \ref{main_result_reducing_noise_general} and \ref{theorem_reduction_of_loss_general}.}
\label{main_theorem_proof}
\begin{proof}

We will prove the main result by showing that if the random component $R_2(t)$ of the weight matrix $W_2(t)$ is non-zero at a local minimum, it contradicts the assumption of a local minimum.

\subsection*{1. Notation and Assumptions:}
\begin{itemize}
    \item We assume that $\alpha^*$ is a local minimum of the loss function $L(\alpha(t))$ defined in Eq. (\ref{loss_noise_det}).
    \item We assume that $\|R_2(\alpha^*)\|_2 \neq 0$, which implies the random component $R_2(\alpha^*)$ is not zero.
\end{itemize}

\subsection*{2. Constructing a DNN at the local minimum:}
\begin{itemize}
    \item We initialize a DNN with parameters $\alpha^*$ and a loss function $\Bar{L}(\alpha(a))$ such that $\Bar{L}(\alpha(0)) = L(\alpha^*)$. That is, we create a new training time, denoted with the variable $a$, such that when $a=0$, the DNN parameters are those at the critical point $\alpha^*$.  This is achievable as the loss function is continuous.
\end{itemize}

\subsection*{3. Introducing a controlled random noise:}
\begin{itemize}
    \item We define a random matrix $R(a)$ with i.i.d. entries drawn from $N(0, \frac{g(a)}{N })$. We will choose $g(a)$ to vanish as $a$ approaches infinity.
\end{itemize}

\subsection*{4. Key Inequality and its Consequence:}
\begin{itemize}
    \item We claim that there exists a constant $g(a)$ such that $g(a) \rightarrow 0$ as $a \rightarrow \infty$ and for all $a>0$ and $N > N_0$ (for some $N_0$) we have that $\exists D(N)$ such that:
    \begin{equation}        \mathbb{P}\bigg(L(\alpha^*) \geq \Bar{L}(\alpha(a))\bigg)>1-D(N),
    \end{equation}
    with $D(N) \to 0$ as $N \to \infty$.

\end{itemize}

\subsection*{5. Proof of the Key Inequality:}
\begin{itemize}
    \item We will utilize the following properties:
    \begin{itemize}
        \item Lemma \ref{main_result_remove _R_loosandacc}: This lemma states that removing the random component $R_2$ from the weight matrix $W_2$ has a minimal impact on the loss function with high probability. That is based on Lemma \ref{main_result_remove _R_loosandacc}, we know that:
        \begin{equation}    \label{pruning_equation_1}
\mathbb{P}\bigg(|L_{\cancel{\text{reg}}}(\alpha_{S_2}(a))-L_{\cancel{\text{reg}}}(\alpha_{W_2}(a))| \leq G(N)H(\phi) \bigg)\geq 1-G(N),
   \end{equation}
where $G(N) \rightarrow 0$ as $N \rightarrow \infty$ and $\alpha_{S_2}(t)$ denotes the parameters of the DNN with the deterministic weight matrix $S_2(t)$.
        \item Frobenius norm: \begin{equation}
            \|A\|_F = \sqrt{\sum_{i=1}^{N} \sum_{j=1}^{N} A_{ij}^2}=\sqrt{\sigma_1^2 + \sigma_2^2 + \ldots + \sigma_N^2}.    
        \end{equation} 
        \item The loss:
\begin{equation}
\label{loss_noise_det}
L(\alpha(a))=-\frac{1}{|T|}\sum_{s\in T}\log\left(\phi_{i(s)}(s,\alpha(a))\right) + \sum_{i=1}^L\mu\|W_i(a)\|_F^2,
\end{equation}
\item Lemma \ref{theorem_time_dependent_variance}. We are assuming that for the non-zero singular values of $S$ we have, $\sigma_i(a)>1$. Because $0\leq g(a) \leq 1$, we have, by Lemma \ref{theorem_time_dependent_variance}, that:

\begin{equation}
    \sigma'_i(W') \xrightarrow[\text{}]{a.s.}
g(a)\frac{1+(\frac{\sigma_i}{\sqrt{g(a)}})^2}{\frac{\sigma_i}{\sqrt{g(a)}}} . 
\end{equation}
That is, $\exists D_2(N,\epsilon)$ such that $\forall \epsilon$:
\begin{equation}
\label{change_in_SV}
    \mathbb{P}\bigg(| \sigma'_i(W')-C(a)\frac{1+(\frac{\sigma_i}{\sqrt{C(a)}})^2}{\frac{\sigma_i}{\sqrt{C(a)}}} |>\epsilon \bigg)\leq D_2(N,\epsilon), 
\end{equation}
with $D_2(N,\epsilon) \to 0$ as $N \to \infty$. 
    \end{itemize}
    
\item 
Take $\alpha^*$ to be a local min of \eqref{loss_noise_det}, and suppose $\|R(\alpha^*)\|_2 \not = 0$. Initialize the DNN
\begin{equation}
\phi(\alpha,s) = \rho \circ \lambda \circ W_3 \circ \lambda \circ W_2 \circ \lambda \circ W_1 s
\end{equation}
with parameters $\alpha^*$ and loss $\Bar{L}(\alpha(a))$ s.t. $\Bar{L}(\alpha(0))=L(\alpha^*)$.  Assume that $R_2(0)$ is a random matrix with i.i.d.  $N(0,\frac{g(a=0)}{N})$, such that at $a=0$ we have that $W_2(0)=R_2(0)+S_2(0)$ satisfying the assumptions in Subsection \ref{assumptions}. Then there $\exists g(a) \to 0$ monotonically as $a \to \infty$, and take $R_2(a)$ to be a random matrix with i.i.d. taken from $N(0,\frac{g(a)}{N})$ and $W_2(a)=R_2(a)+S_2(0)$. By Lemma \ref{main_result_remove _R_loosandacc}, we know that:
        \begin{equation}    \label{pruning_equation_1}
\mathbb{P}\bigg(|L_{\cancel{\text{reg}}}(\alpha_{S_2}(a))-L_{\cancel{\text{reg}}}(\alpha_{W_2}(a))| \leq G(N)H(\phi) \bigg)\geq 1-G(N).
   \end{equation}

Furthermore, one can apply the same proof that was used to prove Lemma \ref{main_result_remove _R_loosandacc} to show that $\forall a$:

 \begin{equation}    \label{pruning_equation_1}
\mathbb{P}\bigg(|L_{\cancel{\text{reg}}}(\alpha_{S_2}(0))-L_{\cancel{\text{reg}}}(\alpha_{W_2}(a))| \leq G(N)H(\phi) \bigg)\geq 1-G(N),
   \end{equation}
given that $W_2(a)=R_2(a)+S_2(0)$ and the random part $R_2(a)$ has i.i.d. components with zero mean and variance $\frac{g(a)}{N}$. That is, as $a \to \infty$ if the only change in $W_2$ is that the random part $R_2$ goes to zero, then the loss $L_{\cancel{\text{reg}}}$ won't change significantly with probability depending on $N$. 

However, the loss:
\begin{equation}
\label{loss_noise_det}
L(\alpha(a))=-\frac{1}{|T|}\sum_{s\in T}\log\left(\phi_{C(s)}(s,\alpha(a))\right) + \mu \sum_{i=1}^L \|W_i(a)\|_F^2,
\end{equation}
will be decreasing because the term $\mu \sum_{i=1}^L \|W_i(a)\|_F^2$ will get smaller monotonically as $a \to \infty$. This can be seen from the equations:
\begin{equation}
            \|A\|_F = \sqrt{\sum_{i=1}^{N} \sum_{j=1}^{N} A_{ij}^2}=\sqrt{\sigma_1^2 + \sigma_2^2 + \ldots + \sigma_N^2},    
        \end{equation} 
        and \eqref{change_in_SV}. By \eqref{change_in_SV}, the $r$ largest singular values of $W_2$ will becomes smaller as $a \to \infty$ (with some probability depending on $N$), making $\mu \sum_{i=1}^L \|W_i(a)\|_F^2$ smaller monotonically (again with probability depending on $N$). Furthermore, the $N-r$ smallest singular values of $W_2$ will also go to zero monotonically as $a \to \infty$, again ensuring that $\mu \sum_{i=1}^L \|W_i(a)\|_F^2$ decrease (with probability depending on $N$). Thus, because for \eqref{loss_noise_det} the term $-\frac{1}{|T|}\sum_{s\in T}\log\left(\phi_{C(s)}(s,\alpha(a))\right)$ does not change as $a \to \infty$ (with probability depending on $N$) and $\mu \sum_{i=1}^L \|W_i(a)\|_F^2$ decreases as $a \to \infty$  (again with probability depending on $N$), we have that the loss given in $\eqref{loss_noise_det}$ must be decreasing as $a \to \infty$ (with probability depending on $N$).   That is, for $a>0$, $\exists D(N)$ such that $\forall \epsilon$:

     \begin{equation}        \mathbb{P}\bigg(L(\alpha^*) > \Bar{L}(\alpha(a))\bigg)>1-D(N).
     \end{equation}
     That is, $\alpha^*$ is not a local min (with probability depending on $N$).

     Furthermore, for a fixed training time $t$ as $N \to \infty$ we have that removing the random matrix $R_2$ will not affect the cross entropy term in the loss but will remove the singular values of $W_2$ smaller then $\sqrt{\lambda_+}$. That is, the singular values of $W_2$ corresponding to those of $R_2$ will be set to zero, which means that the loss will decrease by $\mu \|R_2\|^2_F$. This proves Theorem \ref{theorem_reduction_of_loss}. In fact, if $R_2\not 0$, then one can remove it (by sending $g(t) \to 0$), and the loss will decrease monotonically (with some probability depending on $N$) by $\mu \|R_2\|^2_F$, showing that we were not in a local min.

\end{itemize}

\end{proof}

\section{Algorithms for RMT-Based 
pruning of DNN}
\label{numerics_updated}

\subsection{BEMA technique for computing $\lambda_+$}
\label{finding_lambda_updated}

We outline the BEMA technique for identifying the optimal fit $\lambda_+$ for $\frac{1}{N}R^TR$ from the ESD of $X=R+S$. This is beneficial when analyzing matrices with information added to noise, aiming to pinpoint the furthest boundary of the MP distribution's compact support. The BEMA technique is known for its computational efficiency and precision for matrices with added noise. The complete procedure can be located in \cite{ke2021estimation}. Below, we present a more streamlined version for an $N \times N$ matrix $R$:

\begin{ex}\label{deterministic_random}
  In this example, we create a random $N\times N$ matrix $R$  with random i.i.d. $\mathcal{N}(0,1)$ (zero mean and unit variance ($\sigma^2=1$), Gaussian).  
  
  In this example, we generate a random $N \times N$ matrix $R$, where each entry is independently sampled from a standard normal distribution $\mathcal{N}(0,1)$ with zero mean and unit variance ($\sigma^2 = 1$).
  We take $S$ to be a  $N\times N$ deterministic matrix  with components given by
  
  \begin{equation}
  S[i,j]=\tan(\frac{\pi}{2}+\frac{1}{j+1})+\cos(i)\cdot\log(i+j+1)+\sin(j)\cdot\cos(\frac{i}{j}),
  \end{equation}
      $W=R+S$ and $X = \frac{1}{N}W^T W$. The BEMA algorithm is used to find the $\lambda_+$ of the ESD of $X$, as described in Algorithm \ref{BEMA_algo}. $R$ is a random matrix satisfying the conditions of Theorem \ref{RMT_MP_theorem}, and so the ESD of $\frac{1}{N}R^TR$ converges to the Marchenko-Pastur distribution as $N \to \infty$ and has a $\lambda_+$ that determines the rightmost edge of its compact support. We can imagine a situation in which $R$ is not directly known, and the goal is to find an estimator of $\lambda_+$ from the ESD of $X$. See Fig. \ref{fig:ESD_X&MP} for the result of the ESD of $X$ with the Marchenko-Pastur distribution that best fits the ESD shown in red. 
 \begin{figure}[h!]
	\centering	\includegraphics[width=.8\textwidth]{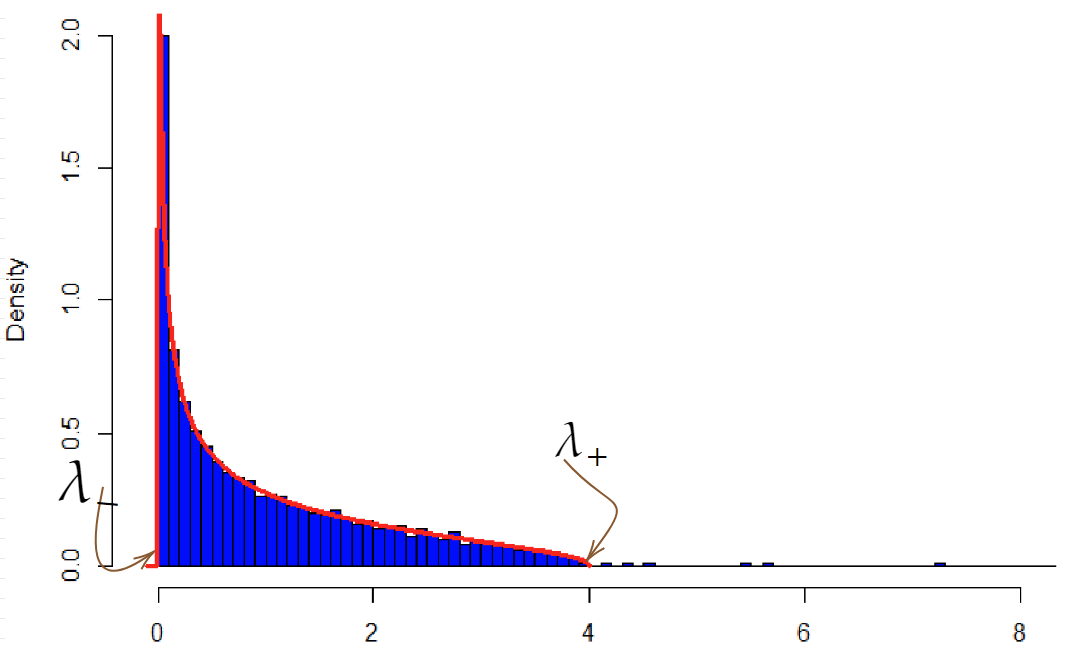}
			
		\caption{In blue we have the ESD of X, in red the Marchenko-Pastur distribution which best fits the ESD based on the BEMA algorithm.}
		\label{fig:ESD_X&MP}
\end{figure}

\end{ex}

\begin{remark}
\label{alpha_hyper}
The method relies on parameters $\alpha \in (0,1/2), \beta \in (0,1)$. This dependency can be observed by adjusting $\alpha$ and $\beta$ as showcased in Fig. \ref{alpah_beta_dep}. The highlighted line represents $\lambda_+ = 4$, the accurate $\lambda_+$ for $\frac{1}{N}R^TR$. In this case, while the influence of $\alpha$ is minimal for ample values, adjusting $\beta$ regulates the assurance that the eigenvalues of the matrix $R$ will remain below the $\lambda_+$ estimator of the MP distribution.

\begin{figure}
     \begin{subfigure}{0.4\textwidth}
  \includegraphics[width=\textwidth]{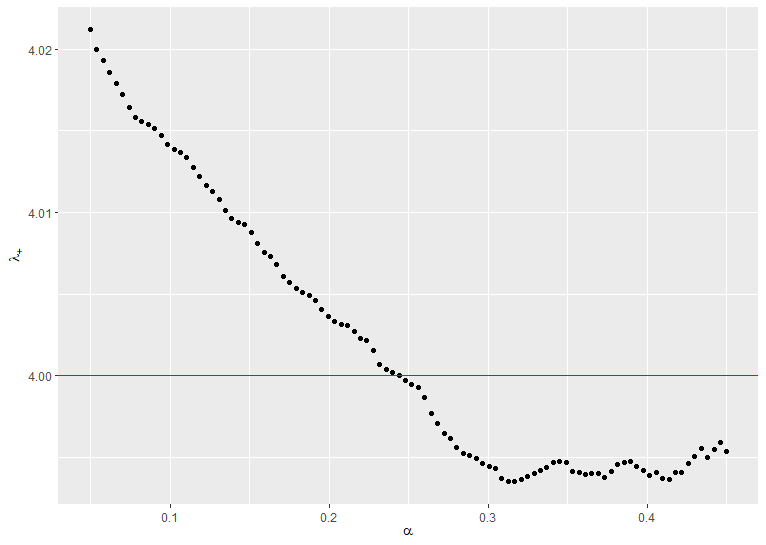}
    \caption{Dependence of algorithm the choice of $\alpha$, $\beta = 0.5$. In this example the rank of the deterministic matrix $S$ is fairly low.}
    \label{dep_alpha}
     \end{subfigure}
     \hfill
     \begin{subfigure}{0.4\textwidth}
  \includegraphics[width=\textwidth]{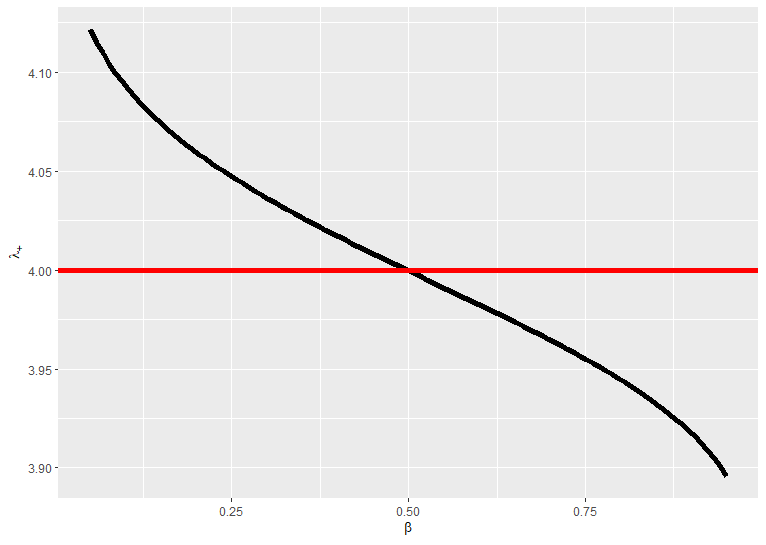}
    \caption{Dependence of algorithm on the choice of $\beta$, $\alpha = 0.25$.}
    \label{dep_beta}
     \end{subfigure}
     \caption{How $\lambda_+$ depends on the hyperparameters $\alpha$ and $\beta$.}
     \label{alpah_beta_dep}
     \end{figure}

\end{remark}

The BEMA technique will be utilized to deduce $\lambda_+$ from $X_l(t)$'s ESD. As DNN learning evolves, most of the $X_l(t)$ eigenvalues are anticipated to fit the MP distribution. Yet, certain eigenvalues might surpass the MP distribution's bulk and correspond to $S_l(t)$'s singular values. The BEMA's objective is to spot the MP distribution's outermost boundary, which aids in assessing $\lambda_+$. Grasping $\lambda_+$ is pivotal as it provides an understanding of the DNN's functionality during learning and its adaptability to fresh data.

Together with the SVD, the BEMA technique can guide which singular values of DNN's weight matrices $W_l$ to omit during learning. The SVD dissects the weight matrix, yielding a breakdown into singular values and vectors, which facilitates an RMT-driven analysis of their distribution. By employing the BEMA, one can distinguish eigenvalues corresponding to $S_l$'s singular values from those linked to $R_l$'s singular values. Removing eigenvalues tied to $R_l$ can enhance the DNN learning process's effectiveness.

\subsection{Importance of singular value decomposition in Deep Learning}

For a matrix $A$ of dimensions $N \times M$, its SVD comprises a factorization $A = U\Sigma V^T$, where:

\begin{itemize}
\item $U$ is a square orthogonal matrix of size $N \times N$.
\item $V$ is a square orthogonal matrix of size $M \times M$.
\item $\Sigma$ is a diagonal matrix of dimensions $N \times M$, with its diagonal entries corresponding to singular values $\sigma_i$, and non-diagonal entries being zero.
\end{itemize}

For matrix $X = W^T W$, with eigenvalues $\lambda_i$, the associated singular values of $W$ are denoted by $\sigma_i = \sqrt{\lambda_i}$. This ties singular values to eigenvalues of a matrix $W$'s symmetrization.

For a DNN's weight matrix $W_l$, studies have shown that omitting insignificant singular values through its SVD during training can diminish parameters yet augment accuracy. In the upcoming sections, we will elucidate how RMT can assist in determining which singular values to exclude from a DNN layer, ensuring the DNN's precision remains intact.

Specifically, integrating the BEMA technique with the SVD of $W_l$ can determine which singular values to exclude during training. This is accomplished by first computing $W_l$'s SVD, followed by determining the eigenvalues of the symmetrical matrix $X_l = \frac{1}{N}W_l^T W_l$. The derived eigenvalues can be related to the singular values of $W_l$ via $N\lambda_i = \sigma_i^2$. Using the BEMA to estimate $\lambda_+$ enables setting a threshold for $W_l$'s singular values. Those below this threshold can be discarded without affecting accuracy, as they are probably less vital for the DNN's output. This procedure can be recurrent during training, updating the threshold as required.

\subsection{How well does the  ESD of $X$ fit the MP Distribution with Spikes}
\label{Alignment_Evaluation}
In this subsection, we outline a method to investigate whether the ESD of \(X\) may be derived from a particular MP distribution that may contain spiked eigenvalues. The starting point of this method uses the BEMA technique to pinpoint the most appropriate MP distribution. The best-fitting distribution offers a theoretical cumulative distribution function (CDF), while the observed cumulative spectral distribution for \(X\) is obtainable. Contrasting these distributions, we can rule out the idea that \(X\) is governed by the proposed MP distribution if there is a notable disparity between them. We will now elaborate on these notions, beginning with the empirical cumulative spectral distribution.

\begin{defn} \label{ECS_Reformulated}
    Consider \(G\) as an \(N \times M\) matrix with its ESD \(\mu_{G_M}\) as presented in Definition \ref{ESD_Definition}. The observed cumulative spectral distribution of \(G\), denoted as \(F_G: \R \to \R\), is:
    \begin{equation}
        F_G(a) = \mu_{G_m}((-\infty, a])
    \end{equation}
\end{defn}

Remarkably, the CDFs for the MP distribution are available in an explicit form. Using these formulas, we can lucidly detail our approach. A tuning parameter \(\gamma \in (0,1)\) specifies the sensitivity of our evaluation.

This approach measures the utmost disparity between the predicted and observed CDFs by sampling at each point in the observed distribution. As this is meant for the special scenario of testing for MP distributions with spikes, this knowledge allows us to refine our evaluation over simply determining the \(L^\infty\) disparity between the two distributions.

This refinement is evident in the step that determines \(i_{\text{min}}\) and \(i_{\text{max}}\). Given that BEMA solely utilizes data in the quantile between \((\alpha, 1 - \alpha)\) for the optimal fit, it's reasonable to inspect the fit within the same boundaries. In this scenario, a spiked MP distribution would likely be inadequately represented by its generative MP distribution around the predominant eigenvalues (the spikes), so it's logical to test only the core values for fit accuracy.

\subsection{Eliminating singular values while preserving accuracy}
\label{stable_acc}

\begin{figure}%
    \centering
    \subfloat[\centering Full Empirical Density]{{\includegraphics[width=6.5cm]{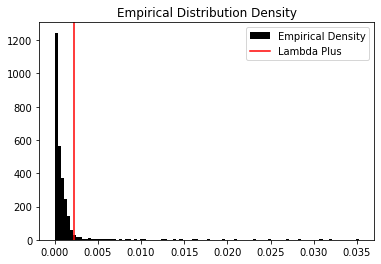} }}%
    \qquad
    \subfloat[\centering Zoomed Density ]{{\includegraphics[width=6.5cm]{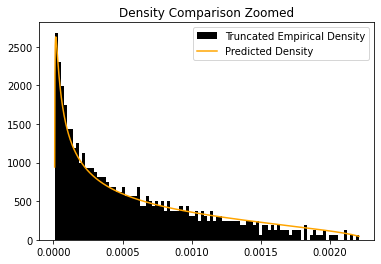} }}
    \caption{The ESD of \(X_l\) and its best fit MP distribution}
    \label{MPdis1}%
\end{figure}

Algorithm \ref{algo_1} is how we train DNNs together with MP-based pruning.

\begin{ex}
    Consider an original DNN with two hidden layers, each consisting of 10 nodes. The total number of parameters in this case would be 100. By employing SVD and removing \(8\) small singular values in the weight layer matrix of this DNN, we can split the hidden layer into two, resulting in a new DNN with three hidden layers. The first layer will have 10 nodes, the second layer will have 2 nodes, and the third layer will have 10 nodes. By keeping only two singular values in the SVD, we now have only 20 parameters, see Figure \ref{fig:dnn_split}. In practice, we don't actually split the layer 

\begin{figure}[h]
\centering
\begin{tikzpicture}
% Original DNN
\foreach \i in {1,2,3,4,5,6,7,8,9,10} {
\node[draw, circle, minimum size=0.5cm] (A\i) at (1, 0.5*\i) {};
\node[draw, circle, minimum size=0.5cm] (B\i) at (3, 0.5*\i) {};
}
% Dashed lines for original DNN
\foreach \i in {1,2,3,4,5,6,7,8,9,10} {
\foreach \j in {1,2,3,4,5,6,7,8,9,10} {
\draw[dashed] (A\i) -- (B\j);
}
}
\node[rotate=90] at (0.5, 2.5) {Original DNN};% New DNN
\begin{scope}[xshift=5cm]
    \foreach \i in {1,2,3,4,5,6,7,8,9,10} {
        \node[draw, circle, minimum size=0.5cm] (C\i) at (1, 0.5*\i) {};
        \node[draw, circle, minimum size=0.5cm] (E\i) at (3, 0.5*\i) {};
    }
    \node[draw, circle, minimum size=0.5cm] (D1) at (2, 2) {};
    \node[draw, circle, minimum size=0.5cm] (D2) at (2, 3) {};
    
    % Dashed lines for new DNN
    \foreach \i in {1,2,3,4,5,6,7,8,9,10} {
        \draw[dashed] (C\i) -- (D1);
        \draw[dashed] (C\i) -- (D2);
        \draw[dashed] (D1) -- (E\i);
        \draw[dashed] (D2) -- (E\i);
    }
    \node[rotate=90] at (0.5, 2.5) {New DNN};
\end{scope}
\end{tikzpicture}
\caption{Original DNN with two hidden layers, each with 10 nodes (total 100 parameters), is transformed into a new DNN with three hidden layers. The first layer has 10 nodes, the second layer has 2 nodes (keeping only two singular values in the SVD), and the third layer has 10 nodes, resulting in a total of 20 parameters.}
\label{fig:dnn_split}
\end{figure}
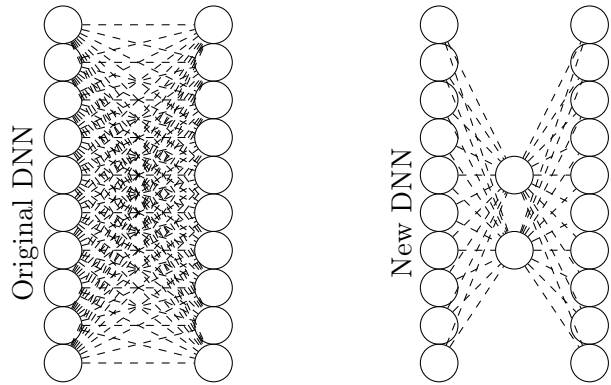

\end{ex}

\begin{ex}
\label{threshold_example}
     We used the above approach for a DNN trained on MNIST. In this example the DNN has two layers, the first with a \(784 \times 1000\) matrix \(W_1\) and the second with a \(1000 \times 10\) matrix \(W_2\). The activation function was ReLU. We trained the DNN for \(10\) epocs and achieved a \(98\%\) accuracy on the test set.

 We perform a SVD on \(W_1\), in this case \(\Sigma\) is a \(784 \times 1000\) matrix. Even if we only keep the biggest \(60\) \(\sigma_i\) of \(W_1\) and transform the first layer into two layers \(W_{1,1}\) and \(W_{2,1}\) the accuracy is still \(98\%\). \(W_1\) has \(784 \times 1000 = 784000\) parameters but \(W_{1,1}\) and \(W_{2,1}\) together have \(60 \times 784 + 60 \times 1000 = 106440\) parameters. This is a significant reduction in the number of parameters without any loss of accuracy. 
\end{ex}

\subsection{Algorithm for Sparsifying Singular Vectors of a DNN Layer}
\label{SV-sparsification}

Let $W_l$ denote the $l$-th layer matrix of a DNN and for simplicity we assume that it is $N \times N$. The SVD of $W_l$ is given by $W_l = U \Sigma V^T$, where $U$ and $V$ are the left and right singular vectors, and $\Sigma$ is a diagonal matrix of singular values.

Given a threshold $\sqrt{\lambda_+ N}$, the algorithm proceeds as follows:

Note: The algorithm focuses on sparsifying the singular vectors based on the dynamic threshold determined by the singular values and a predefined threshold $\theta$. Here, $\text{apply\_sparsity}(U_{:,i}, T_i(\Sigma_{ii}))$ sets to zero all elements in the vector $U_{:,i}$ below the threshold $T_i(\Sigma_{ii})$.

\subsection{Algorithm for element-wise sparsification of a matrix}

\subsection{The MP and spike metrics of a DNN model}
\label{alpha_and_metrics}

First, we present two figures of the ESD of different layers of the base ViT model studied in Subsection \ref{sec:pruning_VITs}, see Figs.  \ref{MP_fit_ViT_layer_example_1}, \ref{MP_fit_ViT_layer_example_2}. We see that some layers fit the MP distribution "very" well, while some layers don't. When pruning layers of the DNN, we remove larger parameters in the layers that are more random because of the larger amount of randomness in those layers.

\begin{figure}[h!]
	\centering	\includegraphics[width=.8\textwidth]{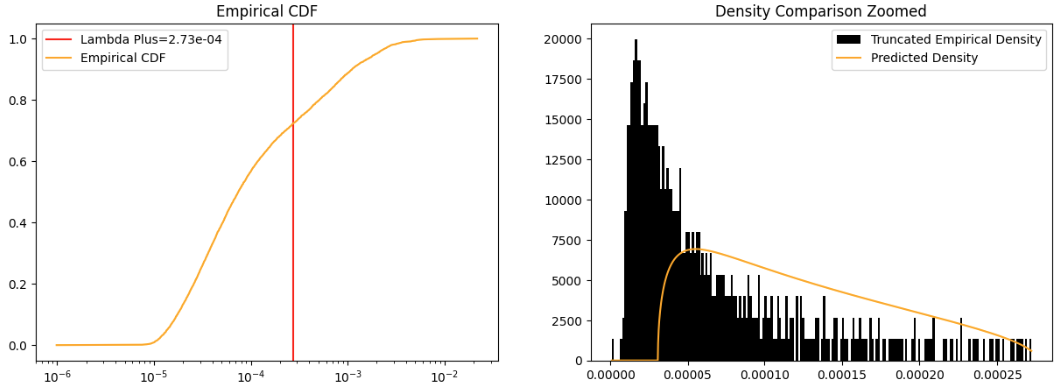}
			
		\caption{Best fit MP distribution (on the left) for $\alpha=.25$ and $\beta=.8$ for a layer of the base ViT model. Here the MP fit error is $.74$ while the percentage of singular values smaller than $\sqrt{\lambda_+}$ is $73\%$. For more on the right figure, see Section \ref{reg_problem}. }
		\label{MP_fit_ViT_layer_example_1}
\end{figure}

\begin{figure}[h!]
	\centering	\includegraphics[width=.8\textwidth]{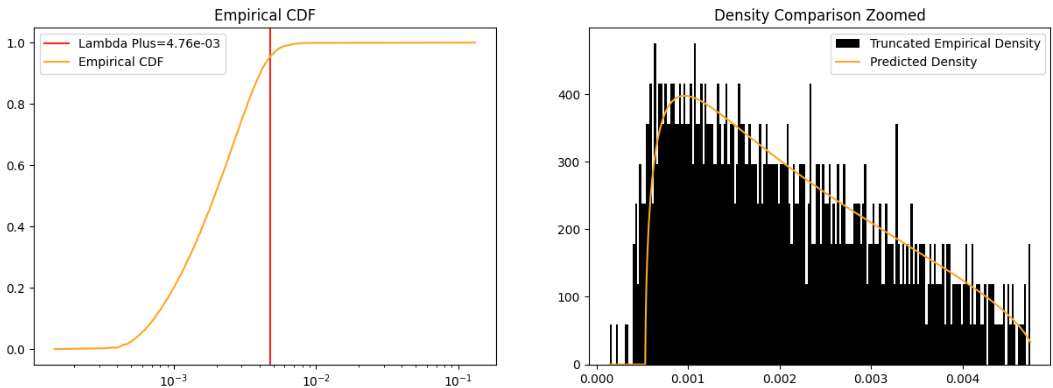}
			
		\caption{Best fit MP distribution (on the left) for $\alpha=.25$ and $\beta=.8$ for a layer of the base ViT model. Here the MP fit error is $.01$ while the percentage of singular values smaller than $\sqrt{\lambda_+}$ is $99.73\%$. For more on the right figure, see Section \ref{reg_problem}. }		\label{MP_fit_ViT_layer_example_2}
\end{figure}

Second, we show how the hyperparameter $\alpha$ discussed in Remark \ref{alpha_hyper} affects the MP fit and spike metrics; see Fig. \ref{alpha_vs_LR_MP}.

\begin{figure}[h!]
	\centering	\includegraphics[width=.8\textwidth]{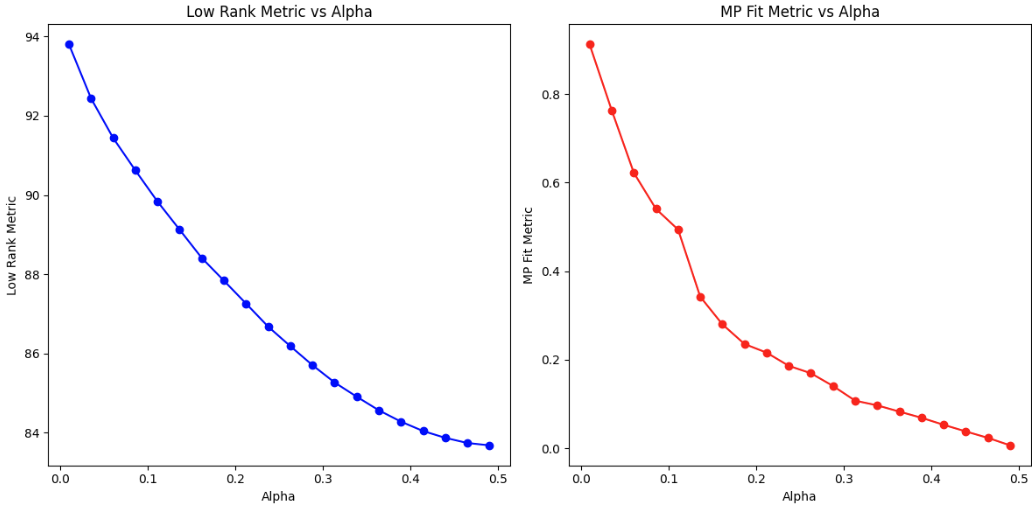}
			
		\caption{In the left figure, we have $\alpha$ vs LRM and on the right we have $\alpha$ vs MP fit metric. This is for the base ViT model studied in \ref{ViT_example1_sparsification_with_RMT_1_FT}.}
		\label{alpha_vs_LR_MP}
\end{figure}

\section{ViT model architecture}
\label{ViT_modl}

\subsection{Image representation and embedding}
Consider an input image \( I \) which is a matrix of pixels. This image is divided into a grid of \( N \) patches, \( \{P_1, P_2, ..., P_N\} \). Each patch \( P_i \) is then linearly embedded into a \( d \)-dimensional space. This process is done by flattening each patch and then applying a linear transformation, often followed by adding a positional encoding.

The mathematical process can be described as follows:
\[
\text{vec}(P_i) = \text{Flatten}(P_i)
\]

We now apply a fully connected layer with no activation function: 

\[
x_i = \text{vec}(P_i)W + b
\]

Here, \(\text{vec}(P_i)\) is the flattened vector of the patch \( P_i \), \( W \) is a learnable weight matrix in \(\mathbb{R}^{p \times d}\), where \( p \) is the number of pixels in a patch, and \( b \) is a learnable bias vector in \(\mathbb{R}^d\). The result \( x_i \) is the embedded representation of the patch.

\subsection{Positional encoding}
Due to the self-attention mechanism's permutation-invariant nature, it does not inherently account for the order or position of the input elements. In the context of images, where the relative position of patches is vital, positional encoding is added to the embedded representations to retain this spatial information.

Positional encodings are typically vectors of the same dimension \( d \) as the patch embeddings, which are added to the embeddings:

\[
x_i^{'} = x_i + \text{PositionalEncoding}(i)
\]

These encodings can be either learned during the training process or predefined based on sinusoidal functions, capturing the relative or absolute position of the patches in the image grid. This way, the transformer can leverage the spatial structure of the image.

\subsection{Self-Attention Mechanism}

In the self-attention mechanism, each element in the sequence is updated by aggregating information from the entire sequence. This is done by computing a weighted sum of value vectors \( V \), where the weights are determined by a compatibility function of query \( Q \) and key \( K \) vectors.

Each vector \( x_i \) is transformed into a set of query \( Q \), key \( K \), and value \( V \) vectors through linear transformations (rather then affine)  using learnable weights \( W^Q \), \( W^K \), and \( W^V \):

\[
Q_i = x_iW^Q, \quad K_i = x_iW^K, \quad V_i = x_iW^V
\]

where \( W^Q, W^K, W^V \in \mathbb{R}^{d_{emb} \times d_k} \) and \(d_{emb}\) is the dimension space of the embeddings and \(d_k < d_{emb}\). Similar to the previous weight matrix \(W\), matrices \(W^Q, W^K\), and \(W^V\) begin with random values that get learned throughout the training. 

%These matrices are initialized randomly and they will be trained parameters. (Explain difference between Q, K and V)

\subsection{Calculation of attention}

The attention mechanism is a fundamental concept in various models for processing sequences of data (an example of such a sequence can be a sequence of patches of data). Fundamentally, it calculates the similarity of each part of the data sequence (patch) with respect to a different element of the sequence (a different patch). 

\subsubsection*{Attention Score}

The attention score is a measure of the relevance or importance between two elements in a sequence, typically represented as patches, tokens, or vectors. The attention score between the $i$-th query vector $Q_i$ and the $j$-th key vector $K_j$ is computed using the dot product of $Q_i$ and $K_j$, which is then scaled by the inverse of the square root of the dimension of the key vectors. Mathematically, it is represented as:

\[
\text{Attention}(Q_i, K_j) = \text{softmax}\left(\frac{Q_iK_j^T}{\sqrt{d_k}}\right)
\]

Here, \( d_k \) represents the dimension of the key (and query) vectors. The dot product $Q_iK_j^T$ computes a measure of similarity or alignment between the query and key. The scaling factor $\frac{1}{\sqrt{d_k}}$ is used to avoid overly large dot products as the dimensionality increases, which can lead to gradients that are too small for effective learning.

\subsubsection{Output of the Attention Layer}

Once the attention scores are computed, the output for each query vector is obtained by taking a weighted sum of the value vectors, where the weights are the attention scores. The output for the $i$-th query vector is computed as:

\[
\text{Output}_i = \sum_{j=1}^{N} \text{Attention}(Q_i, K_j)V_j
\]

In this equation, $V_j$ represents the value vector corresponding to the $j$-th element of the sequence. The weighted sum allows the model to aggregate information from the entire sequence, weighted by how relevant each element is to the query.

\subsection{Position-Wise Feed-Forward Networks}

Each layer of the Transformer also contains a feed-forward network, which consists of two linear transformations with an absolute value activation in between. These networks are applied to each position separately and identically.

\[
\text{FFN}(x) = \text{Abs}(xW_1 + b_1)W_2 + b_2
\]

\end{document}